\documentclass[Afour, sageh, final, times]{style/sagej}

\usepackage{moreverb,url}
\usepackage{makecell}
\usepackage{calc}
\usepackage{mathtools}

\usepackage[bookmarksopen,bookmarksnumbered]{hyperref}
\hypersetup{
    colorlinks=true,
    citecolor=blue,
    linkcolor=blue,
    urlcolor=blue,
    pdfpagemode=FullScreen,
}

\usepackage{subcaption}
\usepackage[toc,page]{appendix}

\usepackage{xspace}
\usepackage{balance}
\usepackage{siunitx}
\usepackage{color}
\usepackage{graphicx}
\usepackage{amsfonts}
\usepackage{amssymb}
\usepackage{amsmath}
\usepackage{algorithm}
\usepackage{algpseudocode}
\usepackage{ifmtarg}
\usepackage{listings}
\usepackage{multirow}

\usepackage{amsthm}
\newtheorem{theorem}{Theorem}

\def\volumeyear{2020}
\runninghead{Orthey et al.}

\usepackage{xspace}
\usepackage{tikz}
\usepackage{pgfplots}
\pgfplotsset{compat=1.16}
\usepgfplotslibrary{fillbetween}
\usetikzlibrary{patterns}
\usetikzlibrary{calc}

\def\Lone{\ensuremath{L^1}\xspace}
\def\Ltwo{\ensuremath{L^2}\xspace}

\algnewcommand\True{\textbf{true}\space}
\algnewcommand\False{\textbf{false}\space}

\newcommand\restrict[2]{\ensuremath{\left.#1\right|_{#2}}}

\def\krrt{k_{\text{RRT}}}
\def\kprm{k_{\text{PRM}}}

\def\R{\mathbb{R}}

\def\quotientspace{Q}

\def\N{\mathbb{N}}

\def\hstar{h^{*}}

\def\free{\textnormal{f}}
\def\init{\textnormal{I}}

\def\X{X}

\def\Xf{\X_{\free}}
\def\Xi{\X_{\init}}

\def\Xk{\X_{k}}

\def\Xkk{\X_{k-1}}

\def\Y{Y}
\def\fiber{F}

\def\B{B}
\def\Bf{\B_{\free}}

\def\base{B}

\def\x{x}
\def\xi{\x_I}
\def\xik{\x^k_I}
\def\xg{\x_G}

\def\Xg{\X_G}
\def\Xgk{\X^k_G}

\def\xk{\x_k}
\def\xj{\x_j}

\def\xr{\x_{\text{rand}}}

\def\xn{\x_{\text{near}}}

\def\xnbh{\x_{\text{nbh}}}
\def\xw{\x_{\text{new}}}
\def\xnew{\x_{\text{new}}}

\def\dmax{d_{\text{MAX}}}
\def\bmax{b_{\text{MAX}}}

\def\PriorityQueue{\ensuremath{\mathbf{X}}}

\def\planningproblem{\left(\x_I, \Xg, \X\right)}

\def\G{\ensuremath{\mathbf{G}}}
\def\Gk{\ensuremath{\G_k}}
\def\Gkk{\ensuremath{\G_{k-1}}}
\def\path{\ensuremath{\mathbf{p}}}

\makeatletter
\renewcommand{\toprule}{\hrule height.8pt depth0pt \kern2pt} 
\renewcommand{\midrule}{\kern2pt\hrule\kern2pt} 
\renewcommand{\bottomrule}{\kern2pt\hrule\relax}
\newcommand{\algcaption}[2][]{%
  \refstepcounter{algorithm}%
  \toprule
  \textbf{{\raggedright\fname@algorithm~\thealgorithm}}\ #2\par 
  \midrule
}
\makeatother

\makeatletter
\newcommand*{\algrule}[1][\algorithmicindent]{\makebox[#1][l]{\hspace*{.5em}\vrule height .75\baselineskip depth .25\baselineskip}}%

\newcount\ALG@printindent@tempcnta
\def\ALG@printindent{%
    \ifnum \theALG@nested>0
    \ifx\ALG@text\ALG@x@notext
    \addvspace{-3pt}
    \else
    \unskip
    \ALG@printindent@tempcnta=1
    \loop
    \algrule[\csname ALG@ind@\the\ALG@printindent@tempcnta\endcsname]%
    \advance \ALG@printindent@tempcnta 1
    \ifnum \ALG@printindent@tempcnta<\numexpr\theALG@nested+1\relax
    \repeat
    \fi
    \fi
}%
\usepackage{etoolbox}
\patchcmd{\ALG@doentity}{\noindent\hskip\ALG@tlm}{\ALG@printindent}{}{\errmessage{failed to patch}}
\makeatother

\begin{document}

\title{Multilevel Motion Planning:\\ A Fiber Bundle Formulation}

\author{Andreas Orthey\affilnum{1,3} and Sohaib Akbar\affilnum{2} and Marc Toussaint\affilnum{1,3}}

\affiliation{\affilnum{1}Max Planck Institute for Intelligent Systems, Stuttgart, Germany\\
\affilnum{2}University of Stuttgart, Germany\\
\affilnum{3}Technical University of Berlin, Germany}

\corrauth{Andreas Orthey}

\email{\{aorthey\}@is.mpg.de}

\begin{abstract}

High-dimensional motion planning problems can often be solved significantly faster by using multilevel abstractions.
While there are various ways to formally capture multilevel abstractions, we formulate them in terms of fiber bundles. 
Fiber bundles essentially describe lower-dimensional projections of the state space using local product spaces, which allows us to concisely describe and derive novel algorithms in terms of bundle restrictions and bundle sections. 
Given such a structure and a corresponding admissible constraint function, we develop highly efficient and asymptotically-optimal sampling-based motion planning methods for high-dimensional state spaces. 
Those methods exploit the structure of fiber bundles through the use of bundle primitives.
Those primitives are used to create novel bundle planners, the rapidly-exploring quotient-space trees (QRRT*), and the quotient-space roadmap planner (QMP*). 
Both planners are shown to be probabilistically complete and almost-surely asymptotically optimal.
To evaluate our bundle planners, we compare them against classical sampling-based planners on benchmarks of four low-dimensional scenarios, and eight high-dimensional scenarios, ranging from 21 to 100 degrees of freedom, including multiple robots and nonholonomic constraints. 
Our findings show improvements up to 2 to 6 orders of magnitude and underline the efficiency of multilevel motion planners and the benefit of exploiting multilevel abstractions using the terminology of fiber bundles.

\end{abstract}

\keywords{Optimal motion planning, Multi-robot motion planning, Nonholonomic planning, Fiber bundles}

\maketitle

\section{Introduction}

As human beings, we often tackle complex problems by employing abstraction hierarchies~\citep{simon_1969, ballard2015brain}. Abstraction hierarchies, however, are not only helpful for us, but they can also be used by robots to solve high-dimensional motion planning problems---efficiently, and with strong guarantees~\citep{orthey_2019, reid_2019, ichter_2019, vidal_2019, gochev_2012}.
\begin{figure}
    \centering
    \includegraphics[width=\linewidth]{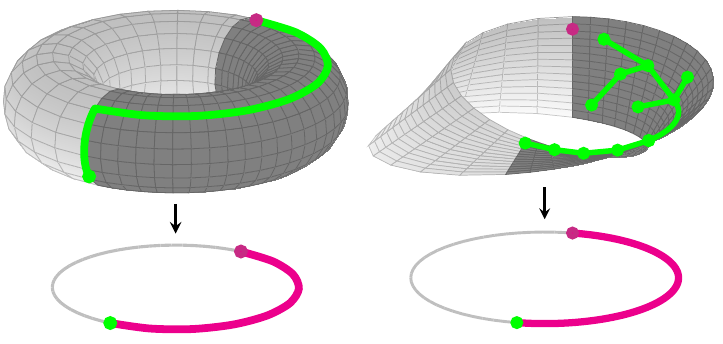}
    \caption{Efficient search over fiber bundles by exploiting path restrictions and sections. \textbf{Left}: A fiber bundle $T^2 \rightarrow S^1$ which abstracts the Torus to a base space (circle). A path on the base space (magenta) imposes a path restriction (dark grey) on the Torus, on which we can find a section (green). \textbf{Right}: A non-trivial fiber bundle from the Mobius strip to a circle, with path restriction (dark grey), and a planned tree (green). \label{fig:pullfigure}}
\end{figure}

\begin{figure*}
\centering
\begin{tikzpicture}[]
     \path[
        nodes={
            rectangle,
            draw,
            align=center,
            minimum width=5.5cm,
            minimum height=.75cm,
            anchor=east
            }
        ]
        (0,3) node[fill=gray!20] (a1) {High-Dimensional State Space}
        (0,1.5) node[fill=gray!20] (a2) {Query (Start, Goal, Objective)}
        
        (6,3) node (b1) {Impose Multilevel Abstraction}
        (6,0) node (b2) {Plan and Optimize}

        (12,3) node (c1) {Model as Fiber Bundles (Sec.~\ref{sec:bundles})}
        (12,2) node (c2) {Derive Bundle Primitives (Sec.~\ref{sec:primitives})}
        (12,1) node (c3) {Create Bundle Planners (Sec.~\ref{sec:bundleplanners})}
        ;

    \begin{scope}
       \draw[->, thick] (a1.east) -- (b1.west); 
       \draw[->, thick] (a2.south) |- (b2.west); 
       \draw[->, thick] (b1.east) -- (c1.west); 
       \draw[->, thick] (c1.south) -- (c2.north); 
       \draw[->, thick] (c2.south) -- (c3.north); 
       \draw[->, thick] (c3.south) |- (b2.east); 
    \end{scope}   
    
\end{tikzpicture}
\caption{Overview about the contributions of this paper (gray boxes are input). We tackle motion planning in high-dimensional state spaces by imposing multilevel abstractions. Those abstractions are modelled as fiber bundles (Sec.~\ref{sec:bundles}), from which bundle primitives are derived (Sec.~\ref{sec:primitives}), and the bundle planners QRRT, QRRT*, QMP, and QMP* are created (Sec.~\ref{sec:bundleplanners}). Those bundle planners efficiently exploit bundle primitives. Given a new query in a high-dimensional state space, this allows us to plan and optimize motions, even in scenarios where non-bundle planners fail.
\label{fig:system}}
\end{figure*}

However, abstractions in robot motion planning are difficult to model. The state spaces involved are usually continuous, high-dimensional, and might contain intricate constraints~\citep{konidaris_2019}. It is often unclear how to model abstractions over such state spaces, how such abstractions can be efficiently exploited, and how we can keep completeness, or optimality guarantees.

To tackle this problem, we introduce the framework of fiber bundles~\citep{steenrod_1951, lee_2003} to robot motion planning. 
Fiber bundles are a convenient way to model multilevel abstractions, because they provide the useful concepts of bundle sections and restrictions. 
Both bundle sections and restrictions allow us to develop novel planning algorithms to exploit them for efficient sampling. 
Fig.~\ref{fig:pullfigure} illustrates these two notions, which we will introduce in more detail later. 
On the left, we show an abstraction (a projection) from the torus $T^2$ to the circle $S^1$. A path on the circle (magenta) imposes a path restriction (dark grey) on the torus $T^2$. 
On this path restriction, we can find path sections (green), which are paths which project onto the path on the circle. The path on the circle acts here as a guide to quickly find solution paths.  
On the right of Fig.~\ref{fig:pullfigure}, we show the same scenario for the Mobius strip, which is a non-trivial fiber bundle~\citep{mobius_1858}.
A tree (green) is shown on the path restriction (dark grey), where planning is restricted by the path on the circle (magenta). 
By constructing sections and restrictions, we can create novel motion planners, which are able to quickly find relevant regions in state space to plan in. 
This allows those planner to efficiently solve high-dimensional motion planning problems. 
An overview of our approach is shown in Fig.~\ref{fig:system}.

\subsection{Our Contributions}

Our work builds on prior publications at the International Conference on Intelligent Robots and Systems (IROS) \citep{orthey_2018} and the International Symposium on Robotics Research (ISRR) \citep{orthey_2019}. Our contributions over this prior work are:
\begin{enumerate}
    \item We propose to formulate multilevel motion planning problems using the terminology of fiber bundles, and introduce the particularly useful notions of bundle sections and bundle restrictions.
    \item Based on this formulation, we improve upon previous bundle planners QRRT and QMP and develop two new bundle planners QRRT* and QMP*.
    \item We show QRRT* and QMP* to be probabilistically complete and asymptotically optimal by inheritance from RRT*, and PRM*.
    \item We define primitive methods on fiber bundles and conduct a meta-analysis to find the best implementation of those methods.
    \item We provide open source implementations of fiber bundles and bundle planners (together with a high-level introduction, a tutorial, and demos), which is freely available in the open motion planning library (OMPL).
    \item We evaluate QMP, QMP*, QRRT, and QRRT* by comparing them against planners from OMPL on four low-dimensional scenarios ranging from $2$ degrees of freedom (dof) to $7$-dof, and on eight high-dimensional scenarios ranging from $21$-dof to $100$-dof.
\end{enumerate}

Our evaluations show that our planners QMP, QMP*, QRRT, and QRRT* can efficiently exploit fiber bundles. While they are competitive in low-dimensional spaces, they are particularly useful in high-dimensional spaces, where other planners have difficulty finding solutions. We show that for high-dimensional state spaces, our bundle space planners can provide runtime improvements by up to $2$ to $6$ order of magnitude.
\section{Related Work}

We provide a brief overview on motion planning with focus on optimal planning. We then discuss multilevel motion planning by discussing how our approach of fiber bundles is connected to existing research. 
In particular, we stress the point that fiber bundles often contribute additional vocabulary, which we can exploit to develop novel methods, simplify notations and better structure our code. We finish by reviewing complementary approaches to fiber bundles and we discuss what our approach adds to existing approaches in (optimal) multilevel motion planning.

\subsection{Motion Planning}

To solve motion planning problems, we need to develop algorithms to find paths through the state space of a robot \citep{lozano_perez_1983}.
Searching such a state space is NP-hard \citep{canny_1988}, but we can often efficiently find solutions using sampling-based algorithms \citep{lavalle_2006, salzman_2019}, where we randomly sample states and connect them to a graph \citep{kavraki_1996} or to a tree \citep{lavalle_1998}. Many variations are possible, for example using bidirectional trees growing from start and goal state \citep{kuffner_2000, lavalle_2001}, lazy evaluation of edges \citep{bohlin_2000, mandalika_2019}, sparse graphs \citep{simeon_2000, jaillet_2008}, safety certificates \citep{bialkowski_2016} or deterministic sampling sequences \citep{janson_2018, palmieri_2019}. 

Often, we like to find an optimal path by minimizing an optimization objective~\citep{karaman_2011}. To find optimal paths, we could transfer ideas from classical planning like lazy edge evaluation \citep{hauser_2015} or sparse graphs \citep{dobson_2014}. However, cost function landscapes \citep{jaillet_2010} often provide additional information we can exploit. Examples include informed sets to prune irrelevant states \citep{gammell_2018, gammell_2020} or fast marching trees to grow trees outward in cost-to-come space \citep{janson_2015}. Recently, sampling-based motion planning algorithms have also been extended to address zero-measure constraints \citep{kingston_2019}, implicit constraints \citep{jaillet_2013}, dynamic constraints \citep{li_2016}, or dynamic environments \citep{otte_2016}.

All those algorithms can robustly solve many planning problems, provide formal guarantees like probabilistic completeness, or asymptotic optimality, and have been verified in a wide variety of applications \citep{lavalle_2006, sucan_2012}. However, as we show in Sec.~\ref{sec:evaluation}, we often cannot use them to solve high-dimensional planning problems in a reasonable amount of time (like less than $60$ seconds). We believe additional information is required to solve those problems efficiently. A possible candidate for this additional information are multiple levels of abstraction.

\subsection{Multilevel Motion Planning}

In multilevel motion planning, we impose a multilevel abstraction on the state space and we develop algorithms which exploit this abstraction. While several models for multilevel motion planning have been put forward, we propose to use fiber bundles. To justify this decision, we show their relation to alternative modelling approaches and provide clues to the additional value they bring to the table. 

\subsubsection{Quotient Spaces}

Fiber bundles are related to quotient spaces \citep{orthey_2018, orthey_2019, brandao_2020}, latent spaces \citep{ichter_2019} or sub-spaces \citep{reid_2020} in that we can represent those spaces as the base space of a fiber bundle. We can often create such a base space by taking the quotient of an equivalence class \citep{pappas_2000}. Using the ideas of base spaces, there are two interesting special cases. First, we can use base spaces to simplify a nonholonomic state space to a holonomic state space \citep{sekhavat_1998, vidal_2019}. Often, having a path on the base space is enough to find a global solution, in particular if some sort of smoothness constraint is imposed \citep{vidal_2019, hoenig_2018}. Second, we can use sequences of base spaces to simplify multi-robot planning problems \citep{erdmann_1987, simeon_2002, solovey_2014}. We can often solve such problems efficiently by graph coordination. In graph coordination, we first plan a graph on each individual robot subspace, then we combine them using specialized algorithms like sub-dimensional expansion \citep{wagner_2015} or directional oracles \citep{solovey_2016, shome_2020}. This is different from our approach, in that graphs are constructed independently, while we construct them sequentially, similar to a prioritization of robots \citep{erdmann_1987, berg_2005_prioritized, ma_2019}. 

While numerous works exist to exploit sequences of base spaces~\citep{zhang_2009, vidal_2019}, we like to highlight two algorithms. First, the Manhattan-like rapidly-exploring random tree (ML-RRT)~\citep{cortes_2008, nguyen_2018}, where path sections are computed similar to the L1 interpolation we advocate. However, the ML-RRT approach differs from ours, in that we use a different collision checking function for the base space and we give formal guarantees using restriction sampling. 
Second, the hierarchical bidirectional fast marching tree (HBFMT) algorithm~\citep{reid_2019, reid_2020}, where restriction sampling is used on sequences of subspaces. 
Similar to our approach,~\cite{reid_2020} prove HBFMT to be almost-surely asymptotically optimal by inheritance from BFMT*. 
Our approach is similar in that we develop asymptotically optimal algorithms based on RRT* and PRM*. However, contrary to both~\cite{reid_2020} and~\cite{jaillet_2008}, we use quotient spaces instead of subspaces, we support manifolds instead of only Euclidean spaces, we support multiple robots with nonholonomic constraints, and we provide a variable path bias for restriction sampling (contrary to a fixed path bias~\cite{reid_2020}). We also differ by providing a recursive path section method which we show to quickly find sections even in high-dimensional state spaces.
\subsubsection{Constraint Relaxations}

Fiber bundles are related to constraint relaxations \citep{boyd, roubivcek_2011}, in that we can often model constraint relaxations as a particular type of fiber bundle, i.e. a bundle with an admissible projection, which does not reduce the dimensionality. 
We can often create constraint relaxations by increasing the free space \citep{hsu_2006}, by retracting the obstacle geometry
\citep{saha_2005}, or by shrinking robot links sequentially to zero \citep{baginski_1996}. While constraint relaxations often do not decrease the dimensionality, there are, however, extensions which do decrease the dimensionality, like progressive relaxations \citep{ferbach_1997} or iterative constraint relaxations \citep{bayazit_2005}. In both methods, we either remove links or robots from the problem and we can use them to model the same multilevel abstractions as we can do with fiber bundles. However, by using fiber bundles, we can add additional insights like path sections and restriction sampling.

Closely related to relaxations are projections \citep{sucan_2009, sucan_2011, rowekamper_2013, luna_2020}. 
Projections are a component of fiber bundles, which we use to project the state space onto a lower-dimensional base space. Contrary to quotient spaces, projections are often not required to be admissible but can even be random \citep{sucan_2009}. A noteworthy approach is projection using adaptive dimensionality \citep{vahrenkamp_2008, gochev_2012, gochev_2013}, where projections remove degrees-of-freedom (dof). We can remove dofs of a robot by having a fixed projection \citep{gochev_2012, yu_2020} or by adjusting the projection depending on which links are closest to obstacles \citep{yoshida_2005, kim_2015}. While similar to fiber bundles, both \cite{yoshida_2005} and \cite{kim_2015} emphasize the role of distances in workspace to choose a multilevel abstraction, which is an interesting complementary approach to ours. We differ, however, by supporting multiple robots, nonholonomic constraints, and by providing asymptotic optimality guarantees.
\subsubsection{Admissible Heuristics}

Fiber bundles are related to admissible heuristics \citep{pearl_1984, persson_2014, aine_2016}, in that we can use metrics on the lower-dimensional base space as admissible heuristics \citep{passino_1994} to guide search on the state space. This is closely related to the idea of computing lower bounds for planning problems \citep{salzman_2016}. When using sequences of fiber bundles, we basically use tighter and tighter lower bounds on the real solution. Our approach differs, however, in that we do not consider inadmissible heuristics, which we could combine with admissible heuristics to often speed up planning \citep{aine_2016, tonneau_2018}. 

While there are many ways to define admissible heuristics \citep{aine_2016}, we believe there are two main approaches for the case of continuous state spaces, namely low-dimensional sampling and guide paths. In low-dimensional sampling \citep{sucan_2009}, we first sample on a lower dimensional base space, then use those samples to restrict sampling on the state space. There are two main approaches. First, we can select sequences of subspaces of the state space \citep{xanthidis_2018}, then sample them by selecting the subspaces based on the density of samples. Second, we can use workspace sampling \citep{berg_2005, zucker_2008, rickert_2014, luna_2020}, where state space samples are taken from the restriction of collision-free sets in workspace. We can do workspace sampling by focusing on narrow passages \citep{berg_2005}, or by selecting promising points on the robot and guiding them through the workspace \citep{luna_2020}. Our approach is similar in that we use lower-dimensional sampling on the base space and we select base spaces based on a density criterion \citep{xanthidis_2018}. 
However, we differ by smoothly changing between path and graph restriction sampling, and by using a recursive path section method to efficiently find solution paths.

Closely related to low-dimensional sampling is the concept of guide paths \citep{tonneau_2018, ha_2019}. A guide path is a solution on the base space, which we use to restrict sampling on the state space \citep{palmieri_2016}. Guide paths are often used in contact planning \citep{bretl_2006, tonneau_2018}, where we can often give sufficiency conditions on when a feasible section exists \citep{grey_2017}. When no feasible section exists, some methods fail while other gradually shift towards graph restriction sampling \citep{grey_2017}. It is also possible to compute multiple guide paths which increase our chance to find a feasible section \citep{vonasek_2019b, ha_2019, orthey_2020}. While we also sample along guide paths (path restriction sampling), we differ in two ways. First, we use adaptive restriction sampling to gradually change sampling from path to graph restriction, whereby we guarantee asymptotic optimality. Second, we use a recursive path section method to quickly find feasible path sections in high-dimensional state spaces. 

\subsection{Exploiting Additional Information}

Fiber bundles are a way to exploit additional information. 
Other approaches, complementary to fiber bundles, exists. 
One approach is region-based decomposition. 
In a region-based decomposition, the problem is divided into regions in which planning becomes computationally efficient~\citep{toussaint_2017, orthey_2020}. Such an approach can be done in two ways. 
First, the workspace can be divided~\citep{plaku_2010, vegabrown_2018}, for example using subdivision grids~\citep{plaku_2015}, Delaunay decompositions~\citep{plaku_2010}, or convex regions~\citep{deits_2014, vegabrown_2018}. Second, the solution path space can be divided~\citep{farber_2008}, for example by using the notion of homotopy classes~\citep{munkres_2000}, where two paths are considered to be equivalent if we can deform them into each other. Homotopy classes are closely related to the notion of topological complexity \citep{farber_2017}, the minimal number of regions in state space which are collapsible into a point (null-homotopic). Several practical solutions exists to compute path homotopy classes, like the H-value \citep{bhattacharya_2012, bhattacharya_2018}, simplicial complices \citep{pokorny_2016_ijrr}, task projections \citep{pokorny_2016}, or mutual crossings of robots \citep{mavrogiannis_2016}. However, all those approaches often become computationally intractable for high-dimensional systems, multiple robots, or nonholonomic constraints. Fiber bundles are a complementary effort to organize regions on different levels of abstraction \citep{orthey_2020b}.

Apart from region-based decompositions, we identify three other methods to exploit additional information. First, we can exploit distance information in workspace to compute sets of feasible states \citep{quinlan_1994}, which can be used to plan safe motions \citep{bialkowski_2016}, or to compute covers of free space \citep{lacevic_2016, lacevic_2020}. Second, we can exploit differentiable constraints when available \citep{toussaint_2018, henkel_2020}. Third, we can exploit alternative state space representations, for example by using topology-preserving mappings \citep{zarubin_2012, ivan_2013}. This is complementary to our approach, in that \cite{zarubin_2012} tries to find alternative representations of a state space, while we concentrate on finding simplifications of a given space.

\subsection{Fiber Bundles and Prior Approaches}

Fiber bundle planners exploit a number of projections to accelerate planning performance. Prior approaches using projections are the KPIECE planner~\citep{sucan_2009}, which uses a projection onto a simplified space, and the SBL planner~\citep{sanchez2003single}, which plans using a simplified grid of the state space.
However, most prior approaches are limited in the number of projections~\citep{sucan_2009,
cortes_2008}, the number of robots~\citep{vidal_2019}, use only holonomic
constraints~\citep{zhang_2009}, use only Euclidean spaces~\citep{reid_2019,
reid_2020}, or work only in specific situations~\citep{gochev_2012, kim_2015}. 
Instead, we
can apply fiber bundles to any manifold space (we show it for compound spaces
including the special Euclidean and orthogonal groups in 2d and 3d), any finite
number of projections (up to $98$ in our evaluations), any finite number of
robots (up to $8$ in our evaluations) and any nonholonomic constraint (for
Dubin's state spaces in our evaluations). With fiber bundles, we also provide a
shared vocabulary, which can unify methods like path restriction sampling~\citep{zhang_2009, tonneau_2018,
vidal_2019}, or graph restriction sampling~\citep{grey_2017, orthey_2018,
reid_2020}. Since we also provide an open source implementation in OMPL, we can
benchmark different multilevel strategies
(Appendix~\ref{sec:appendix:metaanalysis}) and we can show the benefit of fiber
bundles compared to classical motion planners (Sec.~\ref{sec:evaluation}). 

\section{Background on Optimal Motion Planning\label{sec:prelims:motionplanning}}

Let $R_1, \ldots, R_M$ be $M$ robots with associated (component) state spaces $\Y_1, \ldots, \Y_M$, respectively. 
We can combine the robots into one generalized robot $R$ with associated (composite) state space $X = \Y_1 \times \cdots \times Y_M$. 

To each state space $X$, we add two complementary structures. First, we add a constraint function $\phi: \X \rightarrow \{0,1\}$ on $\X$ which takes an element $x$ in $\X$ and returns zero if $x$ is feasible and one otherwise. Examples of constraints are joint-limits, self-collisions, environment-robot collisions and robot-robot collisions. Second, we add a steering function $\psi$, which takes two elements $x_1$ and $x_2$ in $\X$ as input and returns a path steering the robot from $x_1$ to $x_2$ (while potentially ignoring constraints).
We denote a state space $X$ together with the constraint function $\phi$ and the steering function $\psi$ as a \emph{planning space} $(X, \phi,\psi)$. The planning space implicitly defines the \emph{free state space} as $\Xf = \{x\in X \mid \phi(x) = 0\}$.

Given a planning space, we define a \emph{motion planning problem} as a tuple $\planningproblem$. To solve a motion planning problem, we need to develop an algorithm to find a path from the initial state $\xi \in \Xf$ to a desired goal region $\Xg \subseteq \Xf$. Often, we are not only interested in some path, but in a path which optimizes a cost functional $c: \X^I \rightarrow \R_{\geq 0}$ whereby $I$ is the unit interval and $\X^I$ is the set of continuous paths from $I$ to $\X$ with finite length \citep{karaman_2011, janson_2018}. We define the \emph{optimal motion planning problem} as finding a path from $\xi$ to $\Xg$ minimizing the cost functional $c$. 

\section{Multilevel Motion Planning\label{sec:bundles}}

Let $\X$ be a state space and let $\X_K \overset{\pi_{K-1}}{\longrightarrow}
\ldots \overset{\pi_1}{\longrightarrow} \X_1$ be a multilevel abstraction of $\X$ such that $\X_K = \X$, and $\pi_{k}$ are projections from a state space $\X_k$ to a state space $\X_{k-1}$. 
Each projection $\pi_{k} : \X_{k-1} \rightarrow \X_k$ is modelled as a fiber bundle (see Sec.~\ref{sec:fiberbundle_formulation}).
Given a start configuration $\xi \in \X_K$, a goal region $\Xg \subseteq \X_K$, and an objective cost
functional $c$, we define the \emph{optimal multilevel motion planning problem}
as the tuple $(\xi, \Xg, \X_1,\ldots,\X_K)$ asking us to find a path from $\xi$
to $\Xg$ while minimizing the cost $c$. Thus, by defining an optimal multilevel
motion planning problem, we generalize optimal motion planning (Sec.~\ref{sec:prelims:motionplanning}) by \emph{adding additional information}. 

In the following sections, we discuss the framework of fiber bundles which provides us with the concepts of bundle restrictions and bundle sections. Those concepts will be used to define primitive methods (Sec.~\ref{sec:primitives}), which are fundamental to create bundle planners which exploit those primitive methods and plan efficiently over fiber bundles (Sec.~\ref{sec:bundleplanners}).

\subsection{Fiber bundle formulation\label{sec:fiberbundle_formulation}}

To model multiple levels of abstractions of state spaces, we use the framework of fiber bundles \citep{steenrod_1951, husemoller_1966, lee_2003}. A fiber bundle is a tuple $(\fiber, \X, B, \pi)$, consisting of the total space $\X$, the fiber space $\fiber$, the base space $B$ and the projection map 
\begin{equation}
    \pi: \X \rightarrow B ~.
\end{equation}
The mapping $\pi$ needs to fulfil two properties:
\begin{enumerate}
    \item \textbf{Union of Fibers}. The total space $\X$ is a (disjoint) union of copies of the fiber space $\fiber$, parameterized by the base space $B$ \citep{lee_2003}. This means that, if we take any element $b$ in $B$, the preimage $\pi^{-1}(b)$ is isomorphic to the fiber space $\fiber$. \label{enum:projection1}
    \item \textbf{Local Product Space}. The total space $\X$ locally equals the product space $B \times F$. This means, if we take any element $b$ in $\base$, there exists a neighborhood $U$ (an open set containing $b$) such that the preimage $\pi^{-1}(U)$ is homeomorphic to $U \times F$ \citep{lee_2003}.\label{enum:projection2}
\end{enumerate}

In other words, a fiber bundle locally has the structure of a product space, and $\pi$ provides a projection from the total space $X$ to a ``parameterization'' of fibers in $B$. This local product structure and the projection $\pi$ aligns with the terms of equivalence classes and quotient spaces as described in Appendix~\ref{sec:appendix:background}. Our main motivation for leveraging the terminology of fiber bundles are the notions we introduce next.

\subsection{Bundle Restrictions}
\label{sec:bundle_restrictions}

\begin{figure}
    \centering

\begin{tikzpicture}
\def\height{2}
\def\radius{1.25}
\def\heightEllipse{0.5}
\def\distBase{1.8}
\def\pathLength{1.5}
\def\vertexStart{-0.5*\pathLength}
\def\pathStretch{2*3.1415/\pathLength}

\def\xposFiber{-2.2*\radius}
\def\xposPath{0*\radius}
\def\xposGraph{2.2*\radius}

\tikzset{
    classical/.style={thin,double,->,shorten >=4pt,shorten <=4pt,>=stealth}
}
\newcommand\drawCylinder[1]{
\draw (#1,\height) ellipse ({\radius} and \heightEllipse);
\draw (#1-\radius,0) -- (#1-\radius,\height);
\draw (#1+\radius,0) -- (#1+\radius,\height);
\draw (#1-\radius,0) arc (180:360:{\radius} and \heightEllipse);
\draw [dashed] (#1-\radius,0) arc (180:360:{\radius} and -\heightEllipse);
\draw (#1,-\distBase) ellipse ({\radius} and \heightEllipse);
\draw[classical] (#1,0-\heightEllipse) -- (#1,-\distBase+\heightEllipse);
}
\newcommand\ppath[1]{
0.13*sin(180*(\pathStretch*(#1-\pathStart))/3.14)
}
\newcommand\drawPath[1]{
\def\pathStart{#1+\vertexStart}
\draw plot [smooth,samples=10,domain=\pathStart:{\pathStart+\pathLength}](\x,{\ppath{\x} - \distBase});
\fill [gray, pattern=vertical lines, opacity=0.8]
plot [smooth,samples=10,domain=\pathStart:{\pathStart+\pathLength}](\x,{\ppath{\x}}) -- 
plot [smooth,samples=10,domain={\pathStart+\pathLength}:\pathStart](\x,{\ppath{\x} + \height});

\foreach \y in {0,0.5,...,\height}
\draw plot [smooth,samples=10,domain=\pathStart:{\pathStart+\pathLength}](\x,{\ppath{\x} + \y});

\draw (\pathStart,0) -- (\pathStart, \height);
\draw (\pathStart+\pathLength,0) -- (\pathStart+\pathLength, \height);
}

\newcommand\drawFiber[1]{
\def\pointPosX{#1+0.5*\radius}
\draw[line width=0.3mm,black](\pointPosX,0) -- (\pointPosX, \height);
\node[draw,fill,circle,inner sep=0pt,minimum size=0.3mm] at (\pointPosX, 0-\distBase){};
\draw node[left] at (\pointPosX, 0-\distBase) {$x_B$};
\draw node[left] at (\pointPosX, 0.5*\height) {$\pi^{-1}(x_B)$};
}

\newcommand\drawGraphCC[2]{
\draw 
(#1+\vertexStart+1.2*\pathLength,#2 - 0.1*\heightEllipse)
-- (#1+\vertexStart+0.9*\pathLength,#2 + 0.3*\heightEllipse)
-- (#1+\vertexStart+0.6*\pathLength,#2 + 0.5*\heightEllipse)
-- (#1+\vertexStart+0.4*\pathLength,#2 - 0.5*\heightEllipse)
-- (#1+\vertexStart+0.75*\pathLength,#2 - 0.4*\heightEllipse)
-- (#1+\vertexStart+0.9*\pathLength,#2 + 0.3*\heightEllipse);
}

\tikzset{
  vertex/.style = {shape=circle,inner sep=0,outer sep=0},
  edge/.style = {->,-Latex},
}
\node[vertex] (a) at (+\vertexStart+1.2*\pathLength, - 0.1*\heightEllipse) {};
\node[vertex] (b) at (+\vertexStart+0.9*\pathLength, + 0.3*\heightEllipse) {};
\node[vertex] (c) at (+\vertexStart+0.6*\pathLength, + 0.5*\heightEllipse) {};
\node[vertex] (d) at (+\vertexStart+0.4*\pathLength, - 0.5*\heightEllipse) {};
\node[vertex] (e) at (+\vertexStart+0.75*\pathLength, - 0.4*\heightEllipse) {};
\node[vertex] (f) at (+\vertexStart+0.9*\pathLength, + 0.3*\heightEllipse) {};

\node[vertex] (g) at (+\vertexStart-0.2*\pathLength, -0.2*\heightEllipse) {};
\node[vertex] (h) at (+\vertexStart+0.2*\pathLength, + 0.1*\heightEllipse) {};
\node[vertex] (i) at (+\vertexStart+0.05*\pathLength, + 0.5*\heightEllipse) {};

\def\edgeList{a/b, b/c, c/d, d/e, e/b, g/h, h/i, i/g}

\newcommand\drawGraph[2]{
\foreach \source/\target in \edgeList
    \draw ($(\source)+(#1,#2)$) -- 
    ($(\target)+(#1,#2)$);
}
\newcommand\drawGraphRestriction[3]{
\foreach \source/\target in \edgeList
    \fill [gray, draw, pattern=vertical lines, opacity=0.8]
    ($(\source)+(#1,#2)$) -- 
    ($(\target)+(#1,#2)$) --
    ($(\target)+(#1,#2)+(0,#3)$) --
    ($(\source)+(#1,#2)+(0,#3)$) --
    ($(\source)+(#1,#2)$);
}

\drawCylinder{\xposFiber}
\drawFiber{\xposFiber}

\drawCylinder{\xposPath}
\drawPath{\xposPath}


\drawGraph{\xposGraph}{-\distBase}
\foreach \y in {0,0.5,...,\height}
   \drawGraph{\xposGraph}{\y};
\drawCylinder{\xposGraph}
\drawGraphRestriction{\xposGraph}{0}{\height};


\end{tikzpicture}
    \caption{Bundle restrictions on fiber bundle $\R^3 \rightarrow \R^2$. Left: Point restriction (Fiber). Middle: Path restriction. Right: Graph restriction.\label{fig:bundle_restriction}}
    
\end{figure}

Given a fiber bundle $(F,\X,B,\pi)$, and a subset $U \subseteq B$, a bundle restriction $\restrict{\X}{U} \subseteq \X$ is the subset of the total space that projects to $U$. This set $\restrict{\X}{U} = \pi^{-1}(U)$ is called the restriction of
$\X$ to $U$~\citep{tu_2017}. 

We consider three kinds of restrictions. 
First, given a point $\x_B$ on the base space, we use its restriction $\restrict{\fiber}{\x_B} = \pi^{-1}(\{\x_B\})$. Note that we call $\restrict{\fiber}{\x_B}$ a fiber as it is, by definition, isomorphic to $F$ (Assertion~\ref{enum:projection1}.). We visualize this in Fig.~\ref{fig:bundle_restriction} (Left). Second, given a path $\path_B:
I \rightarrow B$ on the base space, with $I=[0,1]$ the unit interval, we have
its restriction $\restrict{\X}{p_B} =  \pi^{-1}(\{\path_B(t):t\in I\})$ (Fig.~\ref{fig:bundle_restriction}, Middle). And third, given a graph $\G_B$ on
the base space, we have the graph restriction  $\pi^{-1}(\G_B) \subseteq \X$, where we unproject the union of all its vertices and edges (Fig.~\ref{fig:bundle_restriction}, Right). 

We use these three restrictions for different computations. First, we use point restrictions (fibers) to lift base space elements up to the total space (Sec.~\ref{sec:bundle_sections}). Second, we use path restrictions to quickly compute sections, which are paths on the total space constraint to the path restriction (Sec.~\ref{sec:bundle_sections} and Sec.~\ref{sec:primitives:pathsections}). Third, we use graph restrictions to formulate restriction sampling, i.e. sampling restricted to elements of the total space that project onto the base space graph (Sec.~\ref{sec:restriction_sampling}). It is important to note that restriction sampling is dense in the free total space, if the graph on $B$ is dense. We use this denseness property to prove probabilistic completeness and asymptotic optimality (Sec.~\ref{sec:analysisalgorithms}). 

\subsection{Bundle Sections\label{sec:bundle_sections}}

Given a fiber bundle $(F,X,B,\pi)$ and a subset $U\subseteq B$ of the base space, a section of $U$ is a map $s: U \rightarrow \X$  such that $\pi(s(u)) = u$ \citep{lee_2003}. 
In other words, while a restriction $\restrict{\X}{U}$ unprojects $U$ to \emph{all} elements $x\in X$ that project to $U$, a section maps each $u\in U$ to just \emph{one} specific element $x\in X$ that projects to $u$ (It also follows, that $s(u) \in \restrict{\X}{U}$ for any section $s$.). We define useful cases of sections in the following.

\subsubsection{Lift}

When $U$ contains only a single element $\{b\}$, we call the section a \emph{lift}. A lift $s(b)$ takes as input an element $b$ in $B$ and returns an element $x$ on the total space. We often like to single out a specific element $x$ by additionally choosing a fiber space element $f$, whereby we overload the lift as $s(b,f)$. In the case of $\X$ being a product space, we define the lift as $s(b,f) = (b,f)$. However, if $X \rightarrow B$ is not trivial, the base space element $b$ defines an isomorphic transformation of the fiber space (including the fiber element $f$), which in turn uniquely defines the element $x$. In this work, all bundle spaces are trivial except the Mobius strip. For the Mobius strip, we define the transformation as a linear transformation, involving a translation of the fiber space (here: the unit $[0,1]$ interval) around the circle while simultaneously rotating the fiber space. 

\subsubsection{Path Section}

When the subset $U$ is an interval, we call the section a \emph{path section}. A path section of a path $\path_B: I \rightarrow B$ is itself a path $\path: I \rightarrow \X$ such that $\path_B = \pi \circ \path$.

Our algorithms will aim to find feasible path section, i.e., feasible
unprojections of paths in the base space to paths in the full space. We use
three interpolation methods to this end. All three methods take as input a base
path $\path_B$ and two total space elements $\x_1$ and $\x_2$ in $\X$. Let
$\pi_F$ be the projection of the total space onto the the fiber space (i.e.
orthogonal to the base space projection $\pi$). We then compute fiber space
elements $f_1=\pi_F(\x_1)$ and $f_2=\pi_F(\x_2)$ that introduce coordinates
along the fibers, and which we use to interpolate. Each method differs by how we
interpolate between $f_1$ and $f_2$ along the path restriction
$\pi^{-1}(\path_B)$ (see also Fig.~\ref{fig:bundle_sections}). 

\begin{figure}
    \centering
\def\height{0.4\linewidth}
\def\radiusPoint{2pt}
\def\length{0.8\linewidth}
\def\lineWidth{1pt}

\tikzset{
    border/.style={line width=0.1mm,gray!50, dashed},
    L2/.style={line width=1pt, black, shorten <=0.02*\length, shorten >=0.02*\length},
    L1FF/.style={line width=1pt, black, dash pattern={on 7pt off 3pt}, shorten <=0.02*\length, shorten >=0.02*\length},
    L1FL/.style={line width=1.2pt, black, loosely dotted,
    shorten <=0.02*\length, shorten >=0.02*\length},
    axis/.style={line width=1pt,->}
}
\begin{tikzpicture}

\def\offsetZ{-0.5}
\draw[axis](\offsetZ,0) -- (1.05*\length, 0) node[below] {$B$};
\draw[axis](\offsetZ,0) -- (\offsetZ,  1.1*\height) node[left] {$X$};

\draw[border](0,\height) -- (0.9*\length, \height);

\foreach \x [count=\xi] in {0,0.2*\length, 0.6*\length, 0.75*\length,0.9*\length}
{
    \fill (\x,0) circle[radius=\radiusPoint] node[below] {$b_{\xi}$};
    \draw[border](\x,0) -- (\x, \height);
};

\coordinate (z1) at (0, 0.2*\height);
\coordinate (z2) at (0.9*\length, 0.9*\height);
\coordinate (z2prime) at (0, 0.9*\height);
\coordinate (z1prime) at (0.9*\length, 0.2*\height);

\fill (z1) circle[radius=\radiusPoint] node[left] {$\x_{1}$};
\fill (z2) circle[radius=\radiusPoint] node[right] {$\x_{2}$};
\fill (z2prime) circle[radius=\radiusPoint] node[left] {$\x'_{2}$};
\fill (z1prime) circle[radius=\radiusPoint] node[right] {$\x'_{1}$};
\draw[L2] (z1) -- (z2);
\draw[L1FF] (z1) -- (z2prime) -- (z2);
\draw[L1FL] (z1) -- (z1prime) -- (z2);

\end{tikzpicture}
    \caption{Bundle sections on fiber bundle $\X \rightarrow B$ with base path $\{b_1,b_2,b_3,b_4,b_5\}$. We show three interpolated sections on the bundle space: L2 section (solid line), L1 fiber first section (dashed line) and L1 fiber last section (dotted line).\label{fig:bundle_sections}}
    
\end{figure}

\subsubsection{L2 Section}
To interpolate a section, we can use a straightforward \Ltwo section. To interpolate an \Ltwo section, we use the shortest path under the \Ltwo-norm, which is simply the linear interpolation
\begin{equation}
    l_{\Ltwo}(t) = (1-t) f_1 + t (f_2 - f_1) ~.
\end{equation}
We then compute the section as
\begin{equation}
    \path(t) = s(\path_B(t), l_{\Ltwo}(t))
\end{equation}
by lifting each path base element to the bundle space. We use the \Ltwo section
mainly to compute quotient space metrics (Sec.~\ref{sec:metric_and_steering}). 

\subsubsection{L1 Section}

An alternative to \Ltwo sections are \Lone sections. In an interpolation with an \Lone section, we compute the section as
\begin{equation}
    \path(t) = s(\path_B(t), l_{\Lone}(t))
\end{equation}
with the interpolation
\begin{equation}
    l_{\Lone} = \begin{cases}
        f_1, & \text{if } t < \frac{1}{2}\\
        f_2, & \text{if } t \geq \frac{1}{2}
        \end{cases}
\end{equation}
\noindent We use two flavors of \Lone sections. The first flavor are fiber first (FF) sections, where we use the adjusted base path as
\begin{equation}
    \path_{\text{FF}}(t) =
    \begin{cases}
        \path_B(0), & \text{if } t < \frac{1}{2}\\
        \path_B(2(t-\frac{1}{2})), & \text{if } t \geq \frac{1}{2}
        \end{cases}
\end{equation}
\noindent The second flavor are \Lone fiber last (FL) sections, where we use the base path as

\begin{equation}
    \path_{\text{FL}}(t) =
    \begin{cases}
        \path_B(2t), & \text{if } t < \frac{1}{2}\\
        \path_B(1), & \text{if } t \geq \frac{1}{2}
        \end{cases}
\end{equation}
Both fiber first and fiber last \Lone sections are a cornerstones of our find section method,
which we will use alternately to find feasible sections (see Sec.~\ref{sec:primitives:pathsections} for details).

\subsection{Bundle Sequences\label{sec:bundle_sequences}}

With a fiber bundle, we model a single state space abstraction. To model multiple levels of abstraction, we can chain fiber bundles into sequences. A fiber bundle sequence is a tuple $(\X_{1:K},\fiber_{1:{K-1}},\pi_{1:{K-1}})$ such that the $k$-th base space is equal to the $k-1$-th total space. We write such a sequence as 
\begin{equation}
    \X_K \xrightarrow{\pi_{K-1}} \X_{K-1} \xrightarrow{\pi_{K-2}} \ldots \xrightarrow{\pi_{1}} \X_1
    \label{eq:bundlesequence}
\end{equation}
whereby we call the space $\X_k$ the $k$-th bundle space. 

\subsection{Admissible Fiber Bundles}

An important type of fiber bundles are the ones using admissible projections~\citep{orthey_2018}. An admissible projection is a projection preserving the feasibility of a state. This is important to preserve probabilistic completeness and asymptotic optimality. We next define admissible projections and discuss the corresponding notion of admissible fiber bundles. 

Let $\phi$ and $\phi_B$ be constraint functions on $\X$ and $\B$, respectively.
Given the constraint functions, we can define the free total space $\Xf$ and the
free base space $\Bf$ (see Sec.~\ref{sec:prelims:motionplanning}). For an
admissible projection, we require the projection mapping to fulfill the first
two requirements (Assertions~\ref{enum:projection1} and \ref{enum:projection2} above) plus the following third requirement:

\begin{enumerate}\addtocounter{enumi}{2}
    \item \textbf{Admissible}. The projection mapping does not invalidate solutions. This means, if we map the free state space $\Xf$ via $\pi$ onto the base space, then the image $\pi(\Xf)$ is a subset of the free base space $\Bf$. Or, equivalently, $\phi_B(\pi(x)) \leq \phi(x)$ for any $x \in \X$ \citep{orthey_2019}.
\end{enumerate}

If a projection mapping is admissible w.r.t.\ given $\phi$ and $\phi_B$, we call the fiber bundle an admissible fiber bundle. Analogously, if a sequence of fiber bundles contains only admissible projections, we call it an admissible fiber bundle sequence. It is important to note that admissibility is a requirement, if we like to prove probabilistic completeness or asymptotic optimality. 

Using admissible fiber bundle (sequences), we thus can tie together the notions of quotient spaces, constraint relaxations and admissible heuristics. 
First, we can interpret fiber bundles as a generalization of constraint relaxations~\citep{orthey_2019}, where paths on the base space are lower bound estimates on solution paths on the total space. 
Second, we can use a solution on the base space as an admissible heuristic~\citep{aine_2016} and exploit it by using either restriction sampling, by using a quotient space (base space) metric~\citep{passino_1994, pearl_1984}, or by computing sections along a given base space path~\citep{zhang_2009}. 

\subsection{Examples of Fiber Bundles\label{sec:bundle_examples}}

To make the discussion more concrete, we discuss two (multilevel) abstractions which are often used in motion planning. 

\subsubsection{Prioritized Multi-Robot Motion Planning}

To plan motions for multiple robots, we can prioritize the robots \citep{erdmann_1987, ma_2019}. Given $M$ robots, we can order them, then plan for the first robot and use its motion as a constraint on the motion of the next robot. We can formalize this as a fiber bundle sequence
\begin{equation}
    \Y_{1:M} \xrightarrow{\pi_{M-1}} \Y_{1:M-1} \xrightarrow{\pi_{M-2}} \cdots \xrightarrow{\pi_1} \Y_1
\end{equation}
whereby $\Y_m$ is the state space of the $m$-th robot and $\Y_{1:m}$ is the Cartesian product of the state spaces $\Y_1,\ldots,\Y_m$. In the fiber bundle sequence, we remove, in each step, the configuration space and the geometry of the least important robot. We can then either plan a path in $Y_1$ and use it as a constraint for the next robot (i.e. finding a feasible section in the path restriction). This is known as path coordination \citep{simeon_2002}. Or we use the graph on $Y_1$ to restrict sampling for the remaining robots, which is known as graph coordination \citep{svestka_1998}. In practice, we can realize graph coordination either by using an oracle to guide expansion \citep{solovey_2016} or by expanding dimensionality when conflicts arise \citep{wagner_2015}.

\subsubsection{Tangent Bundle and Path-Velocity Decomposition}

When planning for a dynamical system, we often can simplify planning using a tangent bundle decomposition. Given a state space $\X$ we impose a tangent bundle $\X = TM = M \times \R^{n}$ with projection
\begin{equation}
    M \times \R^{n} \overset{\pi}{\longrightarrow} M
\end{equation}
whereby $n=\dim{M}$, $\R^n$ is the fiber space and $M$ is the base space. We call $M$ the configuration space and $TM$ the tangent bundle. 
Planning on tangent bundles often follows a two step approach. First, we compute a path $p_M$ on the configuration space $M$ (the base space) avoiding obstacles and self-collisions. Second, we compute a velocity along the path, i.e. a time reparameterization. Such a time reparameterization is a path section of $p_M$ and we can find such a section by solving a convex optimization problem  \citep{bobrow_1985}, which we can solve efficiently \citep{pham_2018}. To guarantee completeness, however, we need to either plan on the full tangent bundle $TM$ \citep{lavalle_2001} or track valid speed profiles along paths on $M$ \citep{pham_2017}.

\section{Primitive Methods on Fiber Bundles\label{sec:primitives}}

\begin{figure}[t]
    \centering
    \input{pseudocode/restriction_sampling}
\end{figure}

To exploit fiber bundles, we derive a set of primitive methods. This includes methods to sample the base space,
to compute a metric, to select a bundle
space to grow next, and to rapidly find a feasible section. To implement each method, we develop
several different strategies and discuss how they influence the algorithms. To
select the best strategies for each algorithm, we perform a meta-analysis in
Appendix~\ref{sec:appendix:metaanalysis}. 

To use the primitive methods, we assume that every bundle space $\X_k$ has access to the following fiber bundle specific functions:

\begin{enumerate}
    \item A fiber space $\fiber_k = \X_k / \X_{k-1}$
    \item A base space projection $\pi_k: \X_k \rightarrow \X_{k-1}$
    \item A fiber space projection $\pi^F_k: \X_k \rightarrow \fiber_k$
    \item Projected start state $\xik$ and goal region $\Xgk$
    \item A graph $\G_k = (V_k, E_k)$ containing $|V_k|$ vertices and $|E_k|$ edges
\end{enumerate}

The primitives will be used in the development of the bundle planners (Sec.~\ref{sec:bundleplanners}), where we exploit primitives for improved planning performance.

\subsection{Restriction Sampling\label{sec:restriction_sampling}}

In restriction sampling, we sample states on the total space $\Xk$ by sampling
exclusively in the graph restriction induced by the graph on the base space
$\Xkk$ (see Sec.~\ref{sec:bundle_restrictions}), as we detail in
Alg.~\ref{alg:restriction_sampling}. We first check if the base space $\Xkk$
exists (Line
\algref{alg:restriction_sampling}{alg:restriction_sampling:exists}). If it does
not exists, we revert to a standard sampling method like uniform sampling (Line
\algref{alg:restriction_sampling}{alg:restriction_sampling:nobase}). If it does
exists, we first sample a base space element (Line
\algref{alg:restriction_sampling}{alg:restriction_sampling:samplebase}), then
use it to sample a fiber space element (Line
\algref{alg:restriction_sampling}{alg:restriction_sampling:samplefiber}) and
finally lift the base space element to the bundle space using the fiber space
element (Line \algref{alg:restriction_sampling}{alg:restriction_sampling:lift}).
The lift operation depends on if the bundle is trivial, in which case we just
concatenate base element and fiber element. If the bundle is non-trivial (like
the Mobius strip), we use the base element to index the correct fiber space,
then use the fiber element to index the correct bundle space element (see Sec.~\ref{sec:bundle_sections}).

To implement the \textsc{Sample} function, we use uniform sampling of the space. However, other sampling techniques are certainly possible, like Gaussian sampling \citep{boor_1999}, obstacle-based sampling \citep{amato_1998}, bridge sampling \citep{hsu_2003}, maximum clearance \citep{wilmarth_1999}, quasi-random \citep{branicky_2001}, utility-based \citep{burns_2005}, or deterministic sampling \citep{janson_2018, palmieri_2019}. To guarantee probabilistic completeness and asymptotic optimality, we only need to verify that those sequences are dense. 

The main method of restriction sampling is the \textsc{SampleBase} method. In the \textsc{SampleBase} method, we sample the graph $\Gkk$ on the base space. While numerous methods exist to sample a graph
\citep{leskovec_2006}, we found five methods particularly important.

\subsubsection{Random Vertex Sampling}
First, we can chose a vertex at random, which we refer to as Random-Vertex (RV)
sampling \citep{leskovec_2006}. In RV sampling, we choose a random integer
between $1$ and $|V|$ which uniquely defines a vertex on the graph $\G$. This
sampling is particularily fast ($O(1)$ operations), but might be overly
constrictive if we sample from a tree or a graph with long edges. However, for large
graphs, this sampling procedure is often the only alternative to not slow down
sampling.

\subsubsection{Random Edge Sampling}

Second, we can choose an edge at random, then pick a state on this edge, a
method we refer to as Random-Edge (RE) sampling \citep{leskovec_2006}.
This method requires two operations, first to pick an edge, then to pick a
number between $0$ and $1$ to determine the state on the edge. This method seems to be superior if the graph
is sparse and has long edges, in particular edges going through narrow passages. 

\subsubsection{Random Degree Vertex Sampling}

Third, we can choose a vertex at random, but biased towards vertices with a low
degree (number of outgoing edges). We refer to this as Random-Degree-Vertex
(RDV) sampling. 
With RDV sampling, we bias samples to vertices which are either in tight corners or inside of narrow passages. 
Vertices in large open passages often have many neighbors and thereby a large degree. 
This method, however, requires to update a probability function which tracks the degrees of each vertex. 

\subsubsection{Path Restriction Sampling}

\def\pathBiasFixed{\beta_{\text{fixed}}}
\def\pathBiasDecay{\beta_{\text{decay}}}
Fourth, we can choose a sample on the lowest cost path on the graph, a method we
refer to as \emph{path restriction} (PR) sampling. We can utilize PR sampling in two
ways. Either, we sample on the path restriction with a fixed probability $\pathBiasFixed$. This is similar to the fixed tunnel radius proposed by \cite{reid_2019}. Or, we first sample exclusively on the path restriction, then gradually decay towards the fixed path bias. We call this method PR decay sampling.

PR decay sampling allows us to model a change in belief. It is often true that
the shortest path on the base space contains a feasible section, which we should
search for by exclusively sampling on the path restriction \citep{orthey_2018}.
If we do not find a valid section, we should gradually dismiss our belief that a
section exists and try to sample the graph restriction instead. To model this
change in belief, we use an exponential decay function to smoothly transition
from probability $1$ down to the fixed probability $\pathBiasFixed$ using a
decay constant $\lambda$. See Appendix~\ref{sec:appendix:expdecay} for the definition of exponential decay. 

Before using PR decay sampling, we simplify the path. Simplifying the path is similar to the local path refinement method \citep{zhang_2009}, where a path is optimized to increase its clearance. For this operation, we use a simple short-cutting path optimizer, which does not slow down planning in high-dimensional spaces.

We use a path optimizer with a cost term for path length.

\subsubsection{Neighborhood Sampling}

Fifth, we can choose a sample not directly on the graph, but in an epsilon
neighborhood. We refer to this as \emph{neighborhood} (NBH) sampling. NBH
sampling is helpful when there is a path through a narrow passage
which comes close to its boundary. Those paths often do not have a feasible section. Instead, if we would perturbate the path slightly, we can often find a path admitting a feasible section. With NBH sampling, we first sample a configuration $x$
exactly on the graph, and then sample a second configuration $x'$ which we
sample uniformly in an epsilon ball around $x$. In practice, we use an
exponential decay (Appendix~\ref{sec:appendix:expdecay}) to smoothly vary the size of the neighborhood from zero up to epsilon. With NBH sampling, we can often solve problems where a solution through a narrow passage has few or no
samples, while using nearby samples allows us a bit more wriggle room. Note that instead of uniform epsilon sampling, we could also use a Gaussian distribution with mean $x$ and epsilon variance \citep{reid_2019}. However, in preliminary testing, we could not observe a difference between them.

\subsection{Bundle Space Metric\label{sec:metric_and_steering}}

An essential component of bundle algorithms are the nearest neighbor computations, which depend on choosing a good metric function. We discuss two possible metrics, the intrinsic bundle metric (ignoring the base space) and the quotient space metric (exploiting the base space).

\subsubsection{Intrinsic Bundle Metric}

To straightforwardly attach a metric to the bundle space, we use the geodesic distance between two points while ignoring the base space. We compute this intrinsic metric on $\X$ as

\begin{equation}
    d(x_1, x_2) = d_\X(x_1, x_2)
\end{equation}

While this is a naive way to compute the metric, we note that using base space information is often costly, and the total space metric is often good enough \citep{orthey_2019}.

\subsubsection{Quotient Space Metric}

If a base space is available, we can consider it as a quotient space, on which
we can define a quotient space metric \citep{guo_2019}. To define a quotient
space metric between two states, we first project both states onto the base
space, compute a shortest path $p_B$ using the base graph and then interpolate
an \Ltwo section along the path restriction $\restrict{\X}{p_B}$ (see Sec.~\ref{sec:bundle_sections}). 

In particular, given two points $x_1$ and $x_2$ in $\Xk$, we project them onto
the base space $\Xkk$ to yield $b_1 = \pi(x_1)$ and $b_2 = \pi(x_2)$. We then
compute the nearest vertices $v_1$ and $v_2$ on the graph $\G_{k-1}$ and we
compute a path on $\G_{k-1}$ between $v_1$ and $v_2$ using the A* algorithm with
the intrinsic base space metric as an admissible heuristic. Finally, we use the
fiber space projection of $x_1$ and $x_2$ to compute fiber space elements $f_1 =
\pi_F(x_1)$ and $f_2 = \pi_F(x_2)$, which we use to integrate an \Ltwo section
(Sec.~\ref{sec:bundle_sections}). We then compute the bundle space metric as
\begin{equation}
    d(x_1, x_2) = 
    \begin{cases}
    d_\X(x_1, x_2) & v_1 = v_2\\
    \begin{aligned}
        &d_F(f_1, f_2)  \\ &+ d_B(y_1, v_1) \\&+ d_B(y_2, v_2) \\&+ d_{\Gk}(v_1, v_2)
    \end{aligned} & \text{otherwise}
    \end{cases}
\end{equation}
with $d_F$ being the fiber space metric (\Ltwo), $d_B$ the base space metric and $d_{\Gk}$ the length of the shortest path on $\Gk$ between vertices under the base space metric. 

While the quotient space metric is more mathematically sound, there are two practical problems. First, computing this metric is costly, because we need to perform a graph search operation. Second, the graph on the base space might not yet be dense, thereby potentially returning values leading to an inadmissible heuristic, which in turn would mislead the planner on the bundle space. 

\subsection{Bundle Space Importance\label{sec:importance}}

In each iteration of multilevel motion planning, we make a choice about expanding a graph by selecting a level. To select a level, we attach an importance function to each bundle space, which we use to rank the bundle spaces. We develop three different importance strategies.

\subsubsection{Uniform\label{sec:importance_uniform}}

In uniform importance, we select all bundle spaces an equal amount of times. This is similar to round-robin change, similar to scheduling operations \citep{russell_2002}. Here we use a slight variation, where we compute the importance based on the number of vertices, thereby ensuring a uniform expansion of each level. In particular, for bundle space $\Xk$ with graph $\Gk$ and $|V_k|$ vertices, we compute its importance as $\frac{1}{|V_k|+1}$.

\subsubsection{Exponential\label{sec:importance_exponential}}

To densely cover spaces with higher dimensions, we usually require more samples. In general, the sampling density is proportional to $N^{\frac{1}{d}}$ where $N$ is the number of samples and $d$ the dimensionality \citep{hastie_2009}. 
Therefore, we should select the space with the lowest density first, thereby guaranteeing equal sampling density across all spaces. 
We can compute an exponential importance as $\frac{1}{|V_k|^{1/d}+1}$ which reflects an exponential increase of samples in higher dimensions. 
This idea is similar to the selection of bundle spaces using a geometric progression \citep{xanthidis_2018}. 
This is also related to multilevel monte carlo \citep{giles_2015} and sparse grid methods \citep{bungartz_2004}.

\subsubsection{Epsilon Greedy\label{sec:importance_greedy}}

Whenever we find a graph connecting initial and goal state on the base space, it
seems reasonable to greedily exploit this graph to find a path on the bundle
space. This strategy is not complete, since the graph might not yet contain a
feasible section (see Sec.~\ref{sec:bundle_restrictions}). We can, however, create a complete algorithm by extending the base space with an epsilon probability while extending the bundle space the rest of the time. We compute this as
\begin{equation}
    f(k) = 
    \begin{cases}
    \epsilon^{K-k} - \epsilon^{K-k+1} & k>1\\
        \epsilon^{K-1} & \text{otherwise}
    \end{cases}
\end{equation}
whereby $k$ is the bundle space level and $K$ is the total number of bundle spaces. We then compute the importance for the $k$-th bundle space as $\frac{1}{|V_k|/f(k)+1}$, reflecting our desire to expand recent levels more aggressively. 

\begin{figure}[t]
    \centering
    \input{pseudocode/hasfeasiblesection}
\end{figure}

\subsection{Finding Path Sections\label{sec:primitives:pathsections}}

Finding path sections quickly and reliably is one of the cornerstones of all
bundle planners. In this section, we use the interpolation methods of
Sec.~\ref{sec:bundle_sections} to develop a recursive path section algorithm,
which we depict in Alg.~\ref{alg:section}. For this to work, we need to have at least a base space (Line \algref{alg:section}{alg:section:exist}). We then compute the shortest path on the base space (Line \algref{alg:section}{alg:section:path}) and recursively compute a section, either by starting from an L1 fiber-first section (Line \algref{alg:section}{alg:section:L1FF}) or if unsuccessful, by starting from an L1 fiber-last section (Line \algref{alg:section}{alg:section:L1FL}). 

To recursively compute a section, we show the pseudocode in Alg.~\ref{alg:recursivesection}. We terminate the algorithm if we reach a certain depth $\dmax$ (Line \algref{alg:recursivesection}{alg:recursivesection:terminatedmax}) or if we reach the goal region (Line \algref{alg:recursivesection}{alg:recursivesection:terminategoal}). Inside each recursion iteration, we interpolate an L1 section, either fiber first (if FF is true) or fiber last (if FF is false) (Line \algref{alg:recursivesection}{alg:recursivesection:interpolate}). We then propagate the system along the section while valid (Line \algref{alg:recursivesection}{alg:recursivesection:propagatewhilevalid}) and return the last valid state. 

If we do not reach the goal state with the last valid state, we do up to $\bmax$ sidesteps along the fiber space. Sidestepping means that we project the last valid state onto the base space (Line \algref{alg:recursivesection}{alg:recursivesection:projectbase}), then sample a random fiber space element (Line \algref{alg:recursivesection}{alg:recursivesection:samplefiber}) and lift the states to the bundle space to obtain a state $\xk$ (Line \algref{alg:recursivesection}{alg:recursivesection:lift}). We then check if we can move from the last valid state to the state $\xk$ (Line \algref{alg:recursivesection}{alg:recursivesection:checkmotion}). Since both states have the same base space projection, we call this a \emph{sidestep} (i.e. a step orthogonal to the base space). If the motion is valid, we clip the remaining base path (Line \algref{alg:recursivesection}{alg:recursivesection:getsegment}) and recursively call the algorithm (Line \algref{alg:recursivesection}{alg:recursivesection:recursion}). In the recursion call, we increase the depth, use the clipped base path segment and change the interpolation method from fiber first to fiber last. We change the interpolation at this point, because we observe an alternation between interpolation methods to substantially improve runtime. 

\subsubsection{Nonholonomic Constraints}

In the case of holonomic constraints, we can use the \Lone interpolation (Line \algref{alg:recursivesection}{alg:recursivesection:interpolate}) and the base space segment (Line \algref{alg:recursivesection}{alg:recursivesection:getsegment}) to follow the path restriction exactly. However, if we have nonholonomic constraints, we often cannot follow the path restriction exactly, in particular if the base space path is piece-wise linear. Note that a base space path is often piece-wise linear if we do not impose additional smoothness assumptions \citep{vidal_2019, hoenig_2018}. 

To still compute path sections over piece-wise linear base space paths in the
nonholonomic case, we do a two-phase approach. First, we compute the
interpolation values as in Sec.~\ref{sec:bundle_sections}, but only at discrete
points, which provides us with a set of points on the bundle space. Second, we
interpolate between those points by using the nonholonomic steering function.
While we might deviate from the base path restriction, we follow, however, the
base path restriction as close as the steering function allows us. This approach
is similar to the idea of interpolating waypoints with dynamically feasible path
segments, which has been done for flying quadrotors \citep{richter_2016} and for
underwater vehicles \citep{yu_2019}. However, we differ by first interpolating
values for the fiber spaces along the base space path. The remaining computation
in Alg.~\ref{alg:recursivesection} remains exactly as in the holonomic case.

\section{Bundle Space Motion Planners\label{sec:bundleplanners}}

\begin{figure}[t]
    \centering
    \input{pseudocode/bundleplanner}
\end{figure}

To solve a multilevel motion planning problem, we develop a set of algorithms
generalizing existing motion planners to fiber bundles. 
All those planners share
the same high-level structure, which we call a \textsc{BundlePlanner} (Alg.~\ref{alg:bundleplanner}). 
In the \textsc{BundlePlanner} method, we first initialize a priority queue sorted by the importance of each bundle space (Line \algref{alg:bundleplanner}{alg:bundleplanner:priorityqueue}). We then iterate over all bundle spaces, try to find a section on the $k$-th bundle space (Line \algref{alg:bundleplanner}{alg:bundleplanner:findsection}) and then push the $k$-th bundle space into the priority queue (Line \algref{alg:bundleplanner}{alg:bundleplanner:pushk}). We then execute the while loop while a planner terminate condition (\textsc{PTC}) is not fulfilled for the $k$-th bundle space (Line \algref{alg:bundleplanner}{alg:bundleplanner:while}). Inside the loop, we select the most important bundle space, grow the graph or tree and push the space back into the queue (Line \algref{alg:bundleplanner}{alg:bundleplanner:popselect} to \algref{alg:bundleplanner}{alg:bundleplanner:pushselect}). We terminate if the \textsc{PTC} for the $K$-th bundle space has been fulfilled. This means we either terminate successfully, found the problem to be infeasible or reach a timelimit. 

All bundle space algorithms are alike in sharing the same high-level structure;
each bundle space algorithm differs in their \textsc{Grow} function (Line
\algref{alg:bundleplanner}{alg:bundleplanner:growselect}) and their primitive
methods (Sec.~\ref{sec:primitives}). 

\subsection{Bundle Planner Variants\label{sec:concretebundleplanner}}

The \textsc{BundlePlanner} algorithm is used to develop novel algorithms by changing the \textsc{Grow} function. To implement the \textsc{Grow} function, we can utilize almost any single-level planning algorithm. In our case, we use the algorithms RRT, RRT*, PRM and PRM* (please consult Tab.~\ref{tab:list_of_algorithms} for abbreviations of algorithms). 

All grow functions in a multilevel versions of our algorithms differ from their single-level version in four points. First, we replace uniform sampling by restriction sampling, as we detail in Sec.~\ref{sec:restriction_sampling}. Algorithms might differ in how we implement graph sampling in restriction sampling. Second, when pushing a new bundle space into the priority queue, we check for a feasible section over the solution path on the last bundle space, as we detail in Sec.~\ref{sec:primitives:pathsections}. This computation is equivalent for each bundle planner. Third, we rank bundle spaces based on a selection criterion, which we detail in Sec.~\ref{sec:importance}. Algorithms might differ in the type of selection criterion we employ. Fourth, we adjust the metric on the bundle space, which affects both nearest neighbors computation and the steering method, as we detail in Sec.~\ref{sec:metric_and_steering}. While different metrics are possible \citep{orthey_2018}, we use the intrinsic bundle metric for all algorithms (as determined by our meta-analysis in Appendix~\ref{sec:appendix:metaanalysis}).
\begin{figure}
    \centering
    \input{pseudocode/qrrt}
    \input{pseudocode/qrrt_star}
    \input{pseudocode/qmp}    
\end{figure}

\subsection{QRRT}

In Alg.~\ref{alg:qrrt_grow}, we show the QRRT algorithm. We previously
introduced QRRT in \cite{orthey_2019}. We differ here by using an exponential
importance primitive (Sec.~\ref{sec:importance_exponential}) and by adding the
find section primitive (Sec.~\ref{sec:primitives:pathsections}). The remaining
structure, however, remains unchanged. In detail, we sample a random element
from the bundle space
(Line~\algref{alg:qrrt_grow}{alg:qrrt_grow:restrictionsampling}) using
restriction sampling (Sec.~\ref{sec:restriction_sampling}). We then choose the nearest vertex from the tree (Line~\algref{alg:qrrt_grow}{alg:qrrt_grow:nearest}) and steer from the nearest to the random element (Line~\algref{alg:qrrt_grow}{alg:qrrt_grow:steer}). We then check if the motion is collision-free and add the new state to the tree. Note that we stop steering if the distance goes above a threshold, similar to RRT \citep{lavalle_2001}.

\subsection{QRRT*}

While QRRT performs well in our evaluations, we can improve upon QRRT by developing an asymptotic optimal version. We call this QRRT* and depict the algorithm in Alg.~\ref{alg:qrrtstar_grow}. By developing QRRT*, we generalize RRT* \citep{karaman_2011} to multiple levels of abstraction. 
To implement QRRT*, we use one iteration of QRRT (Line~\algref{alg:qrrtstar_grow}{alg:qrrtstar_grow:qrrt}), then compute $k$ nearest neighbors of the new state (Line~\algref{alg:qrrtstar_grow}{alg:qrrtstar_grow:nearest}). We choose the $k$ as $k = \krrt \log(N)$ whereby $N$ is the number of vertices in the tree \citep{karaman_2011}. The parameter $\krrt$ can be chosen based on the dimension of the problem \citep{karaman_2011, kleinbort_2019}. 

After computing $k$ nearest neighbors, we perform two rewire operations (this dicussion follows closely \cite{salzman_2016}). First, we rewire the nearest neighbors to the new state (Line~\algref{alg:qrrt_grow}{alg:qrrtstar_grow:rewireone}). Second, we rewire the new state to the nearest neighbors (Line~\algref{alg:qrrt_grow}{alg:qrrtstar_grow:rewiretwo}). We show the rewire operation in Alg.~\ref{alg:rewire}. Inside the rewire algorithm, we update the incoming edge of state $y$ by checking if the cost to come from state $x$ (cost from initial state to $x$) plus the cost to go from $x$ to $y$ is smaller than the cost to come for state $y$. In that case, we update the graph by removing all incoming edges into $y$ and adding a directed edge from $x$ to $y$. Contrary to similar implementations \citep{karaman_2011, salzman_2016}, we also update the tree $\Gk$ such that we can use the same restriction sampling method for each algorithm. While the grow method is similar to the RRT* method \citep{salzman_2016}, we note that much of the complexity is encapsulated in the primitive methods (Sec.~\ref{sec:primitives}), which we use to sample, to compute distances, to find sections and to choose a bundle space to grow next.

\subsection{QMP}

In Alg.~\ref{alg:qmp}, we show the QMP algorithm, which we introduced in
\cite{orthey_2018}. In the QMP algorithm, we differ from QRRT by not growing a
tree, but a graph \citep{kavraki_1996}. QMP generalizes PRM in the sense that
QMP becomes equivalent to PRM when we choose a single-level abstraction. The
algorithm QMP as presented here differs slightly from its original conception
\citep{orthey_2018} in three points. First, we use the epsilon greedy importance
(Sec.~\ref{sec:importance_exponential}) instead of uniform importance to select
a bundle space to expand. Second, we use the intrinsic bundle metric
(Sec.~\ref{sec:metric_and_steering}) instead of the quotient space metric, which
we found to not scale well to high-dimensional state spaces (see Appendix~\ref{sec:appendix:metaanalysis}). Third, we use the \textsc{FindSection} method to quickly check for sections (Sec.~\ref{sec:primitives:pathsections}). 

\subsection{QMP*}

QMP* is similar as QMP, but we use a different $k$ in each iteration to chose the nearest neighbors. This $k$ is chosen such that the resulting algorithm is almost-surely asymptotically optimal \citep{karaman_2011}. In general we use $k = \kprm \log(N)$ with $N$ being the number of vertices in the graph. See also \cite{solovey_2020} for recent developments on choosing the parameter $\kprm$.

\subsection{Open Source Implementation\label{sec:algorithms:implementation}}

To make the algorithms freely available, we provide implementations in C/C++, which we split into two frameworks. The first framework is a graphical user interface (GUI) where users can specify fiber bundles by providing URDF (Unified Robotic Description Format) files for each level and specify the bundle structure in an XML (Extensible Markup Language) file. We then provide functionalities to step through each level and to visualize the lowest-cost path on each level. The code is freely available on github\footnote{\url{https://github.com/aorthey/MotionExplorer}}.

The second framework is the actual implementation of fiber bundles, bundle algorithms, and primitives, which we implement as a submodule of the Open Motion Planning Library (OMPL) \citep{sucan_2009}. In particular, we encapsulate our code as an OMPL planner class, which we can use for benchmarking \citep{moll_2015} or analysis. 
This code is part of OMPL version 1.6.0 and includes a high-level introduction, a tutorial, and additional demos\footnote{\url{https://ompl.kavrakilab.org/multiLevelPlanning.html}}. 

\section{Analysis of Bundle Planners\label{sec:analysisalgorithms}}

Let $\X_K \xrightarrow{\pi_{K-1}} \ldots \xrightarrow{\pi_{1}} \X_1$ be a fiber bundle sequence. We like to prove that, on this fiber bundle sequence, the algorithms QRRT, QRRT*, QMP, and QMP* are probabilistically complete (PC) and that QRRT* and QMP* are asymptotically optimal (AO).

To prove those properties, we use two methods. First, we state three assumptions on the importance function and the datastructures, which we use to establish that restriction sampling is dense. Second, we argue that the bundle algorithms, when using restriction sampling, inherit the PC and AO properties from their single-level counterpart. 

\subsection{Assumptions} 

We require three assumptions to hold true. 

\def\initialcomponent{connected initial component}

\begin{enumerate}
    \item The importance function of each bundle space (Sec.~\ref{sec:importance}) monotonically converges to zero (we select every bundle space infinitely many times)
    \item Restriction sampling is dense in $\X_1$
    \item If restriction sampling is dense, the graph on the $k$-th bundle space is space filling in the \initialcomponent
\end{enumerate}

whereby the \emph{\initialcomponent} is the set of points in $\Xk$ which are path-connected\footnote{We say that two states are path-connected if there exists a continuous path connecting them.} to $\pi_k(x_I)$, i.e. to the projection of the initial state onto the $k$-th bundle space. A graph is said to be space-filling in a set $U$, if for any $x$ in $U$ there exists a path in the graph starting at $x_I$ and converges to $x$ \citep{kuffner_2011} (in the limit when running time goes to infinity). 

\subsection{Proof that Restriction Sampling is Dense}

When stripping down to the essentials, we observe that the bundle planners differ on the last level from non-multilevel planners by replacing uniform sampling with restriction sampling. While uniform sampling is dense in the \emph{complete} state space, restriction sampling differs, in that we can prove it to be dense in the \initialcomponent. 

To prove denseness, we need some notations. 
First, a set $U$ is dense in $\X$ if the intersection of $U$ with any non-empty open subset $V$ of $X$ is non-empty \citep{munkres_2000}. We abbreviate this by saying that a set is dense if its \emph{closure} $cl(U)$, the smallest closed set containing $U$, contains the space $\X$. When using a sequence of samples $\alpha_1,\alpha_2, \ldots$, we can interpret the sequence as a set $A = \{\alpha_i \mid i \in \N\}$. We can then say that the sequence is dense in the space $\X$ if the closure $cl(A)$ contains $\X$ (or is equal to). 

\def\Xkinit{I_k}
\def\Xkinitone{I_1}
\def\Xkkinit{I_{k-1}}
\def\alphak{A_k}
\def\alphaone{A_1}
\def\alphakk{A_{k-1}}

Let $\Xkinit$ be the \initialcomponent (on the bundle space $k$) and let $\alphak$ be a restriction sampling sequence. To prove $\alphak$ to be dense in $\Xkinit$, we choose an arbitrary set $U$ in $\Xkinit$. We then prove that there will be a non-empty intersection of $U$ with $\alphak$, i.e. given enough time, we will at least sample once from $U$. Our proof is inductive, i.e. we prove it to be true for $k=1$, then use this to inductively argue for arbitrary $k$.

In a preliminary version of the proof \citep{orthey_2019}, we showed restriction sampling to be dense in the free state space, which is true only if there is a single connected component. To make the proof more general, we replace the free state space here with the \initialcomponent. 

\begin{theorem}
$\alphak$ is dense in $\Xkinit$ for $k \geq 1$.
\label{theorem:dense_sampling}
\end{theorem}

\begin{proof}
By induction for $k=1$, $\alphaone$ is dense in $\X_1$ by assumption and therefore dense in $\Xkinitone$ since $\Xkinitone \subseteq \X_1$. For the induction step, we can assume $\alphakk$ to be dense in $\Xkkinit$. Let $U$ be a non-empty open subset of $\Xkinit$. 
Since $U$ is open, $\pi_{k-1}(U)$ is open (by property of fiber bundle). 
By induction assumption there exists a $y$ in $\alphakk \cap \pi_{k-1}(U)$.
Consider an open set $V$ of the preimage $\pi^{-1}_{k-1}(y)$. Since $\alphak$ is dense in $\pi^{-1}_{k-1}(\alphakk)$ (by definition of restriction sampling), there exists an $x$ in $\alphak \cap V$ which is a subset of $U$. Since $U$ was arbitrary, $\alphak$ is dense in $\Xkinit$.
\end{proof}

Due to Theorem~\ref{theorem:dense_sampling}, we observe that restriction sampling differs from uniform sampling by removing states which cannot be feasible. Therefore, algorithms using restriction sampling maintain all their properties, which we can inherit. 

\subsection{Inheritance of Probabilistic Completeness}

A motion planning algorithm is probabilistically complete, if the probability that the algorithm will find a path (if one exists) goes to one as time goes to infinity. This property has been proven for sampling-based planners, in the case of a graph \citep{svestka_1996} including the case of a tree \citep{kuffner_2000}.

Probabilistic completeness follows in our case directly from the assumptions and our proof that restriction sampling is dense. In particular, let us assume a given motion planning problem to be feasible and containing a solution in the interior of the free space. Since restriction sampling is dense, by assumption, we have a space-filling graph containing a path starting at the initial state and converging to the goal state. 

In the grow functions of QRRT, QRRT*, QMP and QMP*, we directly implement the
corresponding versions of RRT, RRT*, PRM and PRM*, which all have been shown to
be probabilistically complete (see corresponding papers listed in Tab.~\ref{tab:list_of_algorithms}). They therefore necessarily need to construct a space-filling graph (tree) \citep{kuffner_2011} and all bundle space planners, when using restriction sampling, inherit the probabilistic completeness property. 

\subsection{Inheritance of Asymptotical Optimality}

An algorithm is (almost surely) asymptotically (near-) optimal (AnO) \citep{karaman_2011, salzman_2016} if it converges to  a cost at most $(1+\epsilon)$ times the cost of the optimal path. An algorithm is (almost surely) asymptotically optimal if it is AnO with $\epsilon=0$. 

Similar to probabilistic completeness, we argue that QRRT* and QMP* are asymptotically optimal, since this property is inherited from RRT* and PRM* \citep{karaman_2011}, respectively. This is true, since on the last level, we only change the sampling function from uniform to restriction sampling. Since we showed restriction sampling to be dense and we will select the last bundle space infinitely many times, we can be sure that the optimality properties are kept intact. Note that this line of reasoning is slightly different from the proof of asymptotic optimality for HBFMT \citep{reid_2020}, where \cite{reid_2020} define a probability $l$ with which they switch to use uniform sampling, thereby guaranteeing optimality by actually reverting to BFMT. We, however, rely on the denseness property of restriction sampling, thereby avoiding an uniform extension step.

Detailed proofs of asymptotic optimality for sampling-based planner can be found in \cite{karaman_2011}. See also \cite{salzman_2016} and \cite{solovey_2020} for a treatise of asymptotic near-optimality.

\section{Evaluation\label{sec:evaluation}}

To show the wide applicability of fiber bundles and bundle algorithms, we apply them to a broad range of planning scenarios. 
In particular, we evaluate our algorithms on four low-dimensional and eight high-dimensional planning scenarios, including computer animation, pre-grasping, multi-robot coordination, and non-holonomic constraints. 
The dimensionality of the state spaces in the high-dimensional case ranges from 21-dof (box folding) to 100-dof (hypercube). We compare our algorithms with available algorithms implemented in the Open Motion Planning Library (OMPL) as of May 2020 \citep{moll_2015}. 
References and details of those algorithms are shown in Tab.~\ref{tab:list_of_algorithms}. 
All algorithms, except QMP, QMP*, QRRT, and QRRT*, do not use the additional information which fiber bundles provide. We like to show that fiber bundles are helpful to solve scenarios which are near unsolvable using classical sampling-based methods~\cite{kavraki_1996, kuffner_2000}

\noindent\textbf{Evaluation Metrics.} For all scenarios, we let
each algorithm run $10$ times with a cut-off time limit of $60$s. For the low-dimensional scenarios, we report a success-cost plot showing convergence rate and success rate over time. 
In this case, we let the algorithms run for $60$s and query their current best cost with a $100$Hz update frequency (i.e. every $0.01$s). For the high-dimensional scenarios, we run two separate evaluations. 
First, a pure runtime evaluation, where we compare the average runtime on each scenario, comparing against all available OMPL planners. In this case, planners run until they either find a solution or the cut-off time limit has been reached. 
Second, we report on a success-cost plot for our algorithms against four well performing algorithms from OMPL, namely RRTConnect, RRT*, BIT*, and LBTRRT. In this case, all algorithms are run for $60$s with best cost queries at $100$ Hz update frequency. 

\noindent\textbf{Hardware.} 
Concerning hardware, we use a 4-core laptop
running Ubuntu 18.04 with $20$GB of RAM to run the runtime evaluation on the high-dimensional planning problems.
For the low-dimensional planning problems and the cost function evaluation on the high-dimensional problems, we use a 4-core laptop running Ubuntu 16.04 with $8$GB of RAM.
Concerning parameters, our algorithms are
set as follows. 
For the \textsc{FindSection} method, we use $\dmax = 3$ and
$\bmax=10$. 
For path restriction sampling, we use the decay constant $\lambda =
\num{1e-3}$ and fixed probability $\pathBiasFixed=0.1$. 
For QRRT, we use a
maximal distance range of $0.2\mu$ whereby $\mu$ is the measure of the space
(same value as in RRT or RRTConnect). 
For QMP, we use $k=10$ to compute nearest
neighbors (same as in PRM). 
For QMP*, we use the optimal number of nearest
neighbors in each iteration as in PRM* \citep{solovey_2020}. 
The choice of
primitive methods has been independently optimized using a meta analysis (See
Appendix~\ref{sec:appendix:metaanalysis}). 
We set any other parameters to be equivalent to the corresponding single-level planner. 

\noindent\textbf{State Spaces.}
In all scenarios, the state spaces of the robots are modelled using the following mathematical spaces. For rigid bodies, we use $SE(2)$ and $SE(3)$, the special Euclidean group in two and three dimensions, respectively~\citep{selig_2004}. The spaces in those groups model all rotations and translations applicable to a rigid body in two or three dimensions. For rotating joints, we use $SO(2)$ and $SO(3)$, the special orthogonal group. The spaces in those groups model all rotations about a fixed point of the robot. For all other kinematic chains with rotational joints and joint limits, we use the Euclidean space $\R^n$ of $n$ dimensions.

\begin{figure*}
    \centering
    \includegraphics[width=0.48\textwidth]{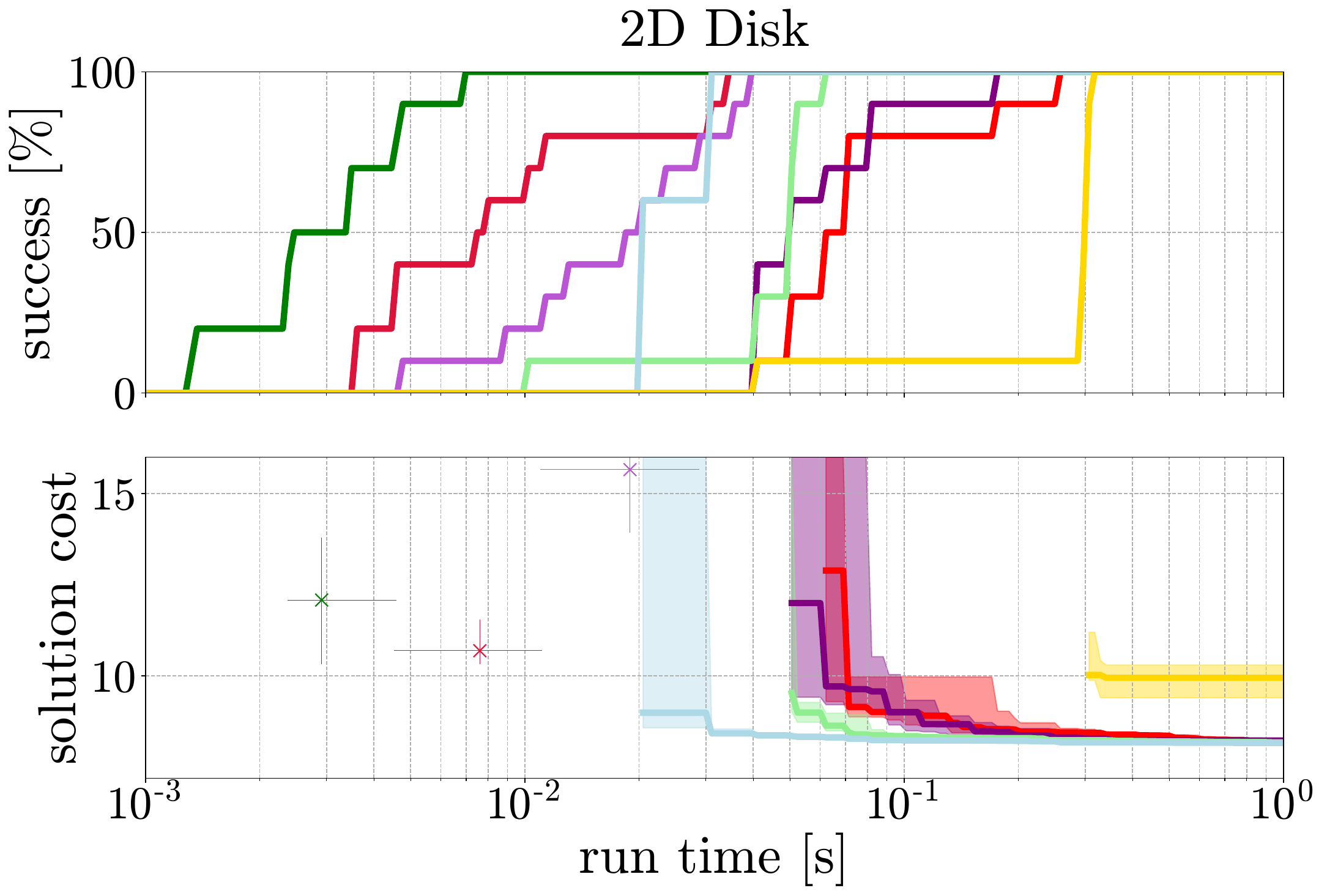}
    \includegraphics[width=0.48\textwidth]{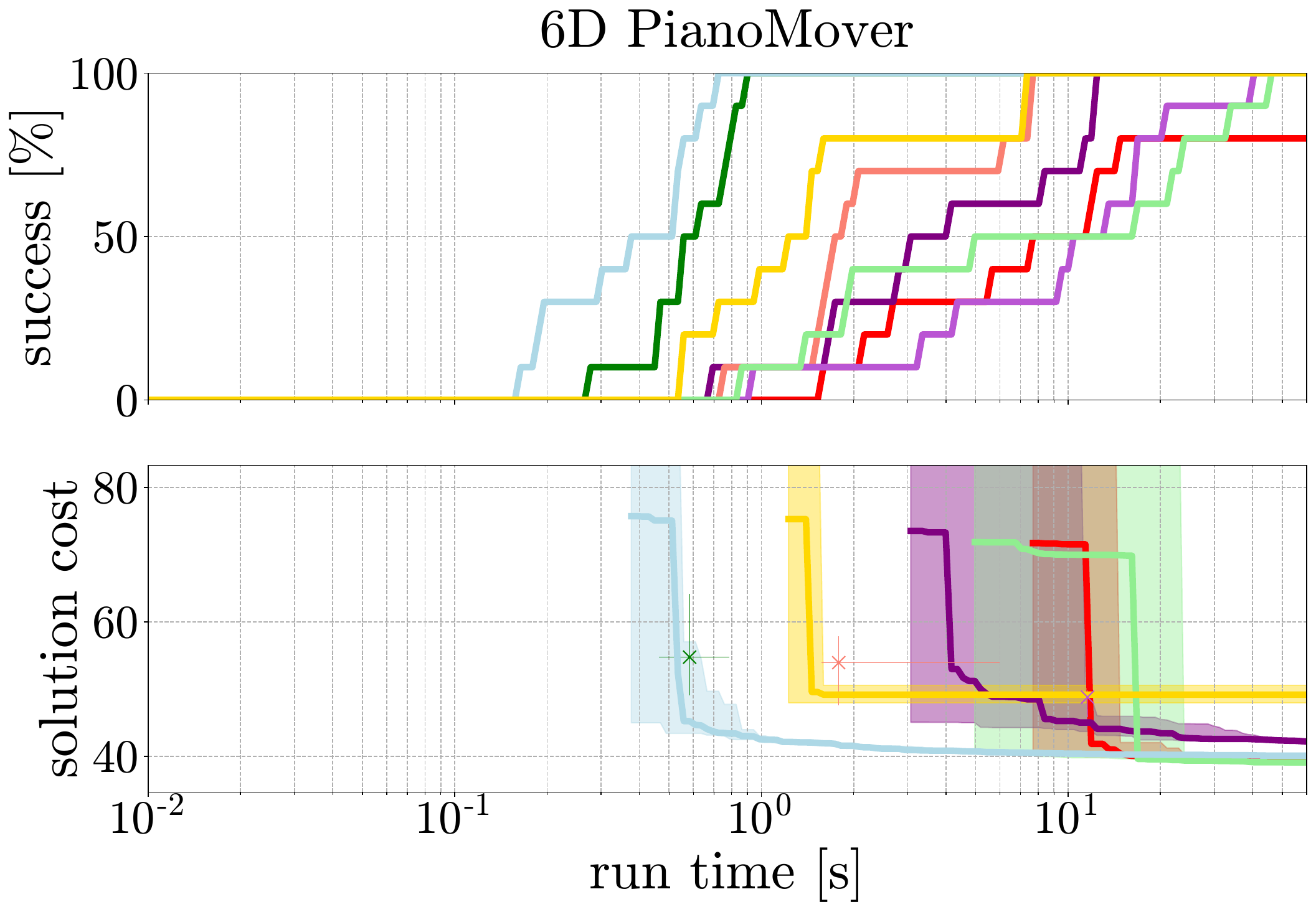}\\
    \includegraphics[width=0.48\textwidth]{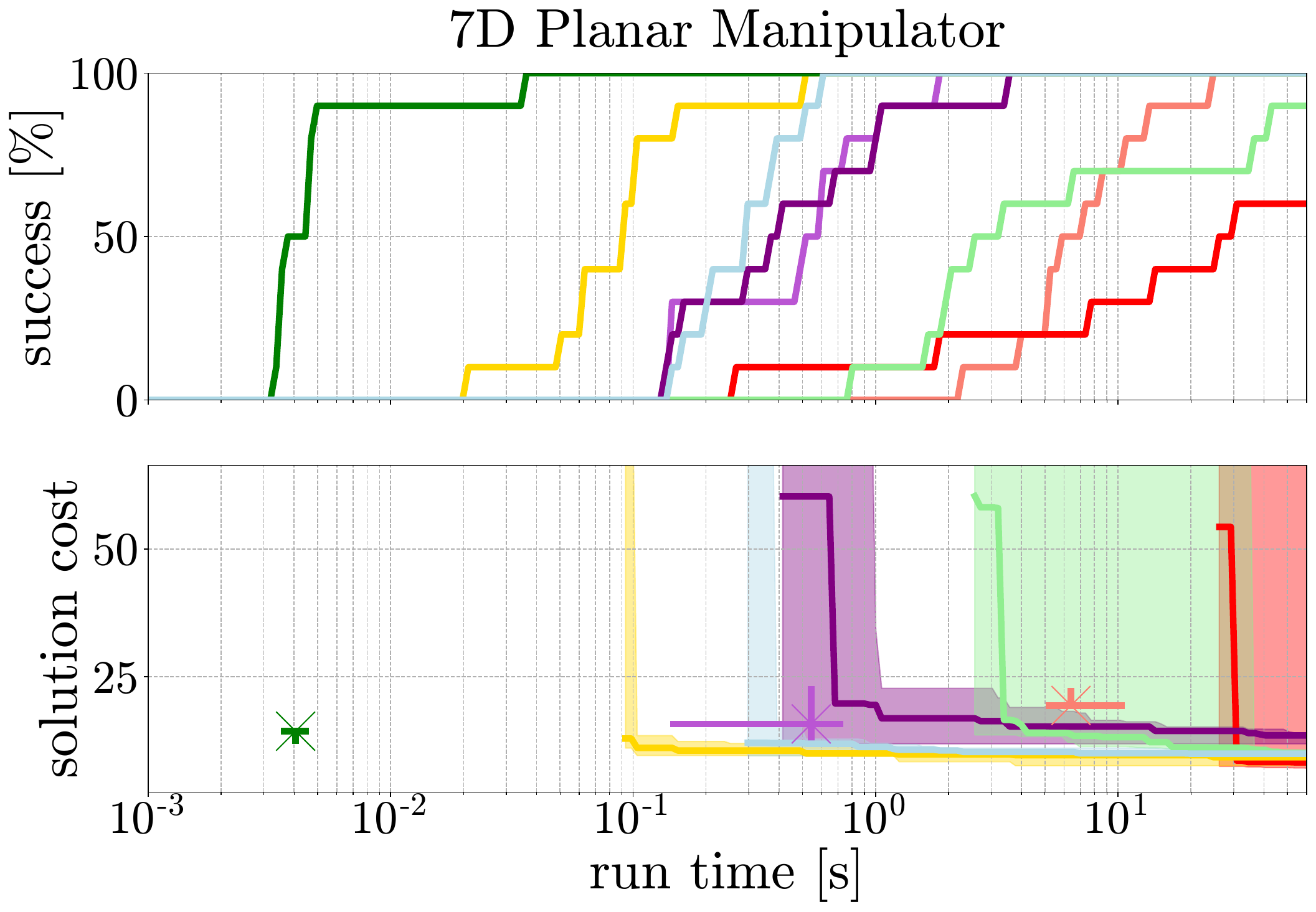}
    \includegraphics[width=0.48\textwidth]{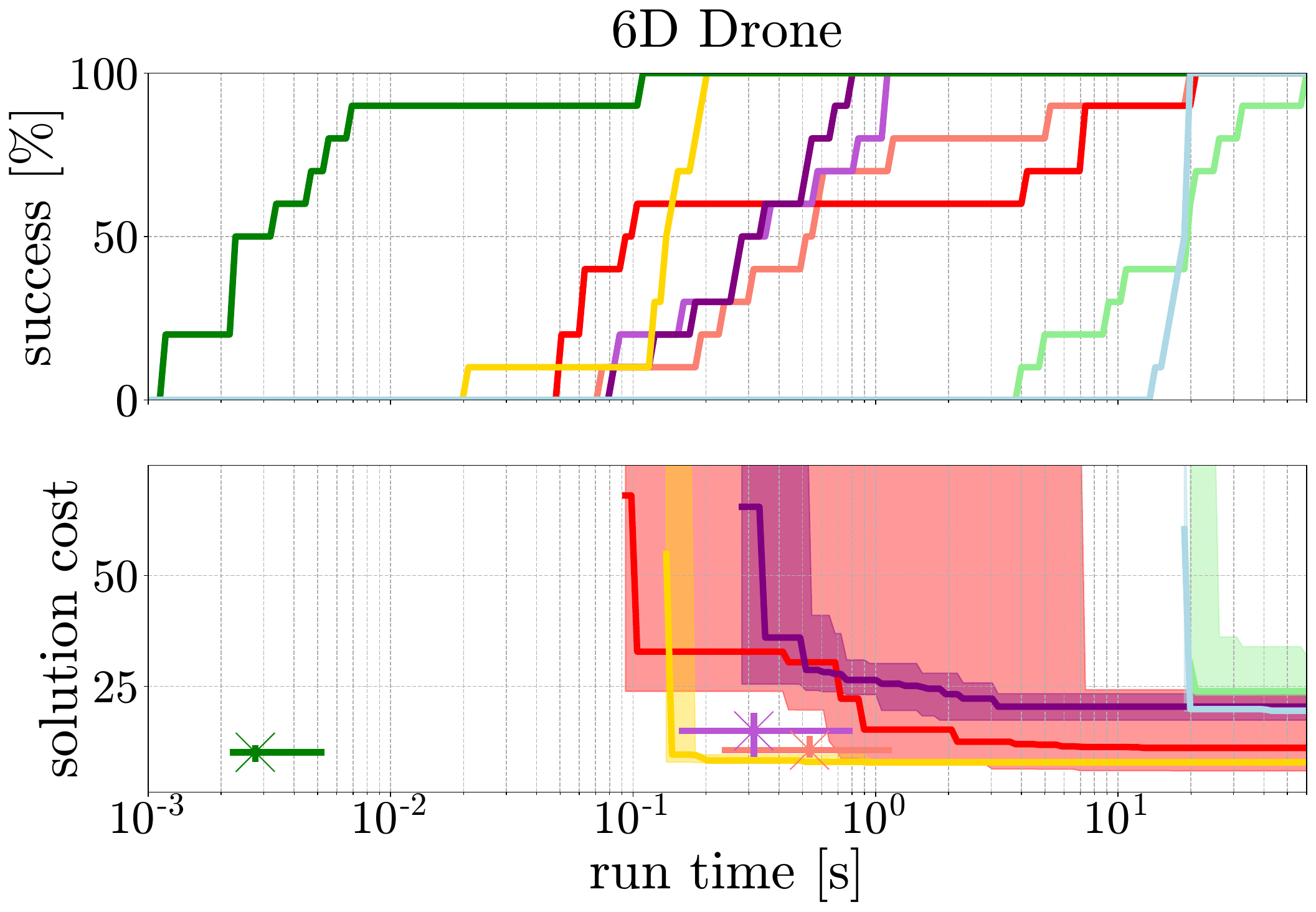}\\
    \centering\includegraphics[width=0.7\textwidth]{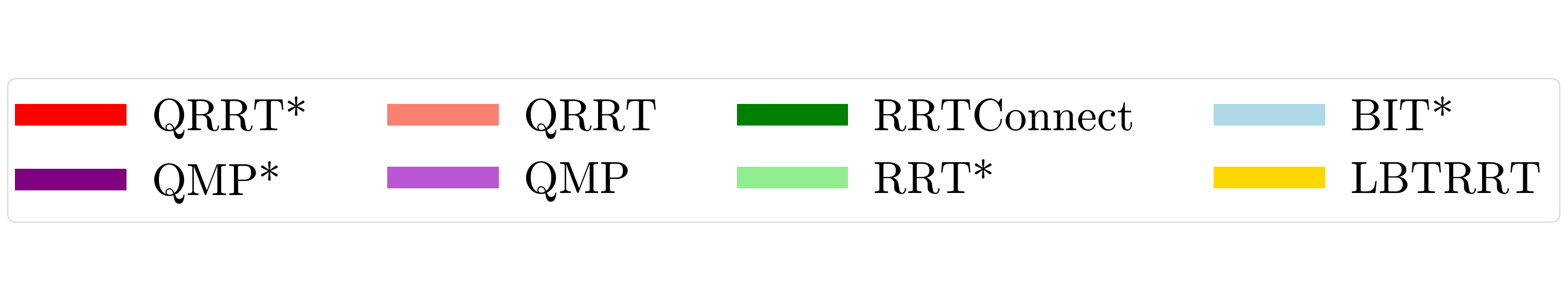}\\
    \caption{Success-cost plots of the four low-dimensional planning scenarios.\label{fig:successcost-lowdim}}
\end{figure*}

\begin{figure*}
    \begin{subfigure}[t]{0.5\textwidth}
    \centering
    \includegraphics[width=0.48\textwidth]{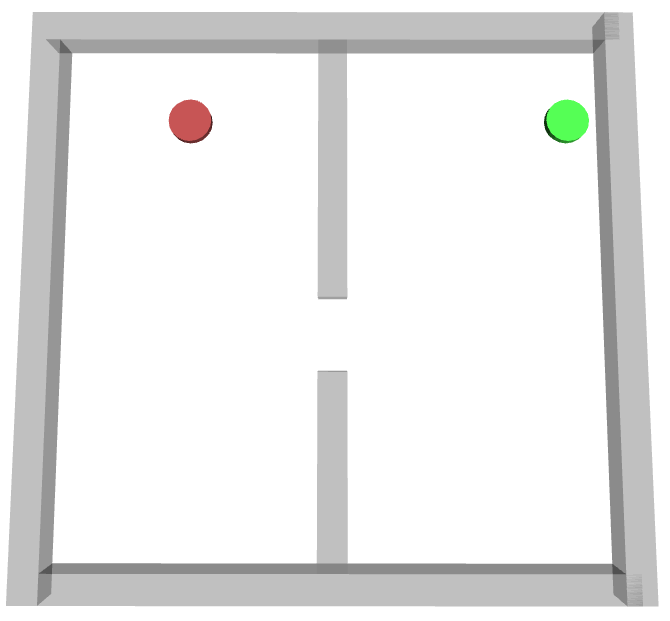}
    \includegraphics[width=0.48\textwidth]{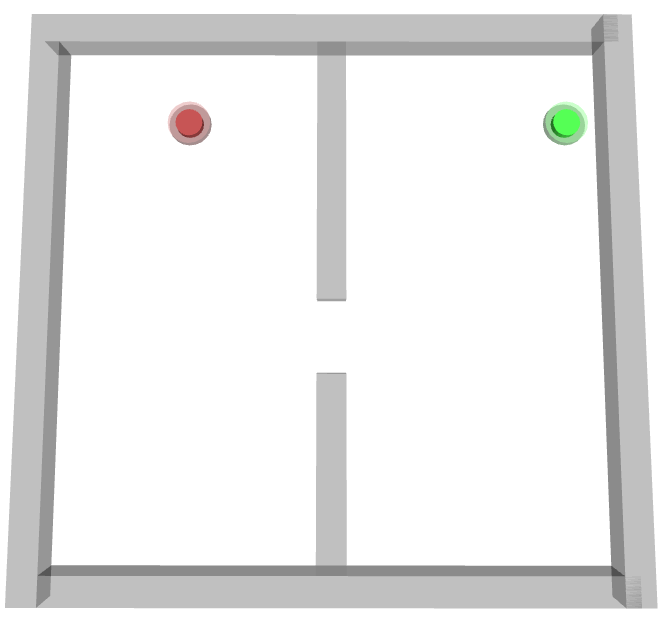}
    \caption{02D disk}
    \end{subfigure}
    \begin{subfigure}[t]{0.5\textwidth}
    \includegraphics[width=0.48\textwidth]{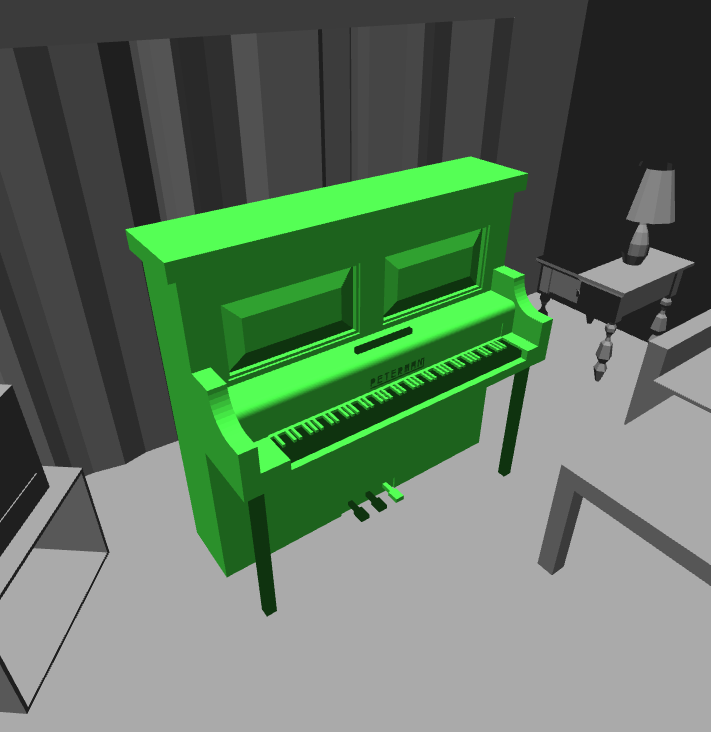}
    \includegraphics[width=0.48\textwidth]{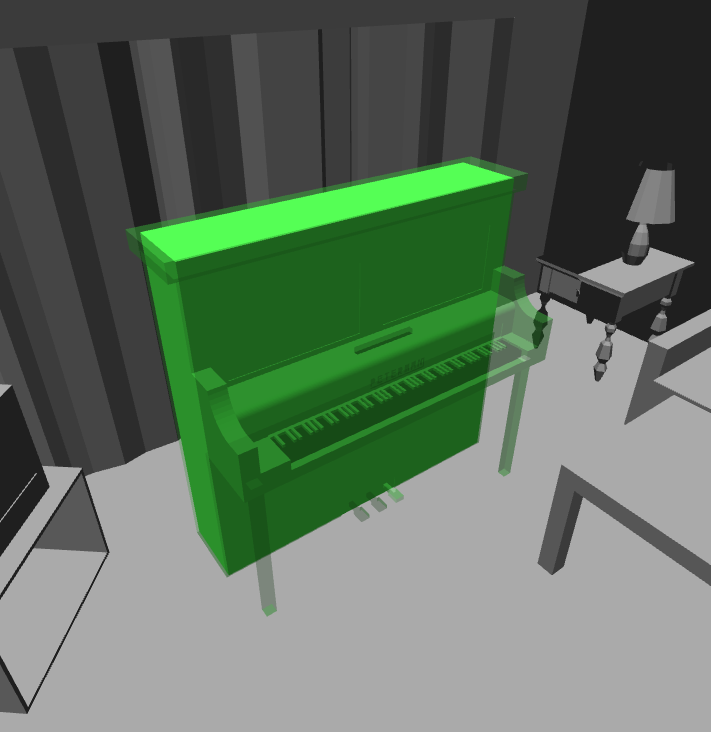}
    \caption{06D Piano Mover's problem}
    \end{subfigure}\\
    \begin{subfigure}[t]{0.5\textwidth}
    \includegraphics[width=0.48\textwidth]{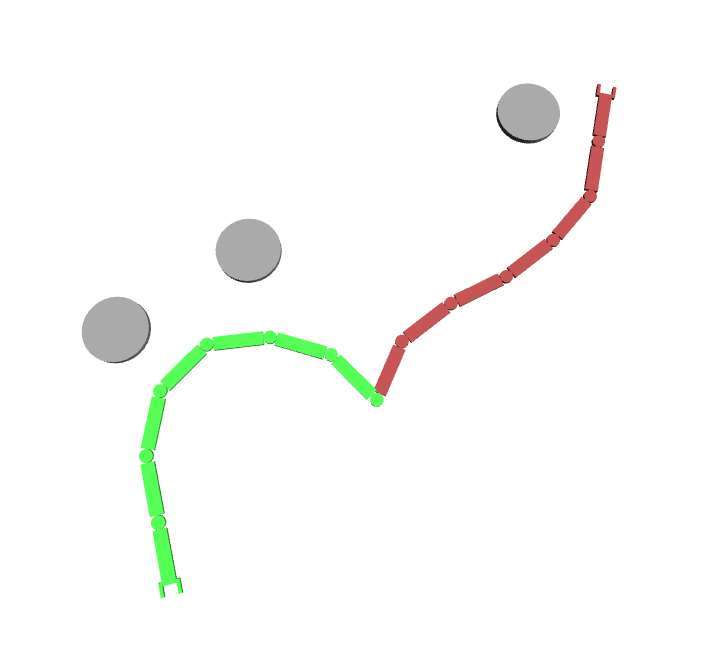}
    \includegraphics[width=0.48\textwidth]{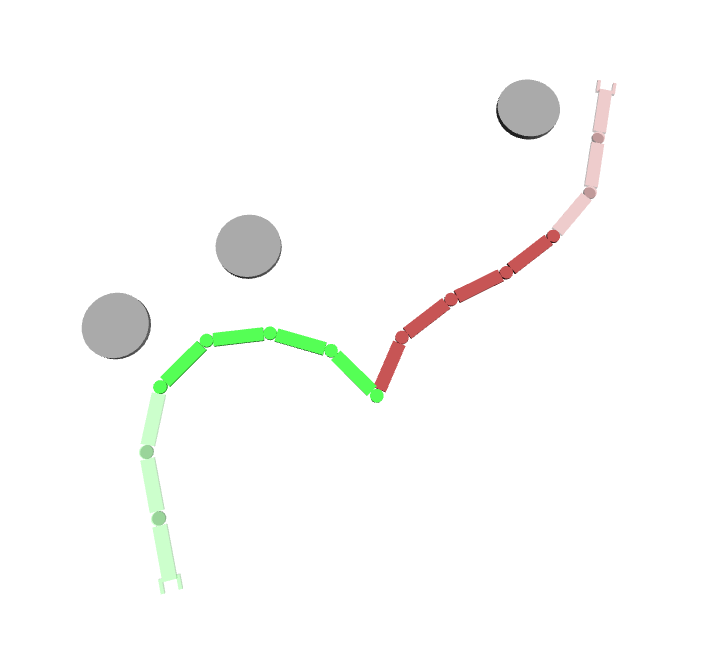}
    \caption{07D Planar Manipulator}
    \end{subfigure}
    \begin{subfigure}[t]{0.5\textwidth}
    \includegraphics[width=0.48\textwidth]{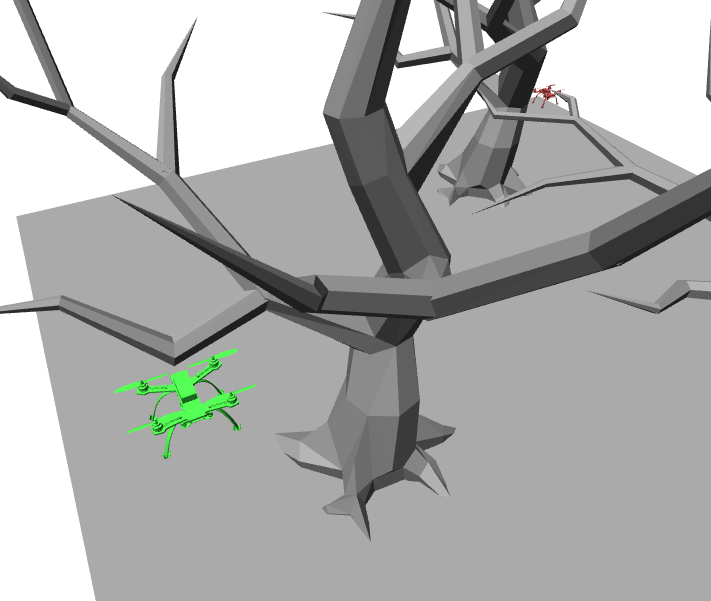}
    \includegraphics[width=0.48\textwidth]{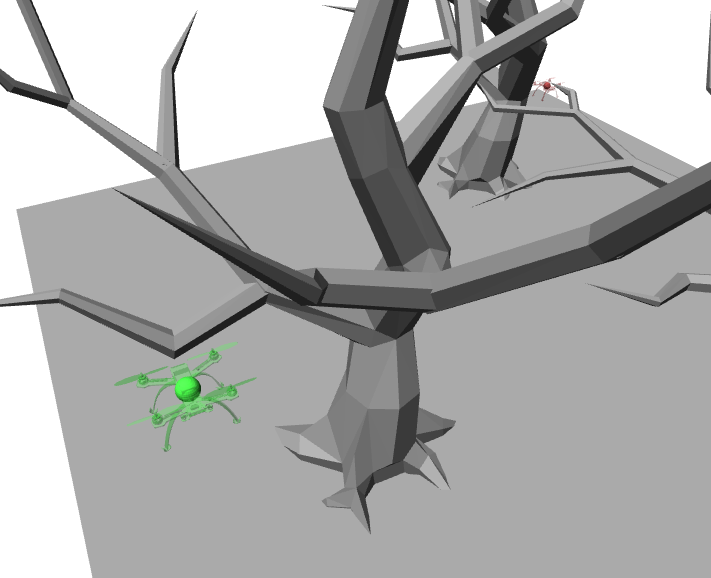}
    \caption{06D Double L-Shape}
    \end{subfigure}   
    \caption{Four scenarios for low-dimensional planning. Start configuration of robot (green) is shown alongside goal configuration (red) when applicable. In each figure, the robot is shown on the original space (left), and with the first projection applied (right), where the original robot is shown with a transparent color.\label{fig:scenarios-lowdim}}
\end{figure*}

\subsection{Low-dimensional motion planning}

In the low-dimensional motion planning evaluation, we evaluate QMP, QMP*, QRRT, and QRRT* against RRTConnect, RRT*, BIT*, and LBTRRT. This is done on four low-dimensional planning problems as shown in Fig.~\ref{fig:scenarios-lowdim}. We let all planners run until time out and collect both time to find the first solution and solution cost over time.

\subsubsection{2-dof Disk problem (2 levels)}
The first scenario is a 2-dof disk problem, where a small disk robot needs to traverse a square with a narrow passage in the middle. For our bundle algorithms, we use a projection onto a smaller inscribed disk with half the radius of the original disk (see Fig.~\ref{fig:scenarios-lowdim}). This creates a fiber bundle as
\begin{equation}
    \R^2 \rightarrow \R^2.
\end{equation}
The evaluation results are shown in Fig.~\ref{fig:successcost-lowdim} (Upper left). RRTConnect performs best in terms of quickest convergence to one hundred percent success rate, while BIT* performs best by converging the fastest to the optimal solution. All bundle planners can successfully solve this problem with competitive results both in terms of success rate (QRRT, QMP), and in terms of cost convergence (QRRT*, QMP*).

\subsubsection{3-dof Piano Mover's problem (2 levels)}

The second scenario is the piano mover's problem~\citep{schwartz1983piano}, where a piano has to be moved on a planar floor from one side of a house to the other side. 
As shown in Fig.~\ref{fig:scenarios-lowdim}, we impose a fiber bundle by inscribing a simpler shape into the original piano, thus imposing a fiber bundle as
\begin{equation}
    SE(2) \rightarrow SE(2).
\end{equation}
The evaluation results are shown in Fig.~\ref{fig:successcost-lowdim} (Upper right). BIT* and RRTConnect outperform in terms of success rate, while BIT* also converges quickest to a low-cost solution. All bundle planners perform slightly worse, but still competitive in terms of runtime and cost convergence.

\subsubsection{7-dof Planar Manipulator (4 levels)}

In the third scenario, we evaluate the planners on a 7-dof planar manipulator task, as shown in Fig.~\ref{fig:scenarios-lowdim} (Lower left). For this scenario, we impose four levels of abstractions, where we first project the 7-dof robot onto a 4-dof robot by removing the last three links. We then project onto a 2-dof robot by removing two links and finally we project onto a 1-dof robot by removing one link. The resulting fiber bundle can be written as
\begin{align}
\begin{aligned}
    SO(2) \times \R^6 \rightarrow SO(2) \times \R^3 \\ \rightarrow SO(2) \times \R^1 \rightarrow SO(2).
\end{aligned}
\end{align}
The evaluation results in Fig.~\ref{fig:successcost-lowdim} (Lower left) show that RRTConnect and LBTRRT perform best in terms of success rate, while LBTRRT converges quickest in terms of solution cost. Both QMP and QMP* perform competitively in terms of success rate and QMP* terms of cost convergence. QRRT has slightly worse performance in terms of success rate, but still solves the problem. QRRT*, however, does not solve all runs of this problem. 

\subsubsection{6-dof Drone (2 levels)}

In the fourth scenario, a drone has to traverse two trees to reach a goal state. The state space is $SE(3)$ with additional constraints on roll and pitch, but not yaw (to prevent impossible maneuvers). We impose a fiber bundle as 
\begin{equation}
    SE(3) \rightarrow \R^3,
\end{equation}
by inscribing a small sphere inside the drone, thereby reducing the state space to $\R^3$. 
The evaluation results in Fig.~\ref{fig:successcost-lowdim} (Lower right) show RRTConnect to converge quickest in success rate, with LBTRRT being fastest in cost convergence. All bundle planners solve this problem, but QMP* returns a slightly larger final cost compared to the lowest cost found.

\subsection{High-dimensional motion planning}

For the high-dimensional planning scenarios, we conduct two evaluations. First, we run a large set of planners from OMPL until a first solution is found (or a timeout occurs) and report on the runtime. Those results are evaluated for all available planners in OMPL, if they are applicable to the problem at hand. This case is discussed in Sec.~\ref{experiment:hypercube} up to Sec.~\ref{experiment:manipulator}. Second, we run the eight planners QMP, QMP*, QRRT, QRRT*, RRTConnect, RRT*, BIT*, and LBTRRT on each scenario until the timeout occurs. We collect both success rate and cost over time and plot those results as success-cost graphs. This case is discussed in Sec.~\ref{sec:evaluations:costanalysis}. Note that each case uses a different hardware setup as mentioned in Sec.~\ref{sec:evaluation}.
\subsubsection{100-dof Hypercube (98 levels)\label{experiment:hypercube}}

The hypercube \citep{gipson_2013} is a classical motion planning benchmark,
where we need to move a point robot in an $n$-dimensional cube $\X = [0,1]^n$
from $x_I = (0,\ldots,0)$ to $x_G = (1,\ldots,1)$. We allow the robot to move
only along corridors of size $\epsilon=0.1$ along the edges of the cube as shown
in Fig.~\ref{fig:100D_hypercube}. For more details see \cite{gipson_2013}. As a fiber bundle, we choose the sequence of reductions
\begin{equation}
[0,1]^n \rightarrow [0,1]^{n-1} \rightarrow \ldots \rightarrow [0,1]^2
\end{equation}
where the constraint function is the constraint function of the corresponding cube. 

Prior work showed solutions to $25$-dimensional cubes in around $100$s
\citep{gipson_2013}. Here, we attempt to solve a $100$-dimensional cube version.
The benchmarks are shown in Fig.~\ref{fig:100D_hypercube_benchmark}. All bundle
planners have an average runtime of less than $0.1$s. Also the non-bundle
planner SPARS2 terminates with a runtime of around $0.2$s. However, we note that
SPARS2 terminates with a probabilistic infeasibility proof, i.e. they declare
this problem infeasible. Only QRRT, QMP and their star versions can solve this
problem in the time limit given. While we terminate all planner at $60$s, we can
provide a rough estimate of performance improvement of QRRT compared to STRIDE
(which outperforms PRM, KPIECE, EST and RRT \citep{gipson_2013}). To do that, we
let STRIDE run on the $n=\{3,\ldots,9\}$ dimensional version of the cube, then
we extrapolate the results by fitting a cubic curve (see Fig.~\ref{fig:hypercube_progression}). Comparing the extrapolation to QRRT at the dimension $100$, we observe that QRRT performs around $6$ orders of magnitude better than STRIDE.

\begin{figure}
    \centering
    \includegraphics[width=\linewidth]{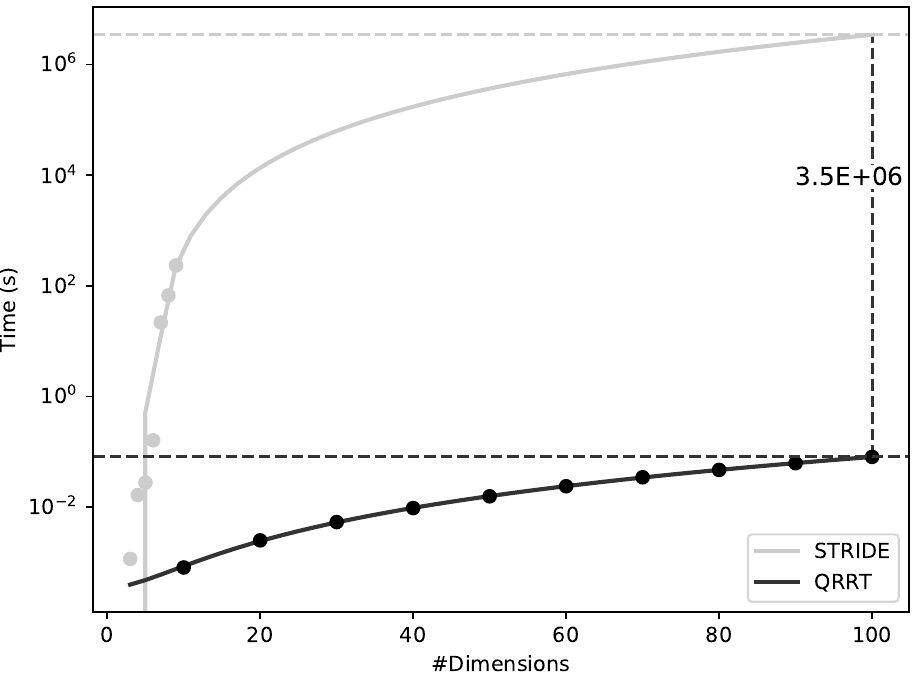}
    \caption{Hypercube scenario comparison of algorithms STRIDE and QRRT.\label{fig:hypercube_progression}}
\end{figure}
\subsubsection{21-dof Box folding (5 levels)}

To automate deliveries or assemble production pieces, we often need to compute folding motions. Here we concentrate on computing the folding motion of a small packaging box with $21$-dof (Fig.~\ref{fig:21D_box_folding}). Such problems are challenging, because parts of the box have to fit into small narrow passages, which is challenging for sampling-based planners. We use a fiber bundle sequence as
\begin{align}
\begin{aligned}
    SE(2)\times \R^{18} 
    &\rightarrow SE(2)\times \R^{16} 
     \rightarrow SE(2)\times \R^{13}\\
      &\rightarrow SE(2)\times \R^{10}
       \rightarrow SE(2)\times \R^{7}\\
       &\rightarrow SE(2)
\end{aligned}
\end{align}
which corresponds to the removal of (1) flaps on lid, (2) lid, (3) right side,
(4) left side, (5) back/front elements. We show the benchmarks in
Fig.~\ref{fig:21D_box_folding_benchmark}. The best performing algorithm is QMP
with $0.68$s of planning time. QRRT performs worse with around $6.4$s. We
discuss this performance difference in Sec.~\ref{sec:discussion}. Both QMP and QRRT together with QMP* outperform all other planning algorithms, i.e. no OMPL algorithm was able to solve this scenario in our timelimit. 

\subsubsection{24-dof Dubins Cars crossing (3 levels)}

With several companies pushing towards autonomous driving,
we need increasingly more efficient algorithms to coordinate multiple
car-like robots under non-holonomic constraints. We concentrate here on the
problem of planning motions for eight Dubins cars \citep{dubins_1957}, which are
cars with constant forward speed, which we can steer left or right. The cars
start on different ends of a crossroad (in reverse direction) and we need them
to cross the road while avoiding each other (Fig.~\ref{fig:24D_crossing_cars}). We impose a fiber bundle as
\begin{equation}
    SE(2)^8 
    \rightarrow (\R^2)^8
       \rightarrow (\R^2)^4
\end{equation}
which corresponds to the reduction onto a disk inscribed in the car and the
removal of the upper four robots, respectively. We show the benchmark in Fig.~\ref{fig:24D_crossing_cars_benchmark}. QRRT performs best with a planning time of $0.28$s closely followed by QRRT* ($0.29$s) and QMP ($1.77$s). QMP* performs less well with $12.41$s of planning time. Except EST with planning time of around $54$s, there was no non-bundle algorithm able to solve this coordination problem in the timelimit given.

\subsubsection{30-dof airport (15 levels)}

While coordinating motions for multiple cars are essential for traffic
coordination, we often need to coordinate multiple vehicles in 3D under
non-holonomic constraints. One particular instance of this problem is an
airport, in which we might need to coordinate cars, planes and zeppelins, each
with different state spaces and different possible nonholonomic constraints.
Here, we use a scenario with $3$ trucks, $1$ zeppelin, $1$ propeller plane, $1$
airplane while
taxiing\footnote{\href{https://en.wikipedia.org/wiki/Taxiing}{Taxiing} refers to
movements of an airplane on the ground, for example after landing or before
take-off.} and $2$ airplanes while flying (see Fig.~\ref{fig:30D_airport}). 
This scenario is particularly challenging, since all vehicles have non-holonomic constraints except the zeppelin. We model the dynamics of the trucks and the planes as Dubins car and Dubins airplane \citep{lavalle_2006}, respectively. Note that arbitrary dynamically constraints could be imposed, but there are implementations of Dubins car and airplane spaces available in OMPL, which makes them also useable with other algorithms in the library. We use a prioritization-like abstraction as
%
\begin{align}
\begin{aligned}
    SE(2)^4 \times SE(3) \times (\R^3 \times SO(2))^3
    \rightarrow\\
    \R^2 \times SE(2)^3 \times SE(3) \times (\R^3 \times SO(2))^3 \rightarrow \\
    SE(2)^3 \times SE(3) \times (\R^3 \times SO(2))^3 \rightarrow \\
    \R^2 \times SE(2)^2 \times SE(3) \times (\R^3 \times SO(2))^3 \rightarrow \\
    SE(2)^2 \times SE(3) \times (\R^3 \times SO(2))^3 \rightarrow \\
    \R^2 \times SE(2) \times SE(3) \times (\R^3 \times SO(2))^3 \rightarrow \\
    SE(2) \times SE(3) \times (\R^3 \times SO(2))^3 \rightarrow \\  
    \R^2 \times SE(3) \times (\R^3 \times SO(2))^3 \rightarrow \\
    SE(3) \times (\R^3 \times SO(2))^3 \rightarrow \\
    (\R^3 \times SO(2))^3 \rightarrow \\
    \R^3 \times (\R^3 \times SO(2))^2 \rightarrow \\
    (\R^3 \times SO(2))^2 \rightarrow \\
    \R^3 \times (\R^3 \times SO(2)) \rightarrow\\
    (\R^3 \times SO(2)) \rightarrow \\
    \R^3
\end{aligned}
\end{align}
where the first four $SE(2)$ spaces represent the three trucks and the taxiing airplane. The $SE(3)$ space represents the zeppelin and the remaining $3$ spaces of $\R^3\times SO(2)$ represent the two flying airplanes and the propeller plane, respectively. Each projection either projects an $SE(2)$ space by using the simpler robots of a nested disk, by removing a robot completely (and its geometry) or by nesting an inscribed sphere.
The benchmarks are shown in Fig.~\ref{fig:30D_airport_benchmark}. The best performing planner are QRRT ($0.52$s), QMP ($0.99$s) and QMP* ($0.94$s). QRRT* performs significantly worse with a planning time of around $58$s, which suggest that it could not completely solve this problem in the time allocated. Besides the bundle planner, we also observe that RRTConnect shows competitive results with $4.5$s of planning time.

\subsubsection{37-dof pregrasp (3 levels)}

Manipulation of objects is a challenging task for robots \citep{dafle_2018,
driess_2020}, in particular if we have to deal with realistic hands with many
dofs. We concentrate here on computing a pregrasp for a 37-dof shadow-hand robot
mounted on a KUKA LWR robot. We define the problem as finding a pregrasp for the
grasping of a small glass, as we depict in Fig.~\ref{fig:37D_pregrasp}. We impose a fiber bundle as
\begin{equation}
    \R^{31} \rightarrow \R^{18} 
       \rightarrow \R^{13} 
\end{equation}
which corresponds to a reduction by first removing all fingers except thumb and
index finger and second removing the thumb. The benchmark for this problem are
shown in Fig.~\ref{fig:37D_pregrasp_benchmark}. Both QMP and QMP* perform
best with around $6.81$s and $12.36$s of planning time. In this scenario, no non-bundle planner can solve this problem. Please note that the
planner QRRT and QRRT* perform around $44$s and $48$s. We discuss this
performance further in Sec.~\ref{sec:discussion}.

\subsubsection{48-dof drones (8 levels)}

Planning motions for multiple quadrotors 
\citep{hoenig_2018} is essential for drone delivery, in disaster response
scenarios and for entertainment purposes. We consider here the problem of
coordinating the motion of eight drones which have to traverse a small
forest-like environment as shown in Fig.~\ref{fig:48D_drones}. We use the fiber bundle
\begin{equation}
    SE(3)^8 \rightarrow SE(3)^7 \rightarrow \cdots \rightarrow SE(3),
\end{equation}
which corresponds to a prioritization of the drones, i.e. in each projection we
remove one robot. The benchmarks are shown in Fig.~\ref{fig:48D_drones_benchmark}. While the best algorithm is QRRT ($0.14$s) closely followed by QMP ($0.15$s) and QMP* ($0.16$s), we observe that also RRTConnect and BFMT show competitive performances with $0.59$s and $6.05$s, respectively.

\subsubsection{54-dof Kraken animation (17 levels)}

Computer animation is an important application of planning algorithms \citep{plaku_2018}. In animations for movies, an animator would probably insert keyframes to guide the planning of motions. However, if we like to compute animations online, for example for a computer game, we require fast planning algorithms.

We show here the problem of animating a 54-dof Kraken-like robot (see Fig.~\ref{fig:54D_kraken}), which has to wrap its arms around a sailing ship. We use a fiber bundle reduction as
\begin{align}
\begin{aligned}
    SE(3)\times \R^{48}
    &\rightarrow SE(3)\times \R^{45} 
    \rightarrow SE(3)\times \R^{42}\\
    &\rightarrow\ldots\rightarrow\\
    SE(3)\times \R^{6}&\rightarrow SE(3)\times \R^{3}
    \rightarrow SE(3)
\end{aligned}
\end{align}
which corresponds to the removal of each arm ($6$-dof revolute joints) on each
stage, whereby we first remove the last three links (removal of $3$-dof) and
then remove the remaining arm ($3$-dof). We show the benchmark in Fig.~\ref{fig:54D_kraken_benchmark}. We observe that both QRRT ($0.20$s) and QMP ($0.23$s) perform below $1$s to find a feasible solution. Next comes QMP* with a planning time of $6.21$s. The next best non-bundle planner is BiTRRT with a performance of around $22$s planning time. The performance of the bundle planner QRRT is thus two orders of magnitude better than the next best non-bundle planner. 

\subsubsection{72-dof manipulators (3 levels)\label{experiment:manipulator}}

When automating construction work \citep{hartmann_2020} or warehouse operations
\citep{salzman_2020, eppner_2016}, we often need to coordinate multiple robots
with many dofs. Here, we consider the coordination of eight KUKA manipulators on
disk-shaped mobile bases. Each manipulator starts around a circle and needs to
change position with its antipodal partner (see Fig.~\ref{fig:72D_manips}). We impose a fiber bundle as
\begin{equation}
    (SE(2) \times \R^6)^8 \rightarrow (\R^2)^8 \rightarrow (\R^2)^4
\end{equation}
which corresponds to the removal of arms and the removal of the upper half of
the robots. The benchmarks are shown in Fig.~\ref{fig:72D_manips_benchmark}. We observe that QRRT solves this problem in $3.65$s while QRRT* requires $19$s. Only one non-bundle planner is able to terminate on average before the timelimit: RRTConnect with around $39$s seconds of planning time. Note that this problem is difficult for QMP ($57$s) and QMP* ($50$s) which perform worse than RRTConnect.

\begin{table*}[t]
    \centering
    \begin{tabular}{|c|c|c|c|c|c|c|}
    \hline
         \multicolumn{2}{|c|}{Motion Planning Algorithm} &
           Origin Paper &
         FB &
        PC & AnO \\
        \hline
        QRRT & Rapidly-exploring random quotient space trees & \citep{orthey_2019} &x & x& \\
        QMP & Quotient-space roadmap planner & \citep{orthey_2018} &x & x&  \\
        QRRT* & Optimal version of QRRT & \textbf{this paper} &x & x& x \\
        QMP* & Optimal version of QMP & \textbf{this paper} &x & x& x \\
        \hline
        PRM & Probabilistic Roadmap Planner & \citep{kavraki_1996} & & x&  \\
        PRM* & Optimal version of PRM & \citep{karaman_2011} & & x &  \\
        LazyPRM* & Optimal version of LazyPRM & \citep{karaman_2011} & & x &  \\
        SPARS & Sparse roadmap spanners & \citep{dobson_2014} & & x & x  \\
        SPARS2 & SPARS without dense graph & \citep{dobson_2014} & & x & x  \\
        RRT & Rapidly-exploring random tree & \citep{lavalle_1998} & & x&  \\
        RRTConnect & Bidirectional RRT & \citep{kuffner_2000} & & x&  \\
        RRT* & Optimal version of RRT & \citep{karaman_2011} & & x &  \\
        LazyRRT & Lazy edge evaluation RRT & \citep{kuffner_2000} & & x & \\
        TRRT & Transition-based RRT & \citep{jaillet_2010} & & x & \\
        BiTRRT & Bidirectional TRRT & \citep{jaillet_2010} & & x & \\
        LBTRRT & Lower-bound tree RRT & \citep{salzman_2016} & & x &x \\
        RRTX & RRT with pseudo-optimal tree & \citep{otte_2016} & & x & x \\
        RRT\# & RRT sharp & \citep{arslan_2013} &  & x &x \\
        InformedRRT* & Informed search RRT* & \citep{gammell_2014} & & x &x \\
        SORRT* & Sorted InformedRRT* & \citep{gammell_2014} & & x &x \\
        SBL & Single-query bidirectional lazy PRM & \citep{sanchez_2003} & & x & \\
        SST & Stable sparse RRT & \citep{li_2016} & & x &x \\
        STRIDE & \makecell[t]{Search Tree with Resolution \\Independent Density Estimation} & \citep{gipson_2013} & & x & \\
        FMT & Fast marching tree & \citep{janson_2015} & & x &x \\
        BFMT & Bidirectional FMT & \citep{janson_2015} & & x &x \\
        BIT* & Batch informed trees& \citep{gammell_2020} & & x &x \\
        ABIT* & Advanced BIT*& \citep{strub_2020} & & x &x \\
        EST & Expansive spaces planner& \citep{hsu_1999} & & x & \\
        BiEST & Bidirectional EST& \citep{hsu_1999} & & x & \\
        ProjEST & Projection EST& \citep{hsu_1999} & & x & \\
        KPIECE & \makecell[t]{Kinodynamic Motion Planning \\by Interior-Exterior Cell Exploration}& \citep{sucan_2009} & & x & \\
        BKPIECE & Bidirectional KPIECE& \citep{sucan_2009} & & x & \\
        LBKPIECE & Lazy BKPIECE& \citep{sucan_2009} & & x & \\
        PDST & Path-Directed Subdivision Tree& \citep{ladd_2004} & & x & \\
        \hline
    \end{tabular}
    \caption{List of motion planning algorithms used in experimental section. Properties of the algorithms are: Supporting fiber bundles (FB), being probabilistically complete (PC) and being asymptotically (near-)optimal (AnO).}
    \label{tab:list_of_algorithms}
\end{table*}
\begin{figure*}[ht]
    \centering
    \begin{subfigure}[t]{0.33\textwidth}
    \centering
    \begin{tikzpicture} 
\def\radiusPoint{3pt}
\def\slant{30}
\def\widthCube{0.5*\linewidth}
\def\heightCube{0.5*\linewidth}
\def\widthEdge{0.1*\widthCube}
\def\lengthEdge{0.9*\widthCube}
\tikzset{
    border/.style={line width=0.1mm,gray!50, dashed},
    shaded/.style={line width=0.1mm,gray!50, fill=gray!20}
}

\newcommand\drawCuboidDashed[5]{

\coordinate (O) at (#1, #2);
\def\width{#3}
\def\length{#4}
\def\height{#5}

\coordinate (B1) at ($(O)+(0, 0)$);
\coordinate (B2) at ($(O) + (\width, 0)$);
\coordinate (B3) at ($(B2) + ({\length*cos(\slant)},
{\length*sin(\slant)})$);
\coordinate (B4) at ($(B1) + ({\length*cos(\slant)},
{\length*sin(\slant)})$);

\coordinate (T1) at ($(B1)+(0,\height)$);
\coordinate (T2) at ($(B2)+(0,\height)$);
\coordinate (T3) at ($(B3)+(0,\height)$);
\coordinate (T4) at ($(B4)+(0,\height)$);

\draw[border] (B1) -- (B2) -- (B3) -- (B4) -- (B1);
\draw[border] (B4) -- (B3) -- (T3) -- (T4) -- (B4);
\draw[border] (B1) -- (B2) -- (T2) -- (T1) -- (B1);
\draw[border] (T1) -- (T2) -- (T3) -- (T4) -- (T1);

}

\newcommand\drawCuboid[5]{

\coordinate (O) at (#1, #2);
\def\width{#3}
\def\length{#4}
\def\height{#5}

\coordinate (B1) at ($(O)+(0, 0)$);
\coordinate (B2) at ($(O) + (\width, 0)$);
\coordinate (B3) at ($(B2) + ({\length*cos(\slant)},
{\length*sin(\slant)})$);
\coordinate (B4) at ($(B1) + ({\length*cos(\slant)},
{\length*sin(\slant)})$);

\coordinate (T1) at ($(B1)+(0,\height)$);
\coordinate (T2) at ($(B2)+(0,\height)$);
\coordinate (T3) at ($(B3)+(0,\height)$);
\coordinate (T4) at ($(B4)+(0,\height)$);


\draw[shaded] (B1) -- (B2) -- (B3) -- (B4) -- (B1);
\draw[shaded] (B1) -- (B4) -- (T4) -- (T1) -- (B1);
\draw[shaded] (B4) -- (B3) -- (T3) -- (T4) -- (B4);
\draw[shaded] (B1) -- (B2) -- (T2) -- (T1) -- (B1);
\draw[shaded] (T1) -- (T2) -- (T3) -- (T4) -- (T1);
\draw[shaded] (B2) -- (B3) -- (T3) -- (T2) -- (B2);
}

\def\x{{\lengthEdge*cos(\slant)}};
\def\y{{\lengthEdge*sin(\slant)}};
\drawCuboid{\x}{\y}{\widthEdge}{\widthEdge}{\heightCube}

\def\x{{\widthEdge+\lengthEdge*cos(\slant)}};
\def\y{{\lengthEdge*sin(\slant)}};
\drawCuboid{\x}{\y}{\lengthEdge}{\widthEdge}{\widthEdge}

\drawCuboid{\lengthEdge}{0}{\widthEdge}{\lengthEdge}{\widthEdge}

\drawCuboidDashed{0}{0}{\widthCube}{\widthCube}{\heightCube}

\coordinate (XI) at (\widthCube, 0);
\fill[green, fill=green] (XI) circle[green, fill=green, radius=\radiusPoint];

\coordinate (XG) at ($({\widthCube*cos(\slant)},
{\widthCube*sin(\slant) + \heightCube})$);
\fill[red, fill=red] (XG) circle[radius=\radiusPoint];

\end{tikzpicture}
    \caption{100-dof hypercube (3-dof version shown)\label{fig:100D_hypercube}}
    \end{subfigure}
    \begin{subfigure}[t]{0.66\textwidth}
    \centering
    \includegraphics[width=\textwidth]{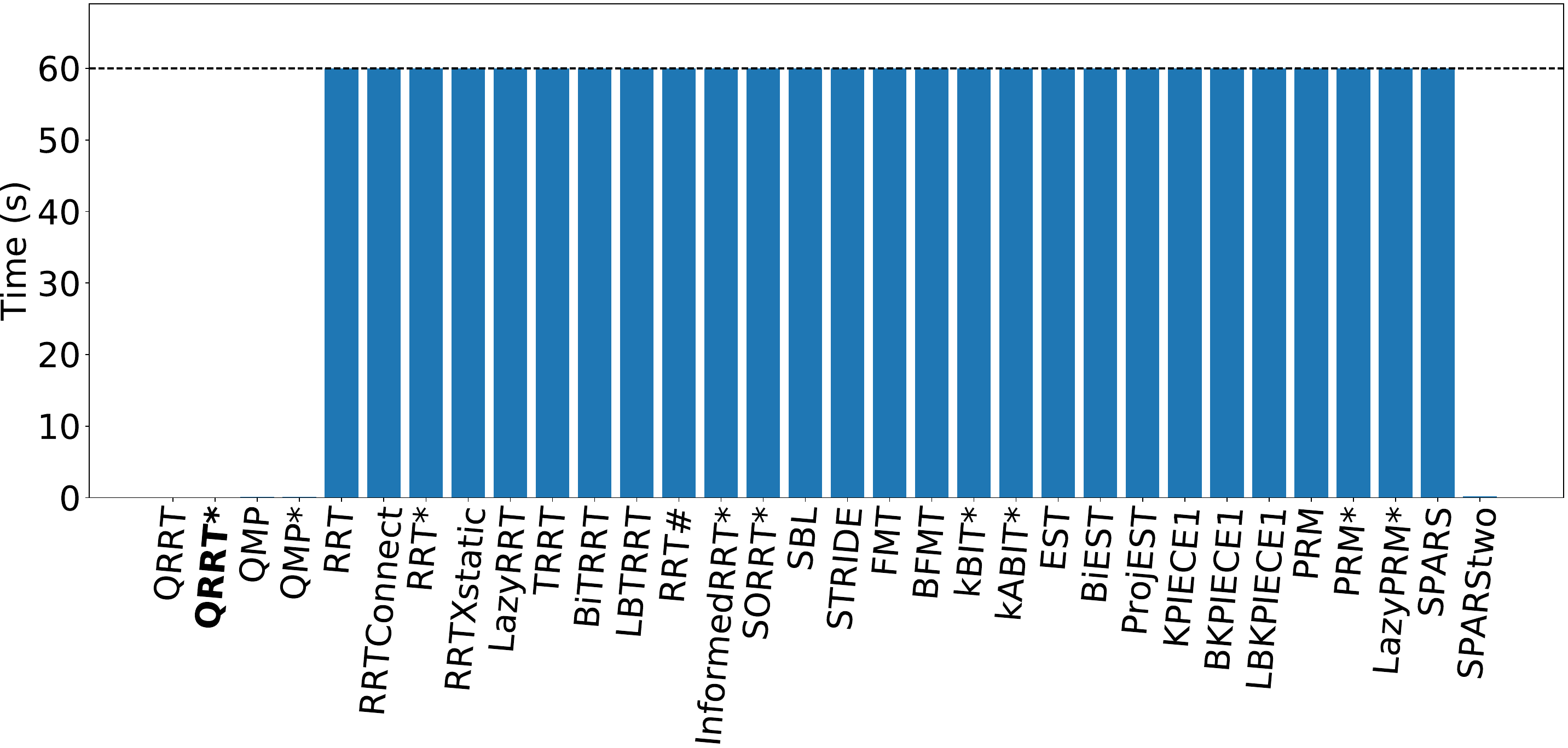}
    \caption{Benchmark of 100-dof hypercube\label{fig:100D_hypercube_benchmark}}
    \end{subfigure}
    \begin{subfigure}[t]{0.33\textwidth}
    \centering
    \includegraphics[width=\textwidth]{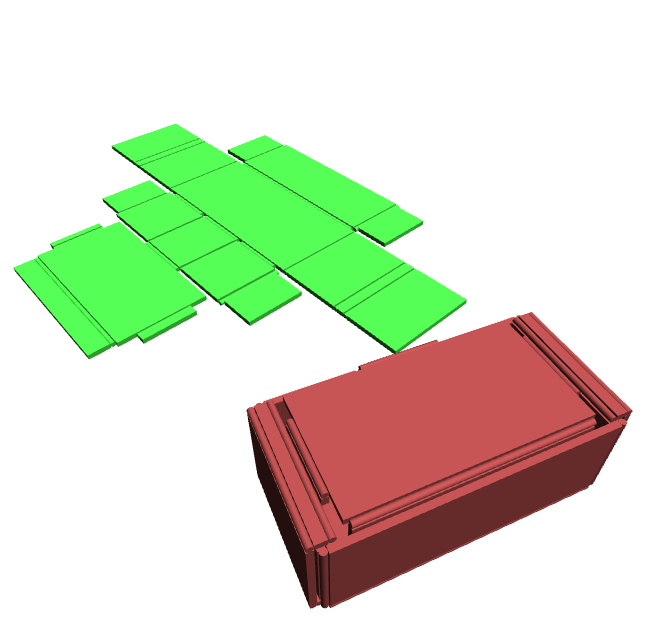}
    \caption{21-dof box folding problem\label{fig:21D_box_folding}}
    \end{subfigure}
    \begin{subfigure}[t]{0.66\textwidth}
    \centering
    \includegraphics[width=\textwidth]{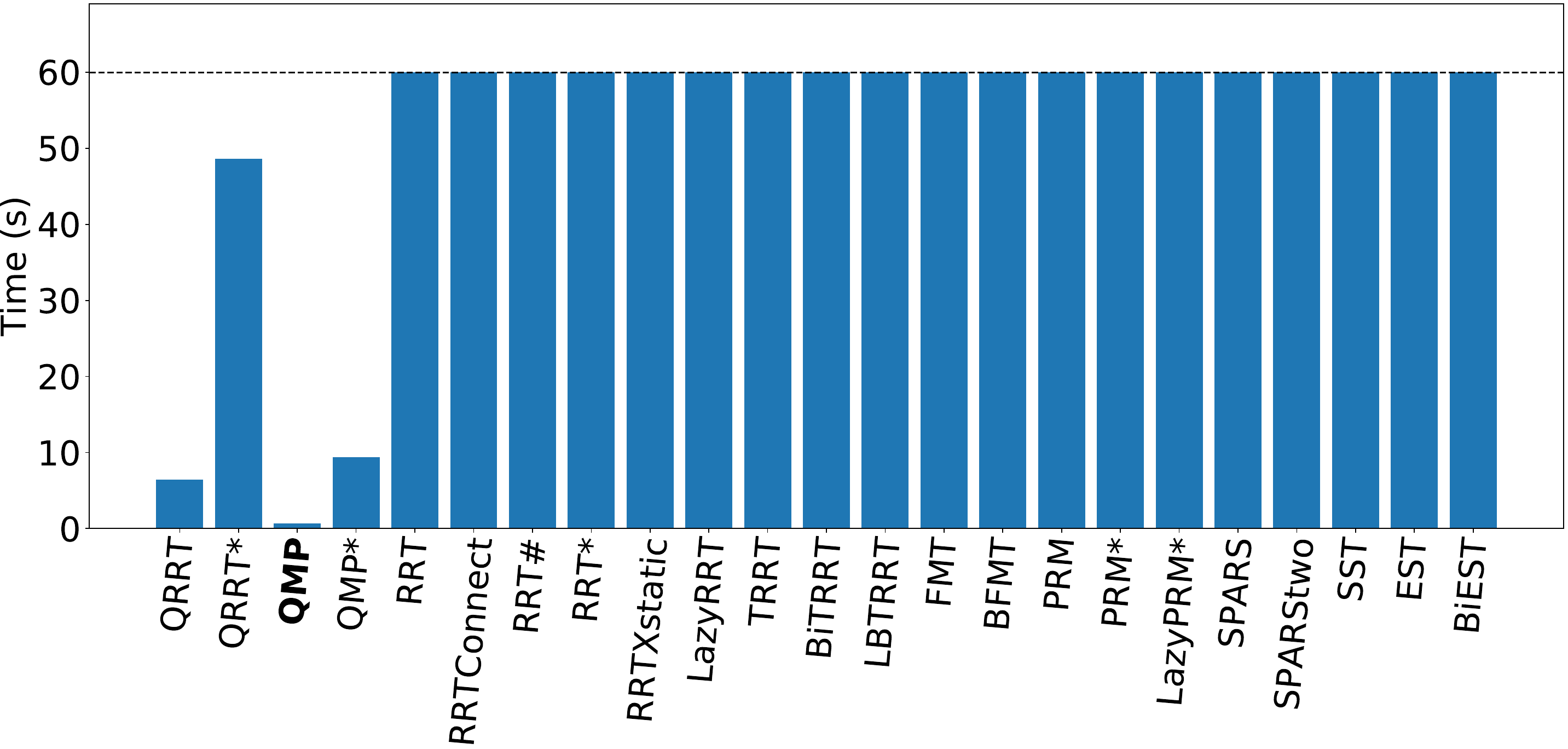}
    \caption{Benchmark of 21-dof folding box problem\label{fig:21D_box_folding_benchmark}}
    \label{fig:21D_folding_box}
    \end{subfigure}
    \begin{subfigure}[t]{0.33\textwidth}
    \centering
    \includegraphics[width=\textwidth]{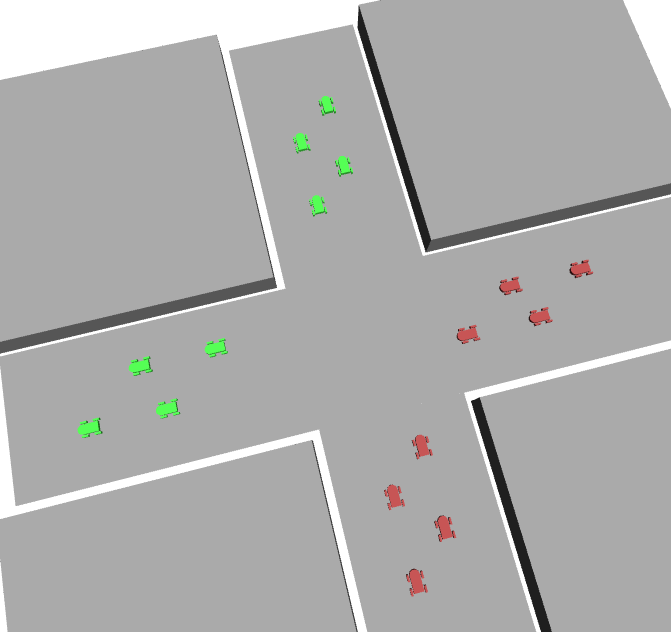}
    \caption{24-dof dubin cars crossing\label{fig:24D_crossing_cars}}
    \end{subfigure}
    \begin{subfigure}[t]{0.66\textwidth}
    \centering
    \includegraphics[width=\textwidth]{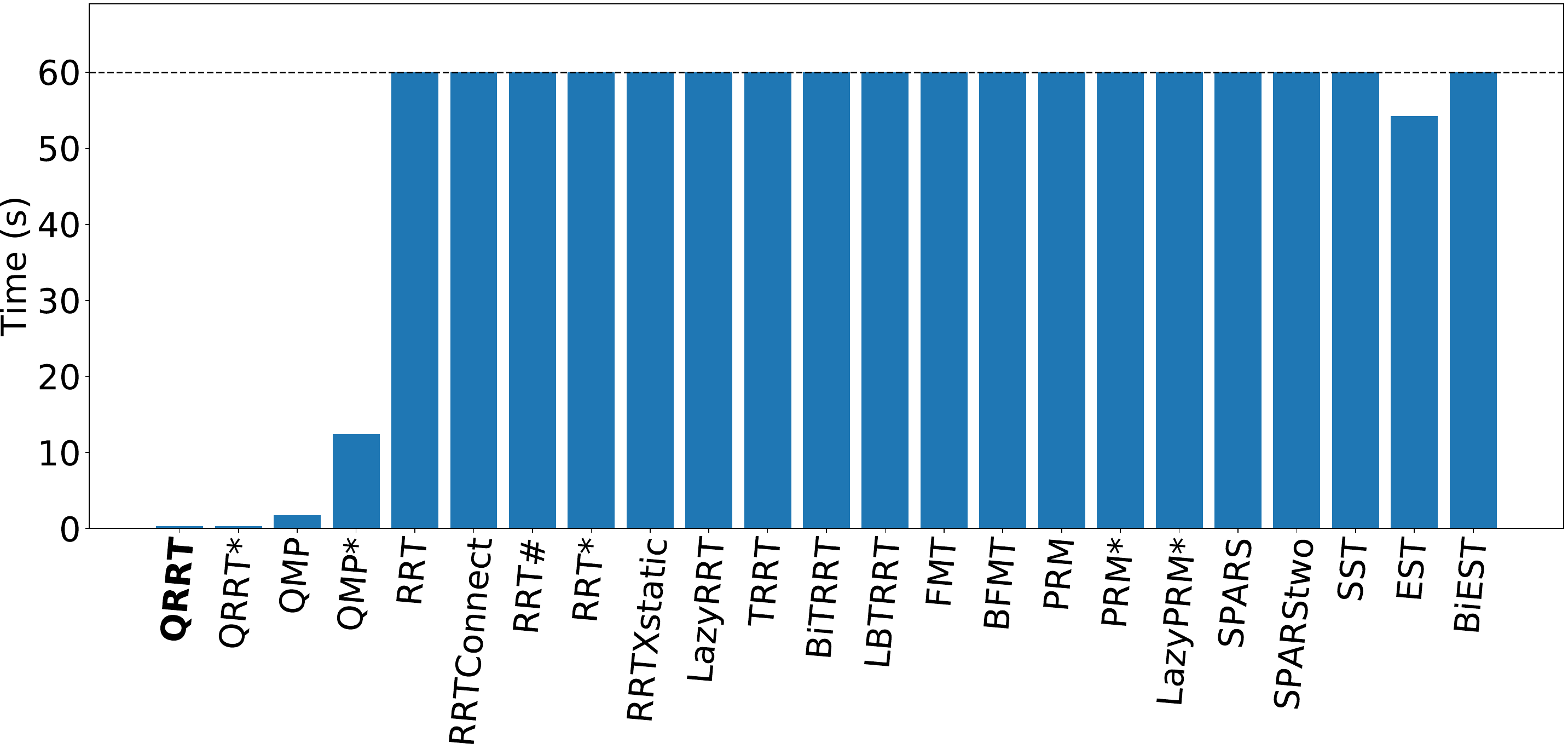}
    \caption{Benchmark \label{fig:24D_crossing_cars_benchmark}}
    \label{fig:24D_crossing_car}
    \end{subfigure}
    \begin{subfigure}[t]{0.33\textwidth}
    \centering
    \includegraphics[width=\textwidth]{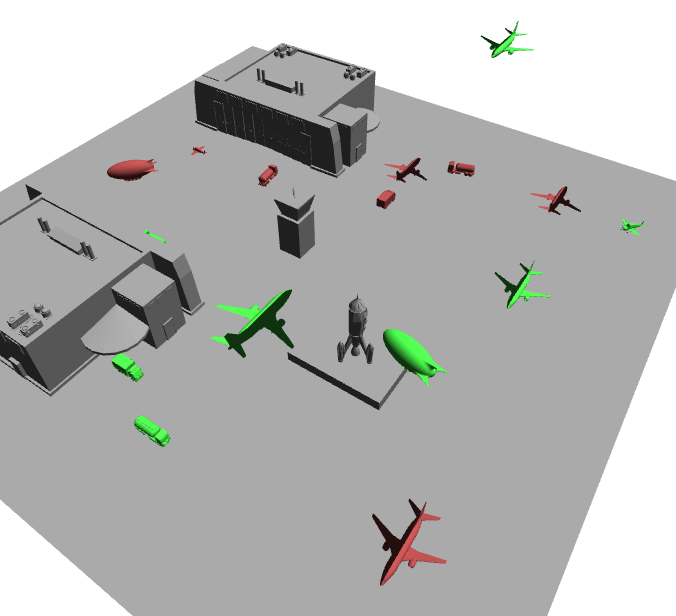}
    \caption{30-dof airport\label{fig:30D_airport}}
    \end{subfigure}
    \begin{subfigure}[t]{0.66\textwidth}
    \centering
    \includegraphics[width=\textwidth]{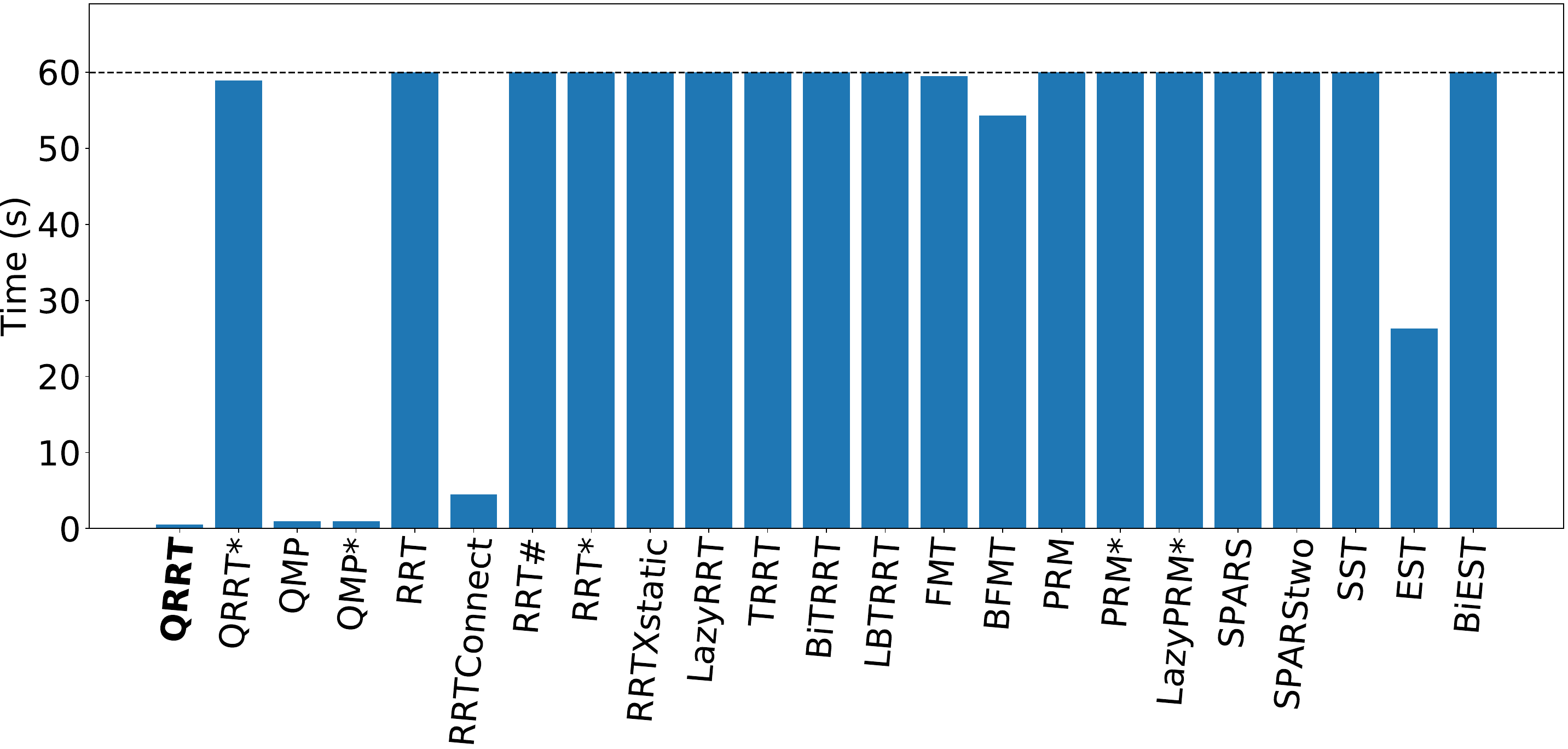}
    \caption{Benchmark\label{fig:30D_airport_benchmark}}
    \end{subfigure}

    \caption{Runtime benchmarks on the first four high-dimensional planning scenarios.}
    \label{fig:benchmarks_1}
\end{figure*}

\begin{figure*}[ht]
    \centering
    
    \begin{subfigure}[t]{0.33\textwidth}
    \centering
    \includegraphics[width=\textwidth]{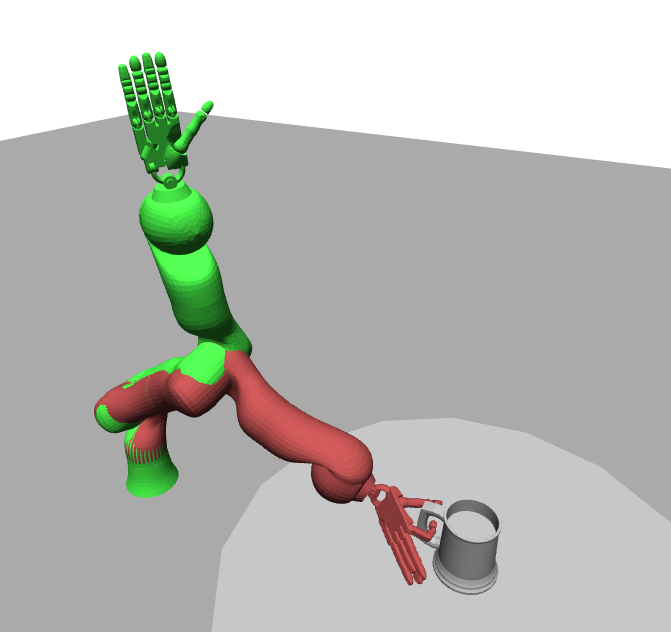}
    \caption{37-dof pre-grasp\label{fig:37D_pregrasp}}
    \end{subfigure}
    \begin{subfigure}[t]{0.66\textwidth}
    \centering
    \includegraphics[width=\textwidth]{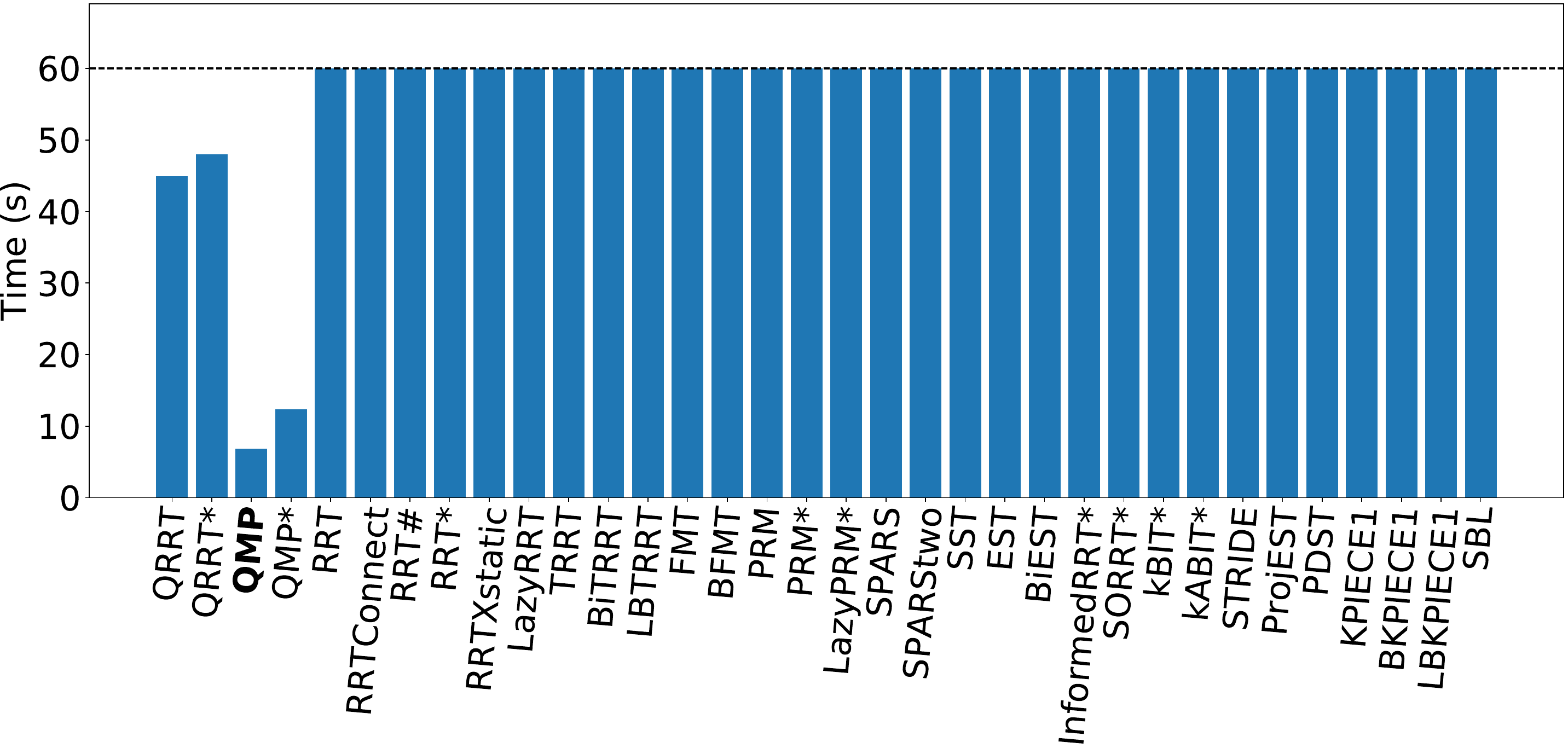}
    \caption{Benchmark\label{fig:37D_pregrasp_benchmark}}
    \end{subfigure}
    \begin{subfigure}[t]{0.33\textwidth}
    \centering
    \includegraphics[width=\textwidth]{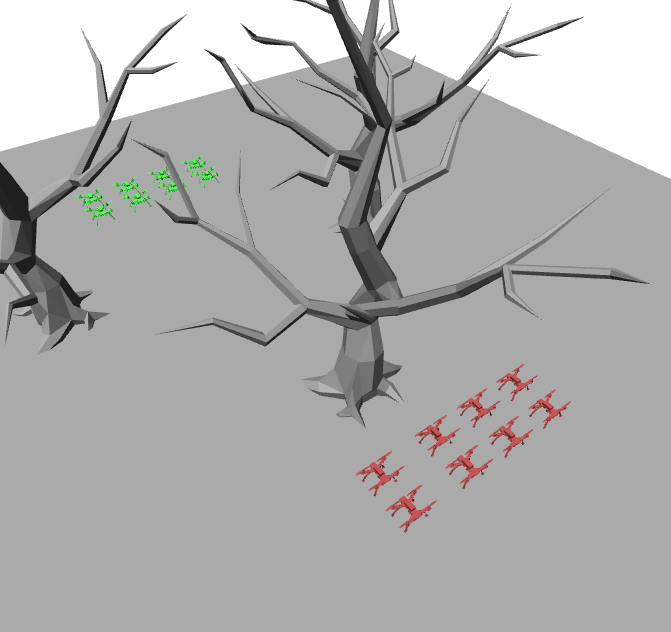}
    \caption{48-dof drones\label{fig:48D_drones}}
    \end{subfigure}
    \begin{subfigure}[t]{0.66\textwidth}
    \centering
    \includegraphics[width=\textwidth]{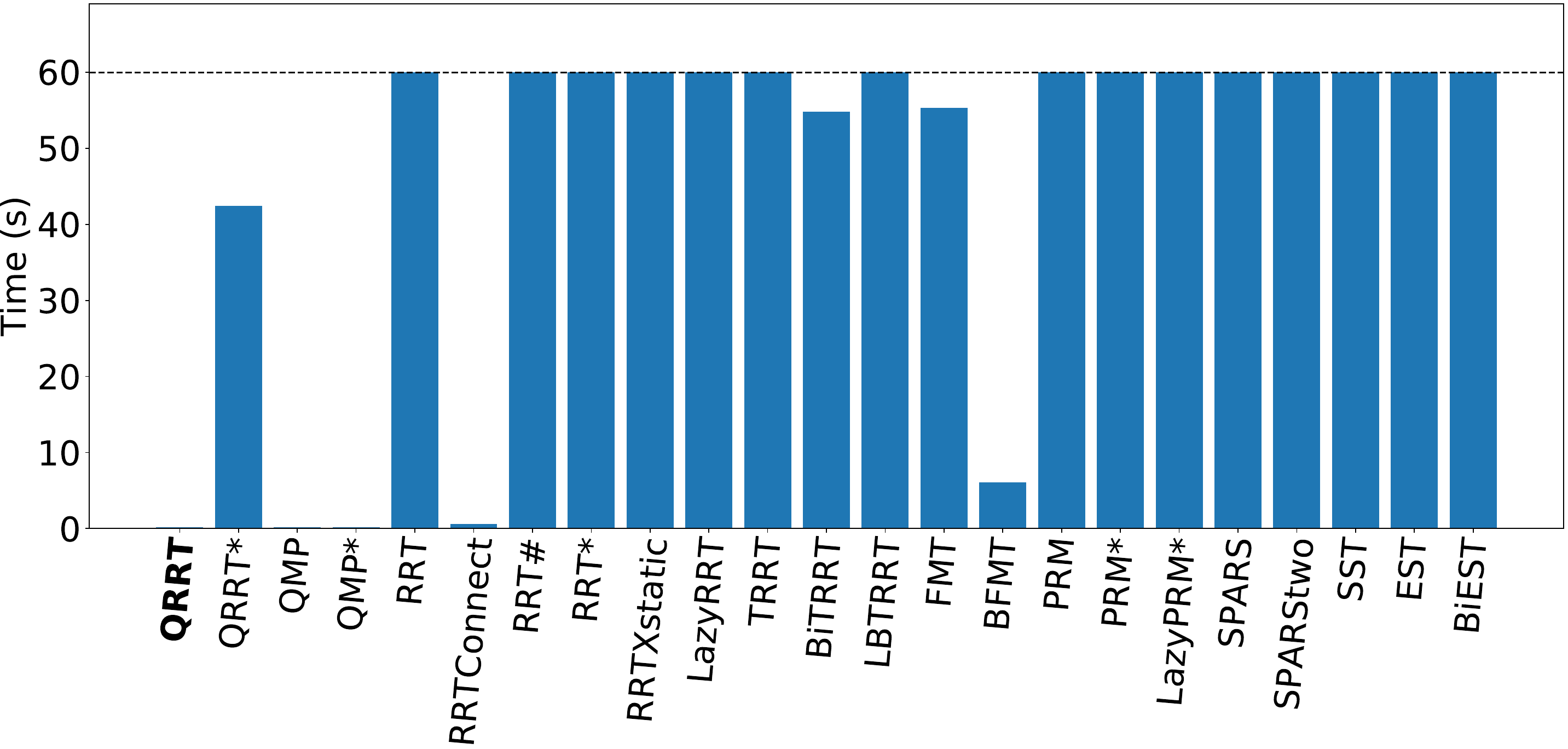}
    \caption{Benchmark\label{fig:48D_drones_benchmark}}
    \end{subfigure}
    \begin{subfigure}[t]{0.33\textwidth}
    \centering
    \includegraphics[width=\textwidth]{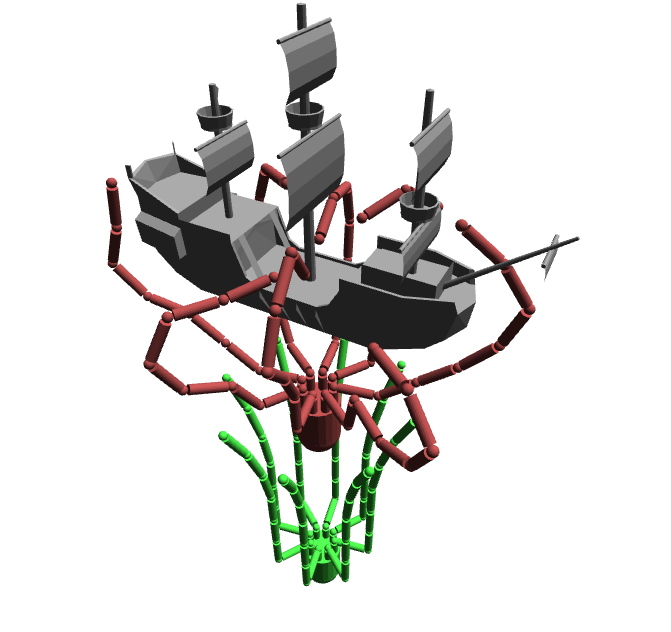}
    \caption{54-dof kraken animation\label{fig:54D_kraken}}
    \end{subfigure}
    \begin{subfigure}[t]{0.66\textwidth}
    \centering
    \includegraphics[width=\textwidth]{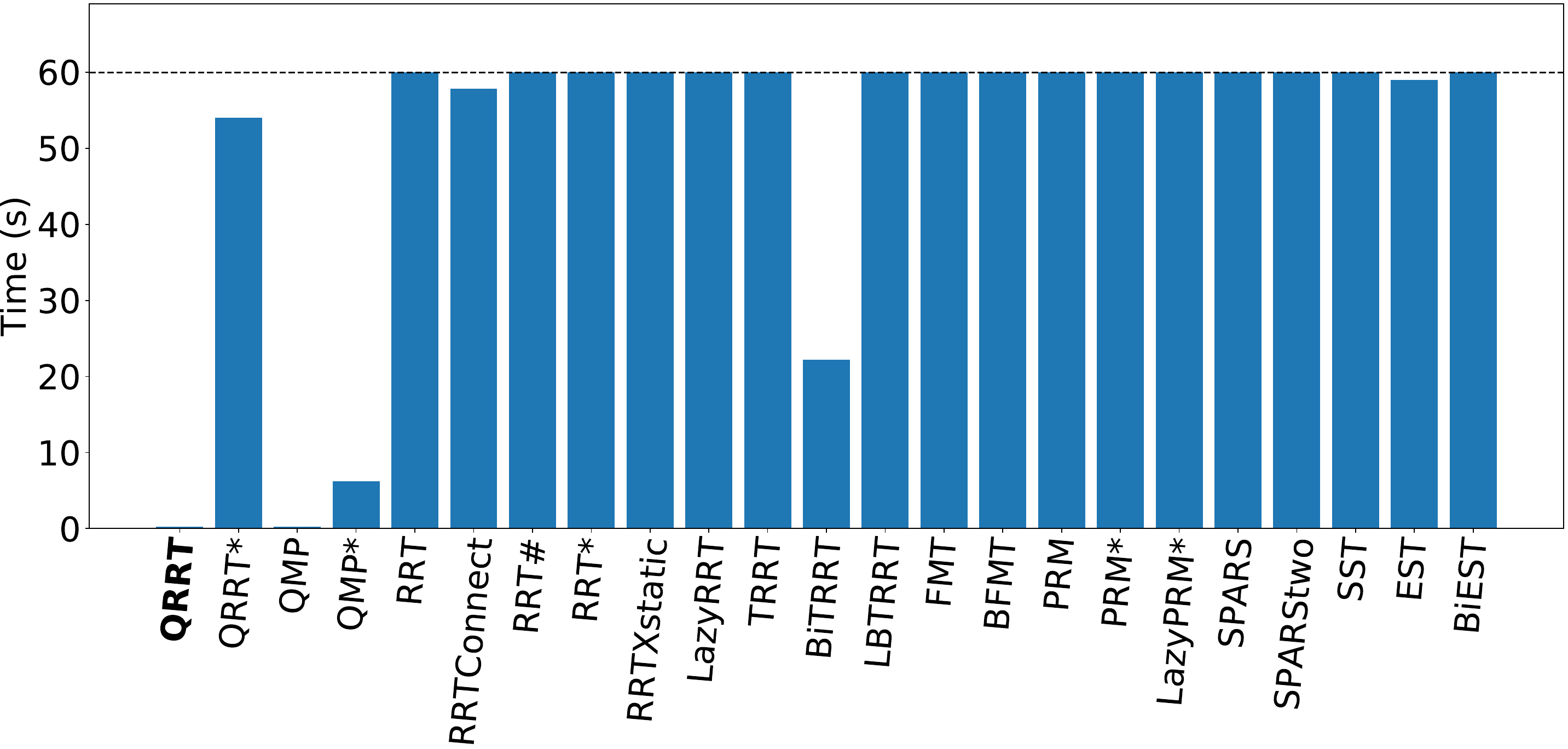}
    \caption{Benchmark\label{fig:54D_kraken_benchmark}}
    \end{subfigure}
    \begin{subfigure}[t]{0.33\textwidth}
    \centering
    \includegraphics[width=\textwidth]{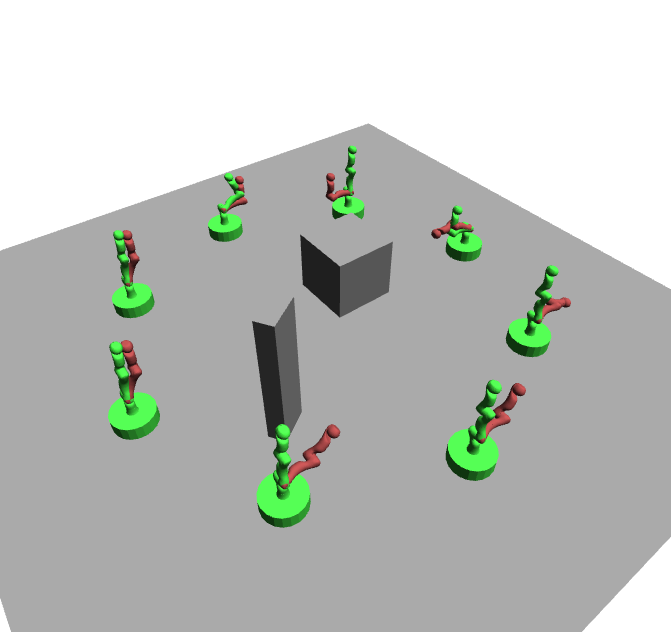}
    \caption{72-dof manipulators\label{fig:72D_manips}}
    \end{subfigure}
    \begin{subfigure}[t]{0.66\textwidth}
    \centering
    \includegraphics[width=\textwidth]{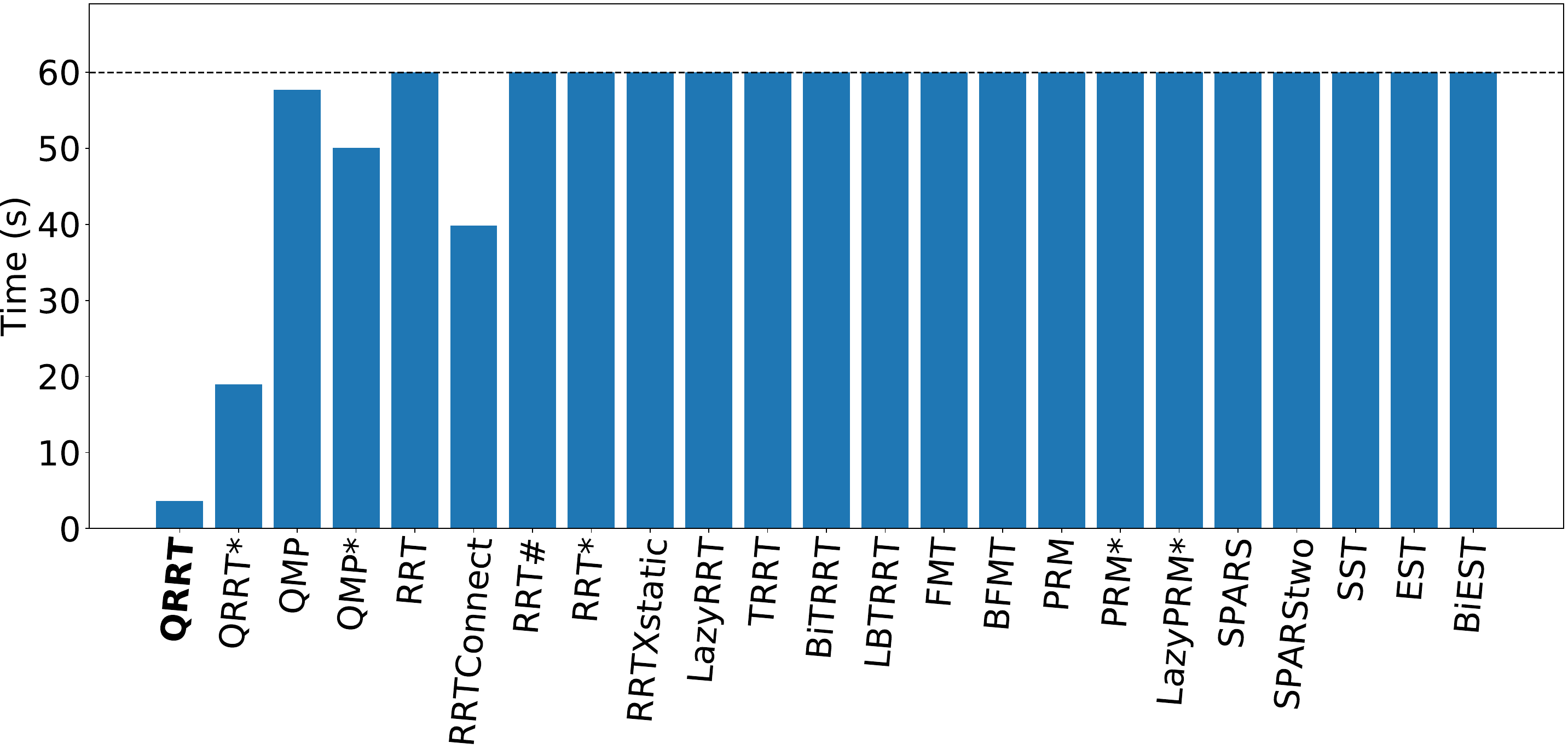}
    \caption{Benchmark\label{fig:72D_manips_benchmark}}
    \end{subfigure}
    \caption{Runtime benchmarks on the last four high-dimensional planning scenarios.}
    \label{fig:benchmarks_2}
\end{figure*}

\subsection{Cost Analysis of High-Dimensional Scenarios\label{sec:evaluations:costanalysis}}
So far, planners have been evaluated with respect to runtime. To also evaluate the cost convergence property, we compare both QRRT* and QMP* on all eight high-dimensional scenarios to QMP, QRRT, BIT*, RRT*, LBTRRT, and RRTConnect. The results are shown in Fig.~\ref{fig:successcostplot-highdim}. 

Let us detail the performance of each algorithm class. First, the non-bundle space planners are only able to tackle two out of eight scenarios. RRTConnect is able to solve the airport and the drones scenario by quickly converging to $100\%$ success rate. In the drones scenario, RRTConnect also finds good, low-cost solutions before any other planner has even found a single solution. However, apart from RRTConnect, the planners RRT*, BIT*, and LBTRRT are not applicable to any of the scenarios with no solved run during the time budget given. 

Second, the bundle space planners QMP, QMP*, QRRT, and QRRT* are able to tackle all eight scenarios.
For the hypercube, QMP, QRRT, and QRRT* quickly find a solution, but are not able to improve upon it. QMP* finds a solution slightly later, but is able to continuously improve upon it. In the box folding task, QMP* is able to solve $90\%$ of the cases while converging quickly to a low-cost solution. Both QRRT and QMP have lower success rates, but find on average a low-cost solution. QRRT*, however, is not able to adequately solve this problem with a success rate of $10\%$. For the crossing cars scenario, all bundle planners reach $100\%$ success rate with both QRRT* and QMP* converging to low-cost solutions over time. For the airport scenario, QRRT and QMP reach $100\%$ success rate, while both QRRT* and QMP* reach only $80\%$ and $30\%$, respectively. In terms of cost convergence, QMP* is not able to improve the initial solution cost and has a large cost variance as indicated by the large shaded region around the average cost. 

In the Shadowhand scenario, QMP, and QMP* reach $90\%$ and $70\%$ success rate, while QRRT, and QRRT* reach only $40\%$ and $20\%$. While QMP* is able to improve the solutions slightly, it has a large variance around the average cost. For the drones scenarios, both QMP and QMP* reach $100\%$ with QMP* converging over time to good low-cost solutions. QRRT is competitive with $90\%$ success rate and low cost average solution as indicated by the cross in the cost plot. However, QRRT* is only able to solve $10\%$ of the runs. For the octopus scenario, QMP, QMP*, and QRRT reach $100\%$ success rate, while QRRT* only reaches $20\%$. QMP* shows quick, and low-variance convergence to an optimal solution. Finally, in the mobile manipulators scenario, QRRT* and QRRT reach $90\%$ success rate, while QMP* reaches $20\%$ and QMP fails to find any solutions. QRRT* is also able to converge quickly over time, reaching a solution cost significantly below solution costs from QRRT, and RRTConnect.

\definecolor{qrrtstar}{HTML}{ED1B23}
\definecolor{qmpstar}{HTML}{99479B}
\newcommand{\sqboxs}{1.5ex}
\newcommand{\sqbox}[1]{\textcolor{#1}{\rule{\sqboxs}{\sqboxs}}}

\def\eWidth{0.47\textwidth}
\begin{figure*}
    \includegraphics[width=\eWidth]{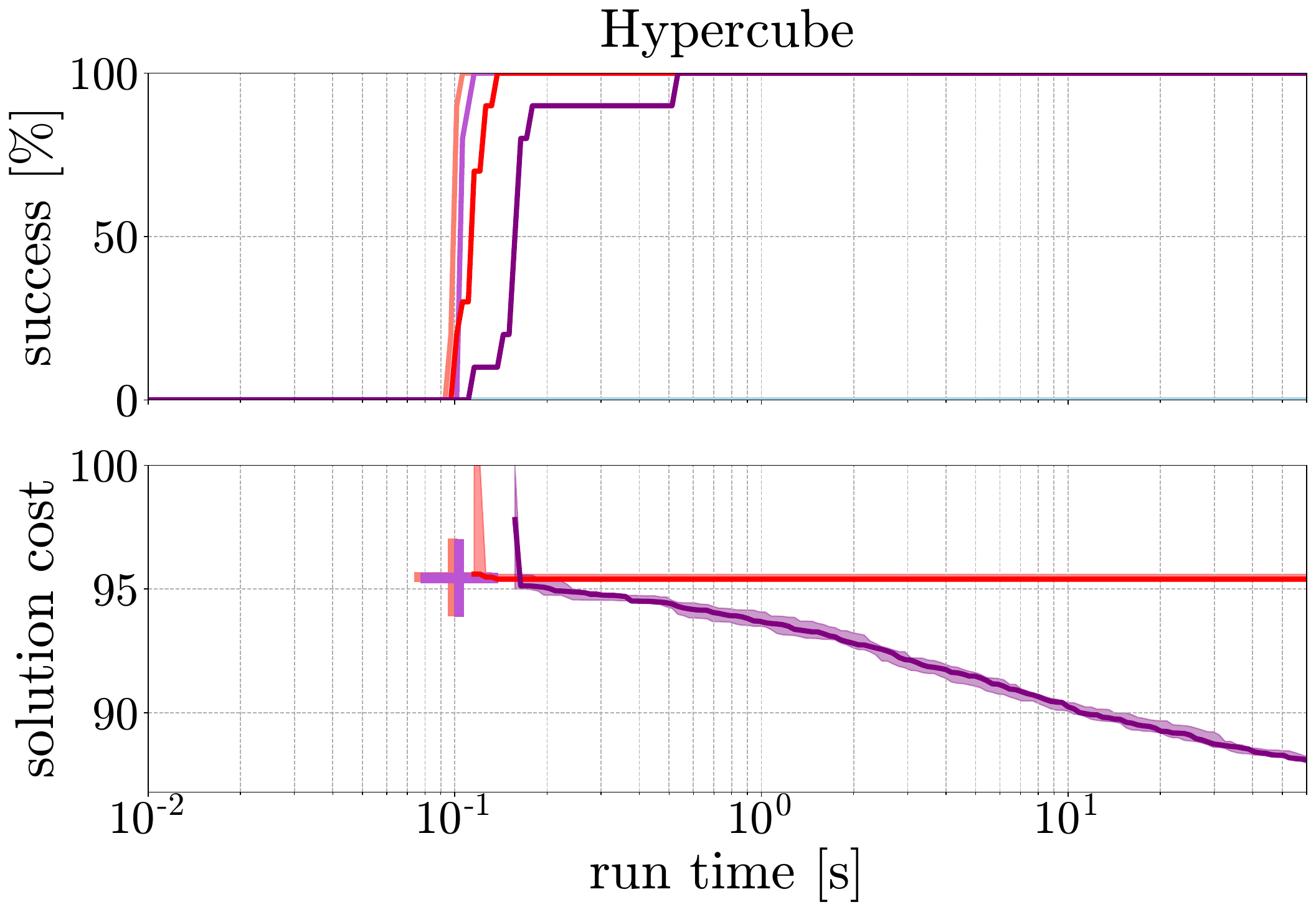}
    \includegraphics[width=\eWidth]{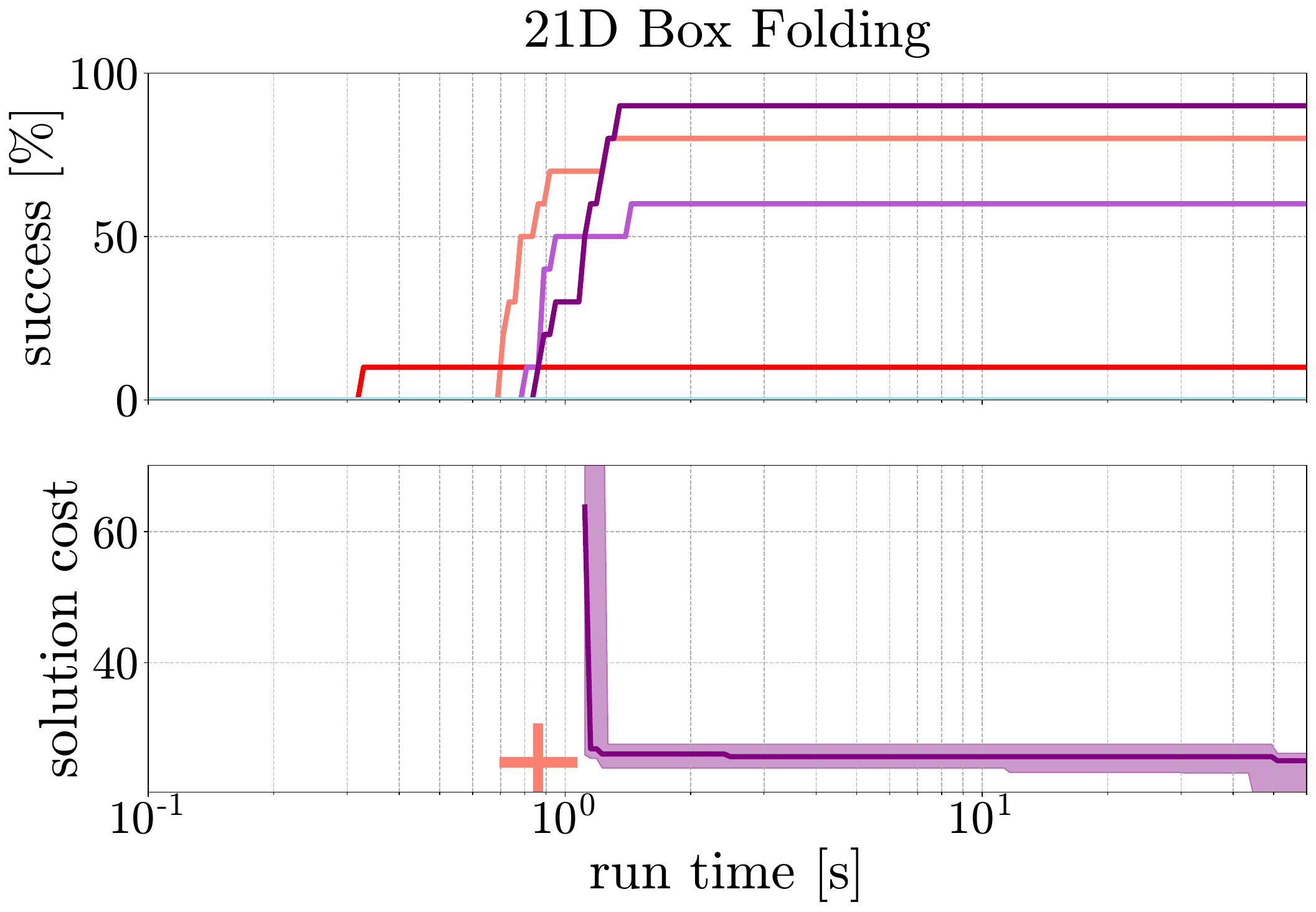}\\
    \includegraphics[width=\eWidth]{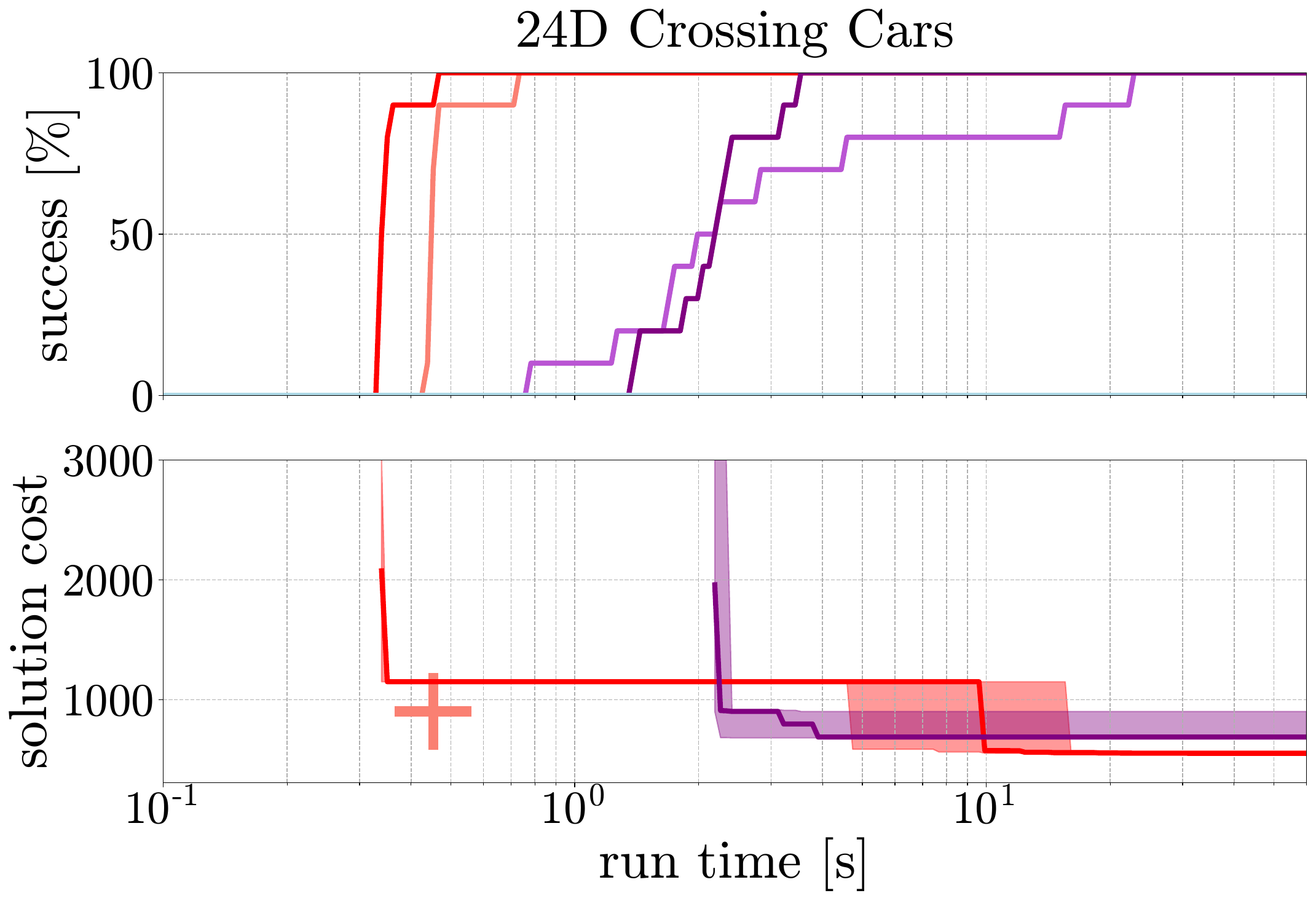}
    \includegraphics[width=\eWidth]{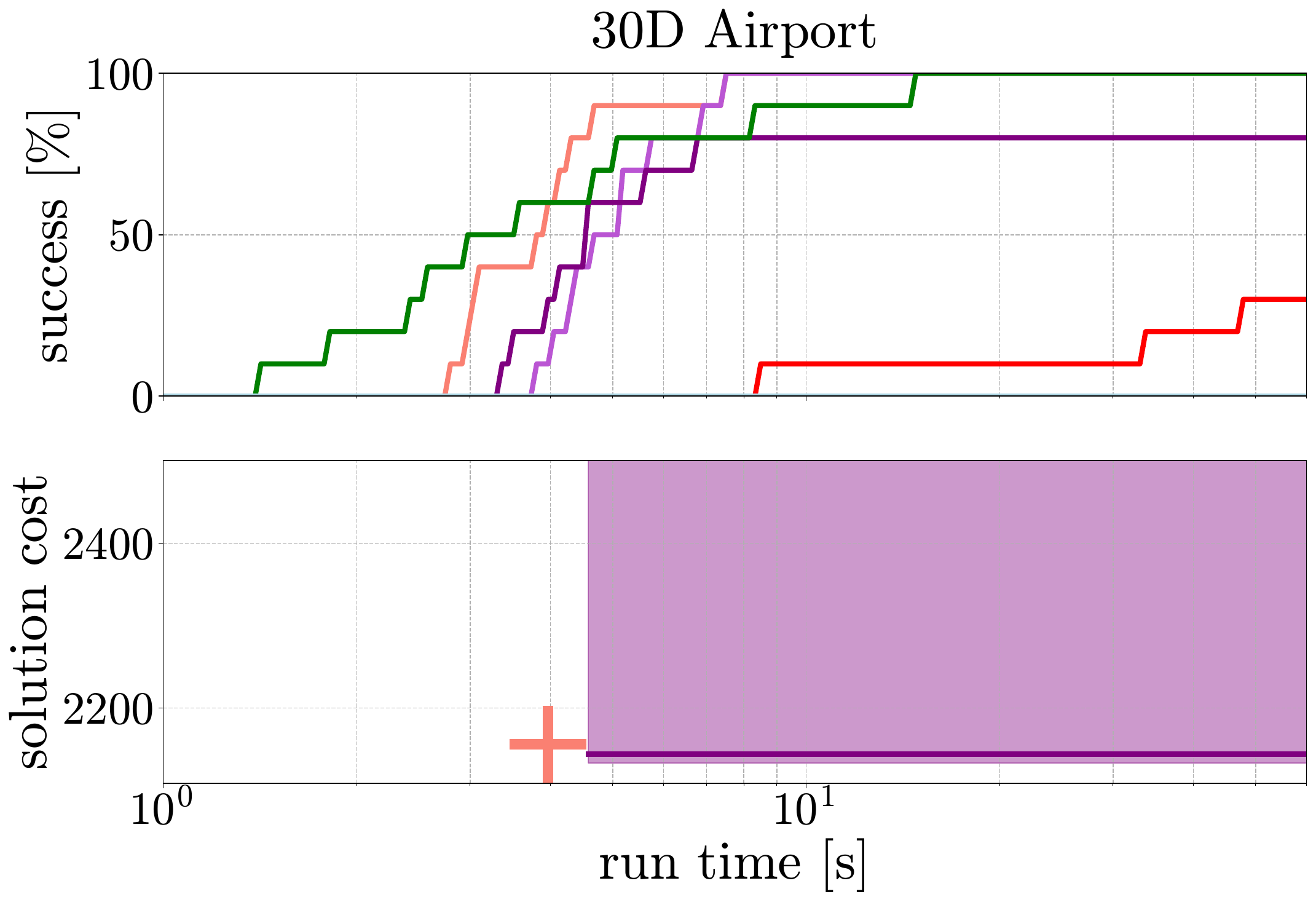}\\
    \includegraphics[width=\eWidth]{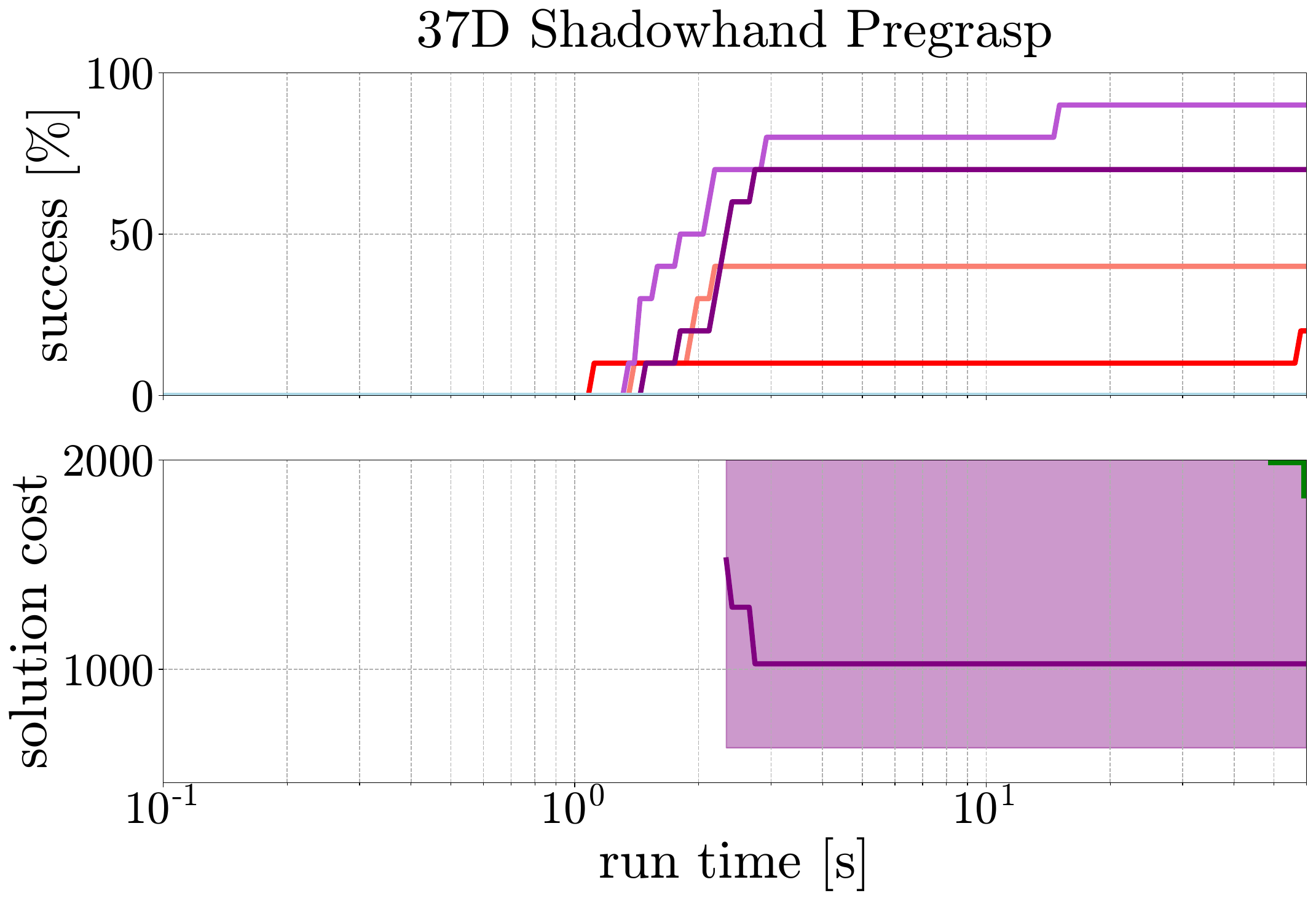}
    \includegraphics[width=\eWidth]{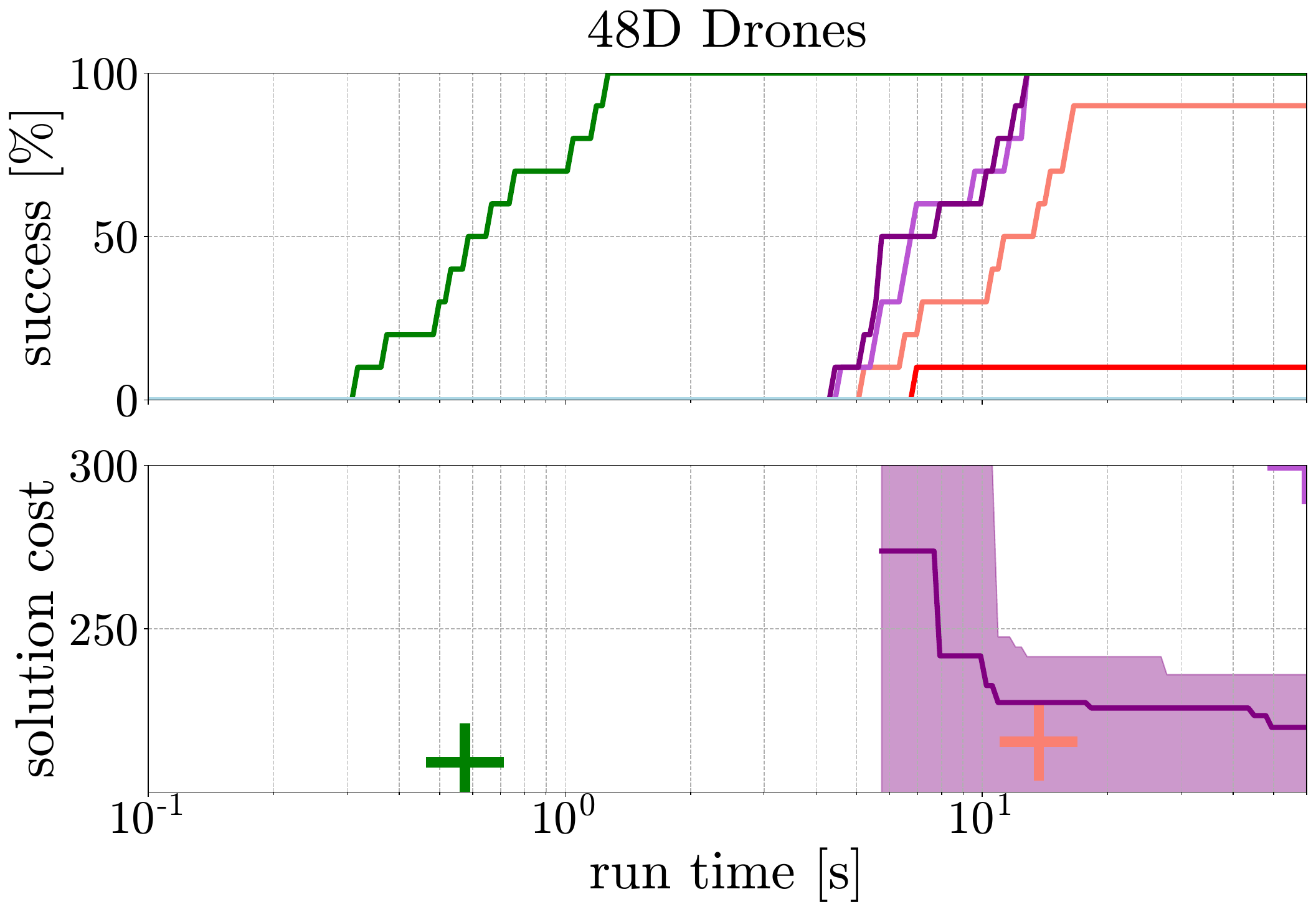}\\
    \includegraphics[width=\eWidth]{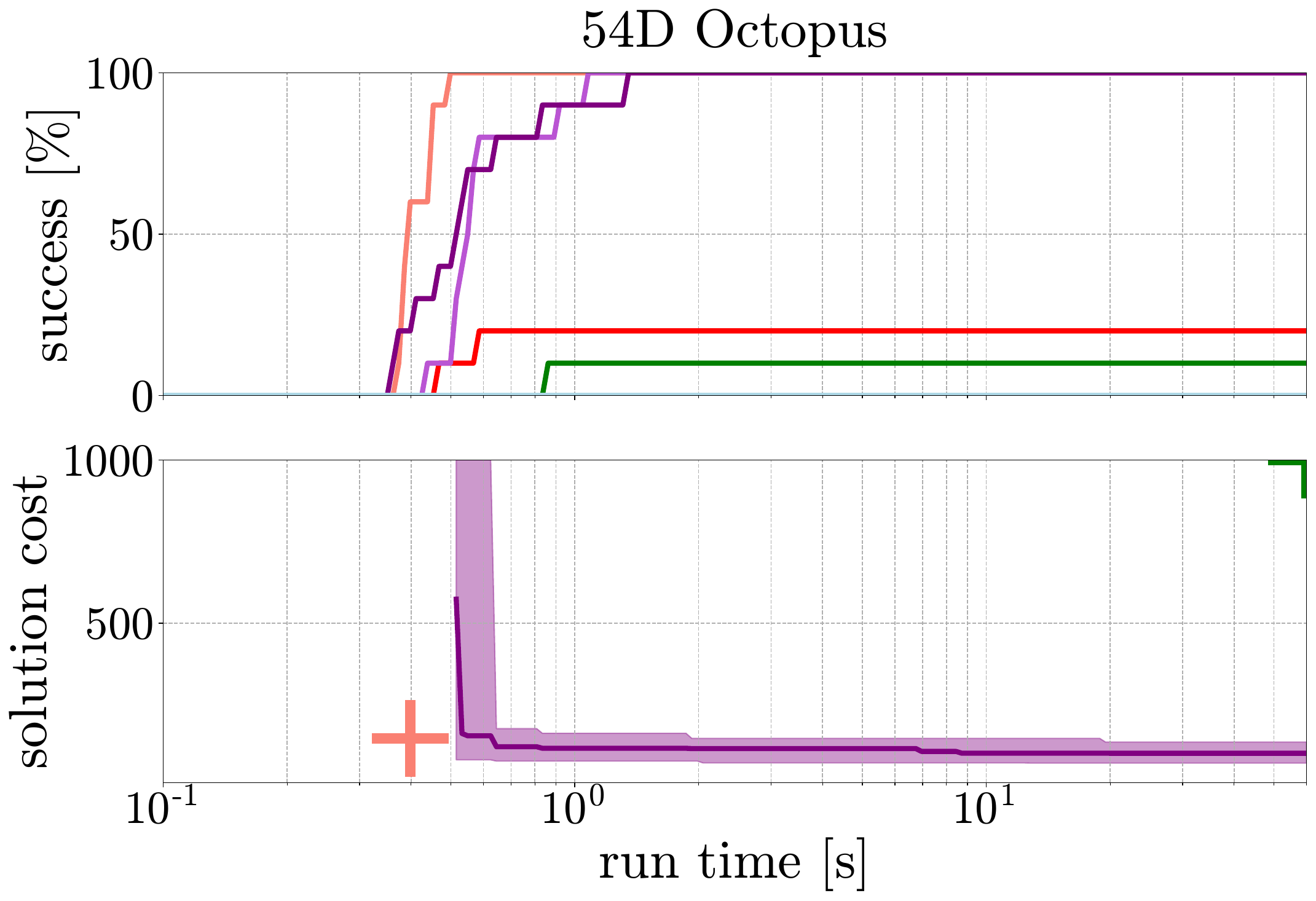}
    \includegraphics[width=\eWidth]{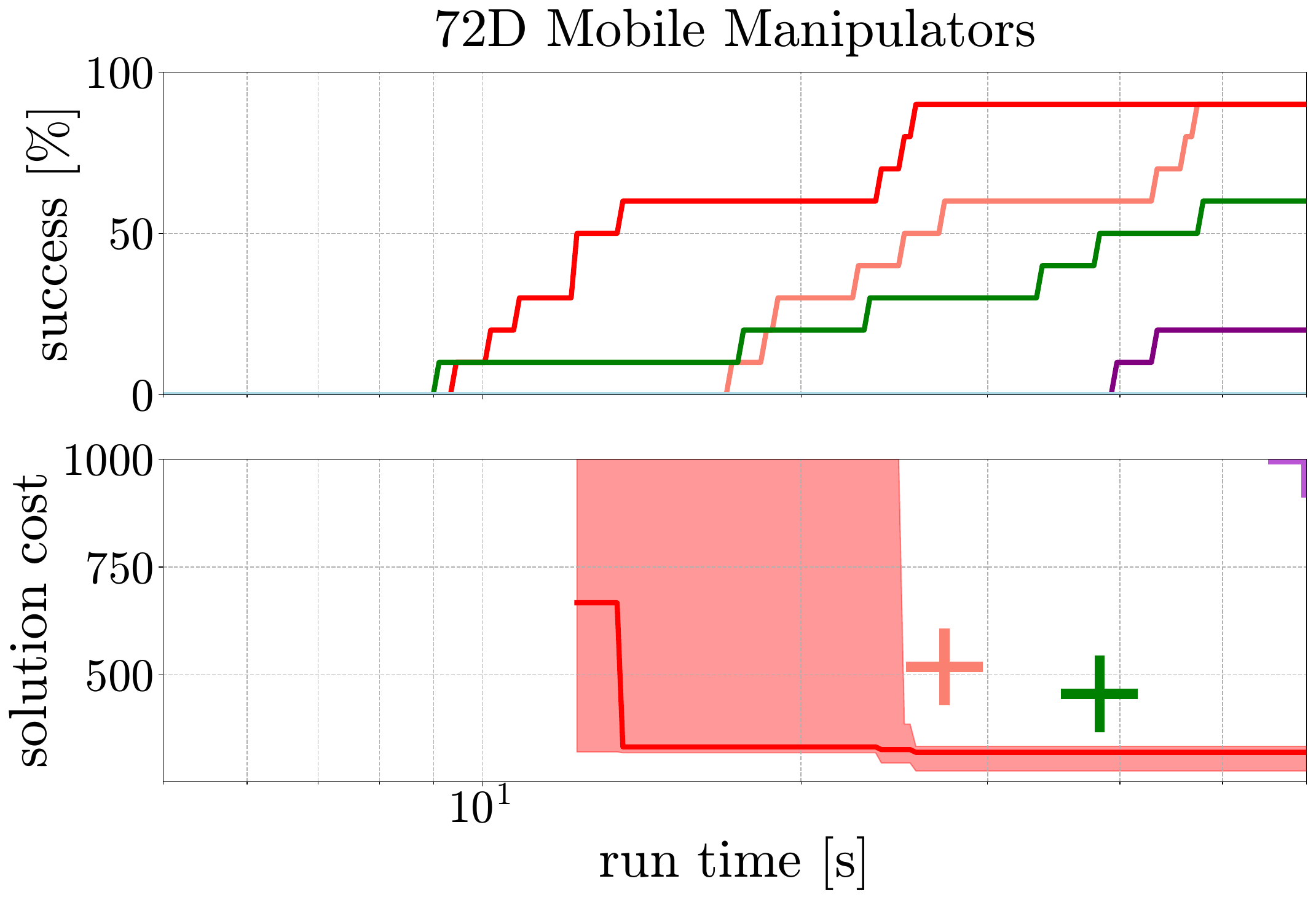} \\
    \centering
    \includegraphics[width=0.7\textwidth]{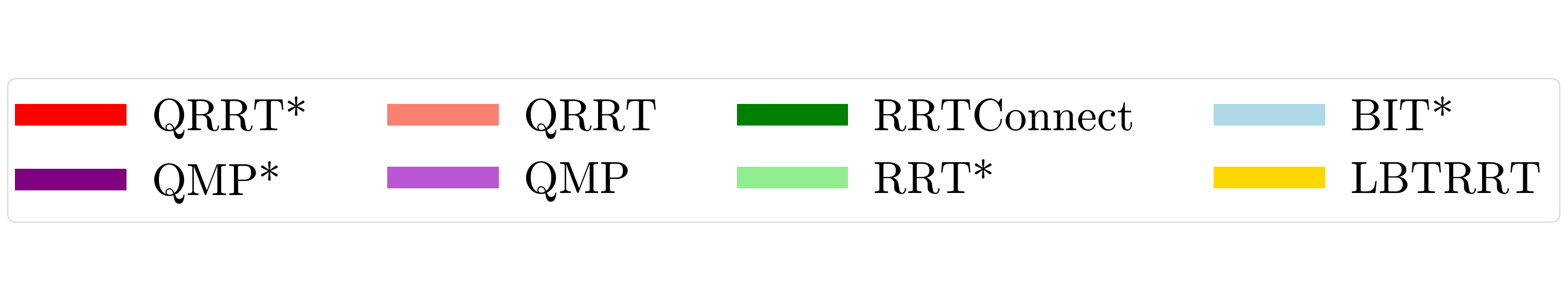}
    \caption{Success-cost plots of the eight high-dimensional planning scenarios.\label{fig:successcostplot-highdim}}
\end{figure*}

\section{Discussion\label{sec:discussion}}

From the preceding evaluation section, we have supporting evidence to draw three broad conclusions. 
First, it is difficult to solve high-dimensional planning problems with classical (non-bundle) motion planning algorithms. This should not be surprising, since the problem is known to be NP-hard \citep{hopcroft_1984, canny_1988, solovey_2020_complexity} and the spaces to contain multiple narrow passages \citep{lozano_perez_1979, salzman_2013}. 

Second, we can often quickly and reliably solve high-dimensional planning problems by exploiting fiber bundles. 
We believe there are three primary contributing factors. 
First, we have expansions of narrow passages. 
If we project a narrow passage onto a base space, we often observe the narrow passage to increase its volume relative to the surrounding space. 
We thereby increase our chance to sample narrow passages on the base space, which we can use to guide sampling on the total space \citep{orthey_2019}. 
Second, we have the removal of infeasible preimages. 
If we find a point on the base space to be infeasible, we can remove their preimage from the bundle space, thereby removing \emph{regions} which cannot be feasible \citep{orthey_2018}. Third, we have dedicated methods to exploit admissible heuristics. If we have a path on the base space, we can often quickly find solutions using the recursive path section method or by using path restriction sampling \citep{zhang_2009}. By staying on the path restriction, we exploit the information from the base space, similar to how we would exploit an admissible cost-to-go heuristic in a discrete search scenario \citep{pearl_1984, aine_2016}.

Third, the cost analysis showed that bundle space planners can successfully converge to low-cost solutions in high-dimensional spaces. However, this seems to only hold true for QMP*, which outperforms QRRT* in terms of cost convergence in seven out of eight scenarios, as shown in Sec.~\ref{sec:evaluations:costanalysis}. QRRT*, however, has inferior performance compared to QMP* and only outperforms QMP* in the mobile manipulators scenario. We believe this is due to QRRT* using tree rewiring, which is an expensive operation. Instead, QMP* does not rely on such an operation and is better suited to tackle high-dimensional spaces.

While our evaluation seems to corroborate those statements, we also like to discuss two limiting issues. The first issue are evaluation outlier, which seemingly contradict our statements. We discuss what they are and what we can do about them. The second issue is our reliance on pre-specification of fiber bundles, which we do for this work manually. We discuss options to automatically specify them in the future.

\subsection{Evaluation Outlier}

From the evaluations, we observe that we often can find solutions over multilevel abstractions quickly and reliably. However, we observe three noteworthy exceptions. First, we observe that QRRT performs below $3$s on every enviroment, except the $37$-dof pregrasp ($43$s) and the box folding task ($8$s). The cost-analysis further shows that QRRT is often not able to reach the $100\%$ success rate. We believe those environments to be challenging for QRRT, because they are examples of ingress problems, i.e. problems where we need to enter a narrow passage, similar to a Bugtrap \citep{yershova_2005}. Such problems could be overcome in future work by developing a bidirectional version of QRRT, by using biased sampling towards narrow passages \citep{yang_2004}, or by selectively expanding states at the frontier of the tree \citep{yershova_2005, denny_2020}. 

Second, we observe QRRT* to perform worse by an order of magnitude compared
to QRRT on five out of eight environments. The cost analysis corroborate this observation by showing that QRRT* performs worse in cost convergence on seven out of eight environments when compared against QMP*. We believe the rewiring of the tree in Alg.~\ref{alg:qrrtstar_grow} slows down planning over multilevel abstractions. In the future, we could overcome this by either postpone rewiring of the tree until a solution is found or by exploiting informed sets \citep{gammell_2014}, which are admissible lower bounds on the optimal solution. It could also be fruitful to investigate the connection between quotient space metrics and the geometric shape of informed sets, which we could use as admissible heuristics \citep{gammell_2020}. 

Third, we observe that the non-bundle planner RRTConnect performs competitively on the $30$-dof airport and the $48$-dof drones environment. 
Also BFMT performs competitively on $48$-dof drones. 
It seems, we could solve both problems without using fiber bundles. 
We believe this to happen because both scenarios involve $SE(3)$ state spaces, where narrow passages might be rarer than in $SE(2)$ scenarios. 
In those environments, we therefore have enough volume to quickly find valid samples, which we can exploit using RRTConnect, or BFMT. 
However, we believe fiber bundles are still needed. 
First, we do not know if RRTConnect or BFMT would still perform well if we further increase dimensionality. 
Second, only by using bundle planners can we consistently and reliably find solutions in all environments. 
Third, fiber bundles are often the only option if we want to rapidly establish infeasibility or organize local minima over high-dimensional state spaces \citep{orthey_2020}. 
It is, however, necessary to investigate how narrow passages slow down planning and how we could overcome them using fiber bundles.
We previously conducted some evaluations in that direction for the QRRT planner~\citep{orthey_2019}.

\subsection{Specifying Fiber Bundles}

For each problem, fiber bundles have to be specified manually. This is problematic, since there is no clear guideline on how to select fiber bundles for a specific problem. This could be overcome by optimizing over a primitive set of fiber bundles. To create a primitive set of fiber bundles, we could use the largest inscribed sphere for a rigid body, the removal of links from a chain, or the removal of nonholonomic constraints from a dynamical system. We can then search the landscape of such primitive fiber bundles to find an efficient fiber bundle for a specific robot and a specific set of environments. A recent study by \cite{brandao_2020} shows promising results in that direction by using evolutionary algorithms to select an abstraction. It could also be promising to use workspace information to select a fiber bundle \citep{yoshida_2005}, either by choosing joints which can actuate links of interest through the workspace \citep{luna_2020} or by choosing a bundle on-the-fly based on which links are closest to obstacles \citep{kim_2015}. We thereby could choose different fiber bundles for large rooms, for narrow passages or for ingress tasks. However, in those cases, we would need to consider fiber bundles with changing dimensions, which are in general given by the concept of a sheaf \citep{bredon_2012}.

\section{Conclusion\label{sec:conclusion}}

We modelled multilevel motion planning problems using the framework of fiber
bundles. To exploit fiber bundles, we developed a set of bundle primitives, and the bundle planners QRRT* and
QMP*, which we showed to be probabilistically complete and asymptotically
optimal. 
We also extended the existing bundle planners QRRT
\citep{orthey_2019} and QMP \citep{orthey_2018} using an exponential importance
criterion and a recursive L1 path section method (Fig.~\ref{fig:pullfigure}). 
We conducted a meta-analysis to find the best implementation of the bundle primitives, including graph sampling, metric, importance selection, and path section methods. Using the bundle planners, we robustly and efficiently solved challenging high-dimensional motion planning problems, from 21-dof to 100-dof. 
We also showed competitive results for low-dimensional scenarios, and we showed QMP* to be superior in cost convergence for high-dimensional scenarios. 

However, we believe there is still room for improvement. In particular, runtime could be further reduced by developing a bidirectional version of QRRT \citep{lavalle_2001}, by improving convergence using informed sets \citep{gammell_2014}, by investigating novel path section optimization methods \citep{zhang_2009}, and by automatically searching fiber bundles to exploit \citep{kim_2015, brandao_2020}---i.e. with respect to a given bundle algorithm \citep{orthey_2019}. 
We also believe it is worthwhile to investigate the connection to complementary approaches, like computing neighborhoods \citep{lacevic_2020} and exploiting sufficiency conditions \citep{grey_2017}.

However, despite room for improvements, we showed that bundle planners can efficiently exploit fiber bundles. 
By exploiting fiber bundles, bundle planners outperformed existing planners often by up to $2$ orders of magnitude, occasionally up to $6$ orders of magnitude. 
Thus, we believe to not only have contributed to solving multilevel planning problems in the now, but also to have contributed tools and insights to investigate high-dimensional state spaces in the future.

\section{Acknowledgement}

The authors disclose receipt of the following financial support for the research, authorship and publication of this article: This work was supported by the Alexander von Humboldt Foundation [individual grant], the Japan Society for the Promotion of Science [individual grant] and the Max-Planck Society [fellowship grant].

\balance
\bibliographystyle{style/SageH}
\bibliography{bib/reductions}

\appendix

\section{Background Fiber Bundles\label{sec:appendix:background}}

Fiber bundles are based upon the concepts of equivalence relations, and quotient spaces, with close ties to constraint relaxation, and admissible heuristics. We provide here a short overview about those concepts.

\subsection{Equivalence Relations\label{sec:prelims:equivalencerelations}}

An equivalence relation $\sim$ is a binary relation on a space $\X$ such that for any elements $x,y,z \in \X$ we have $x \sim x$ (reflexive), if $x \sim y$ then $y \sim x$ (symmetric) and if $x \sim y$ and $y \sim z$ then $x \sim z$ (transitive) \citep{munkres_2000}.

An equivalence relation partitions the space $\X$ into disjoint subsets we call equivalence classes \citep{munkres_2000}. Given an element $x$ in $\X$, the equivalence class of $x$ is the set of elements $[x] = \{ y \mid y \sim x\}$.

\subsection{Quotient Spaces\label{sec:prelims:quotientspaces}}

We often like to simplify a space $\X$ under an equivalence relation $\sim$ by taking the quotient. Taking the quotient means that we compute the quotient space $\quotientspace = \X / \sim$ under the quotient map $\pi: \X \rightarrow \quotientspace$. The quotient space is the set of all equivalence classes imposed by $\sim$ on $\X$. To manipulate those equivalence classes, we can often \emph{represent} the quotient space by assigning an equivalence class to a point of a representative space. We define this representative space as a space $B$ under a (bijective) representative mapping $\nu: \quotientspace \rightarrow B$ \citep{lee_2003}. 

Let us consider an example. In Fig.~\ref{fig:cosets}, we show the plane $\R^2$ with elements $x = (x_1,x_2)$ under the equivalence relation of vertical lines, i.e. $x \sim x'$ if $x_1 = x'_1$. An equivalence class $[x] = \{ x' \mid x' \sim x\}$ represents a vertical line, i.e.\ the set of points in $\R^2$ with equivalent $x_1$ value. Taking the quotient, we obtain the quotient space $\quotientspace = \R^2 / \sim$, the set of vertical lines in $\R^2$ (Fig.~\ref{fig:cosets} Middle). We can then \emph{represent} $\quotientspace$ by the representative space $\R^1$ by associating to each equivalence class (vertical line) the real value $x_1$ using the representative mapping $\nu: \quotientspace \rightarrow \R^1$ we define as $\nu([x]) = x_1$ (Fig.~\ref{fig:cosets} Right).

\begin{figure}
  \centering
  \newcommand\plane[1]{
  \coordinate (A) at (0,0);
\coordinate (B) at (0,-2);
\coordinate (C) at (2,-1);
\coordinate (D) at (2,-3);
\coordinate (E) at (0,-1);
\coordinate (F) at (2,-2);
\draw[<-,shorten <= -10pt] (A) -- (B);
\draw (B) -- (D);
\draw (D) -- (C);
\draw (C) -- (A);
\draw[->,shorten >= -10pt] (E) -- (F);
\node[above, yshift=10pt] at (A){$x_2$};
\node[above right] at (F){$x_1$};
\coordinate (Q) at (0.8,-0.4);
\node[above right] at (Q){$x$};
\draw[black,fill=black,circle] (Q) circle (1pt);
\coordinate (G) at (1,-3.5);
\node at (G){#1};
}
\begin{tikzpicture}
\plane{$\R^2$}
\end{tikzpicture}
\begin{tikzpicture}
\plane{$\R^2 / \sim$}
\foreach \x in {1,...,9}{
   \coordinate (\x0) at (\x*0.2,-\x*0.1);
   \coordinate (\x1) at (\x*0.2,-\x*0.1-2);
   \draw (\x0) -- (\x1);
   }
\end{tikzpicture}
\begin{tikzpicture}
\coordinate (E) at (0,-1);
\coordinate (F) at (1.8,-1.8);
\coordinate (G) at (1,-3.5);
\coordinate (Q) at (0.8,-1.35);
\coordinate (x0) at (0.8,-0.4);
\coordinate (x1) at (0.8,-2.4);
\draw[->,shorten >= -10pt] (E) -- (F);
\node at (G){$\R^1$};
\draw[black,fill=black,circle] (Q) circle (1pt);
\draw[black,fill=black,circle] (x0) circle (1pt);
\node[above right] at (Q){$x_1$};
\node[above right] at (x0){$x$};
\draw[dashed] (x0) -- (x1);
\end{tikzpicture}
\vspace{0pt}
  \caption{Quotient space example. \textbf{Left:} Space $\R^2$. \textbf{Middle:} Quotient space $\quotientspace = \R^2 / \sim$, the set of equivalence classes of vertical lines. \textbf{Right:} Representative space $\R^1$ under representation mapping $\nu: \quotientspace \rightarrow \R^1$ (Adapted from \citep{orthey_2018}).\label{fig:cosets}}
  \vspace{0pt}
\end{figure}

\subsection{Constraint Relaxation\label{sec:prelims:constraintrelaxation}}

To approximate a complex problem, we can often use the concept of constraint relaxation. Let $\X$ be a space and $\phi: \X \rightarrow \R$ be a constraint function on $\X$. To solve a planning problem on $\X$, we need to search through the free space $\Xf$, which might have zero-measure constraints or narrow passages. To simplify such a problem, we replace the constraint function $\phi$ by a constraint relaxation function $\phi_R$ under the condition
\begin{equation}
    \phi_R(x) \leq \phi(x)
\end{equation}
for any $x$ in $\X$. 

We can explain this condition geometrically as an expansion of the free space $\Xf$ when using $\phi_R$ \citep{orthey_2019}. Constraint relaxations \citep{roubivcek_2011} are advantageous, because we can use solutions of the relaxed problem as certified lower bounds on the solution of the original problem.

\subsection{Admissible Heuristics\label{sec:prelims:admissibleheuristics}}

In a search problem, we like to find paths through a state space $\X$ to move from an initial element $\xi \in \X$ to a goal element $\xg \in \X$. When casting this as a search problem, we often like to know which state to expand next. A helpful tool is the cost-to-go (or value) function $\hstar: \X \rightarrow \R$ which defines the cost of the optimal path from any point to the goal. An admissible heuristic is an estimate $h: \X \rightarrow \R$ which lower-bounds $\hstar$ as
\begin{equation}
    h(x) \leq \hstar(x)
\end{equation}
for any $x$ in $\X$ \citep{pearl_1984, edelkamp_2011, aine_2016}. Admissible heuristics are important because we can use them to guarantee optimality and completeness in algorithms like A* \citep{hart_1968, pearl_1984} and to often decrease planning time significantly \citep{aine_2016}.

\section{Exponential Change\label{sec:appendix:expdecay}}

To model quick but smooth transitions between two parameter values, we use an exponential decay function. Let $\kappa_0$ be the start and $\kappa_1$ be the final parameter value. We model the change between $\kappa_0$ and $\kappa_1$ using the exponential decay function
\begin{equation}
    \kappa(t) = (\kappa_0 - \kappa_1) \exp(-\lambda t) + \kappa_1
\end{equation}
with $t \in \R_{\geq 0}$ being the time or iteration number, $\kappa(0) = \kappa_0$, $\lim_{t \rightarrow \infty} \kappa(t) = \kappa_1$, $\exp$ being the exponential function and $\lambda \in \R_{\geq 0}$ being the decay parameter. 

\section{Meta-Analysis of Primitive Methods\label{sec:appendix:metaanalysis}}

\def\pWidth{0.48\linewidth}
\begin{figure*}
    \centering
    \includegraphics[width=\pWidth]{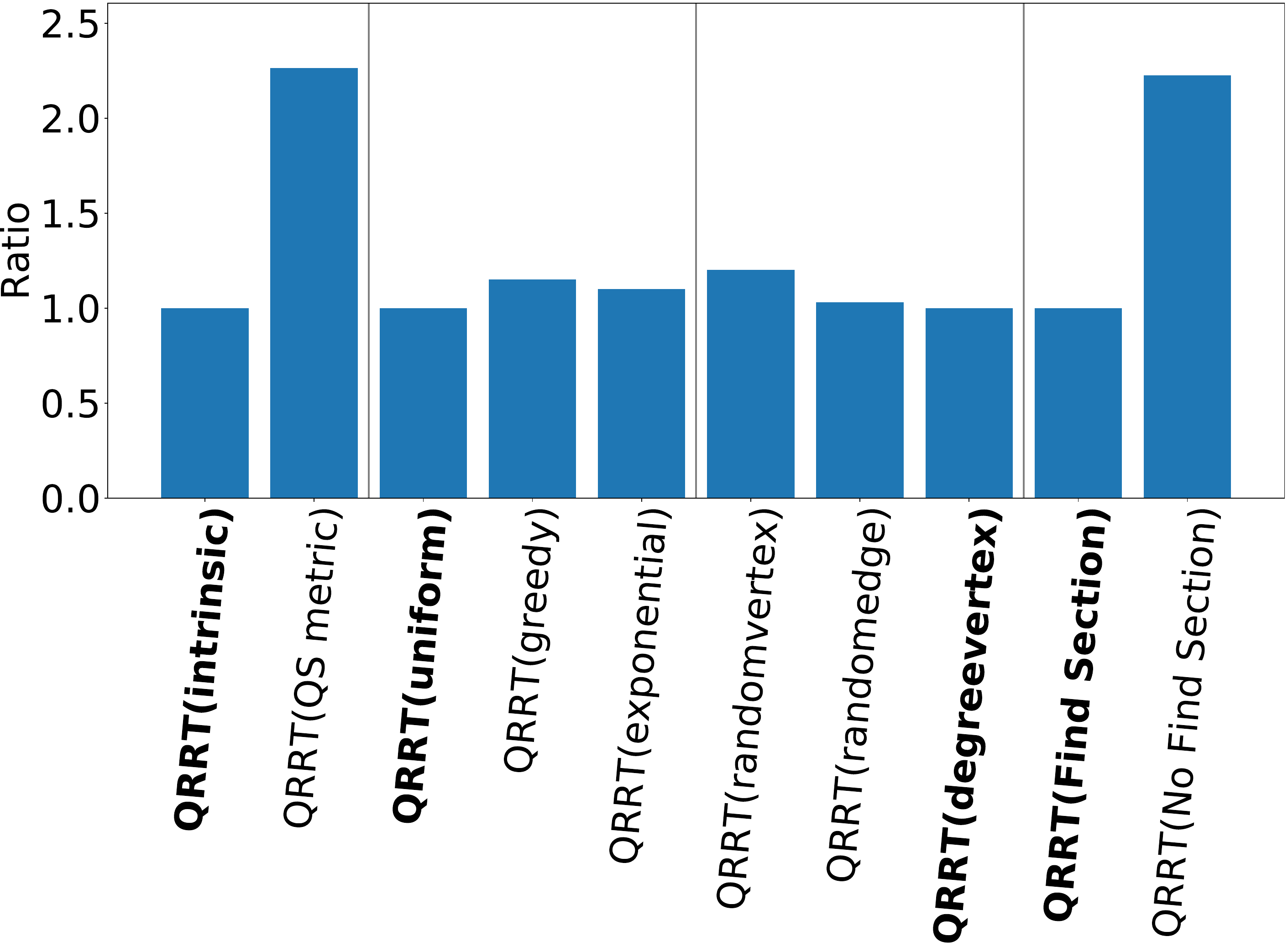}
    \includegraphics[width=\pWidth]{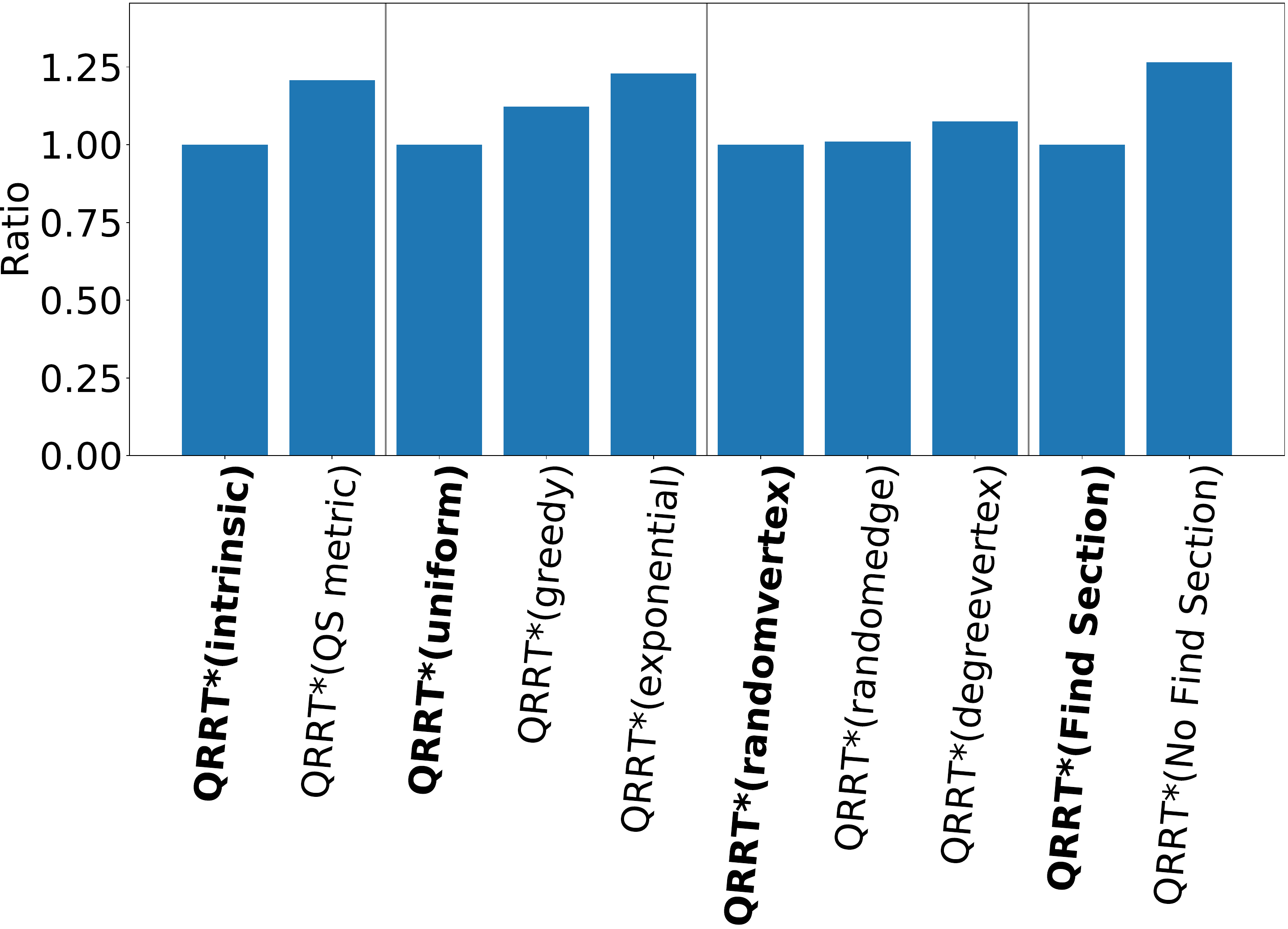}\\
    \includegraphics[width=\pWidth]{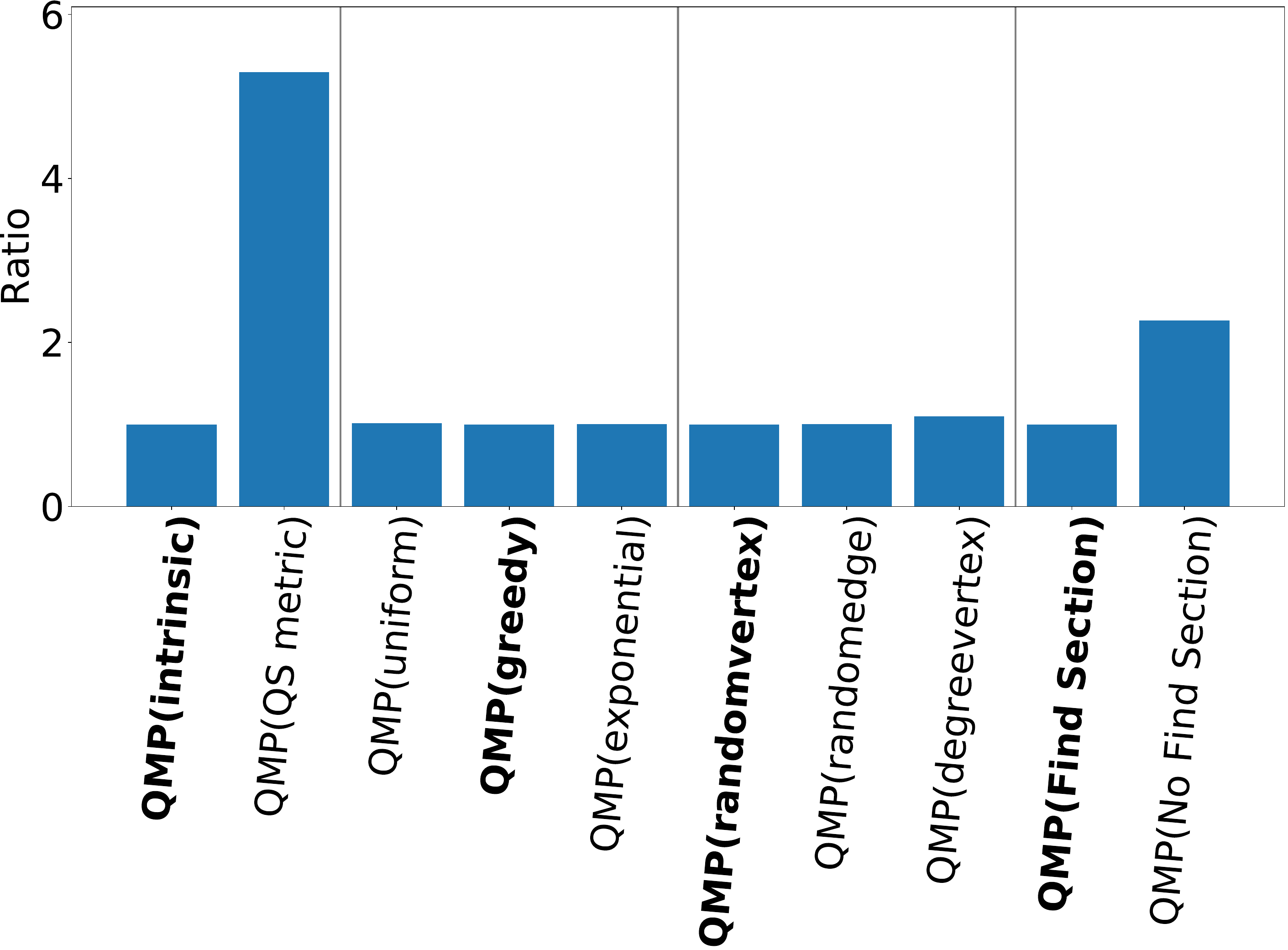}
    \includegraphics[width=\pWidth]{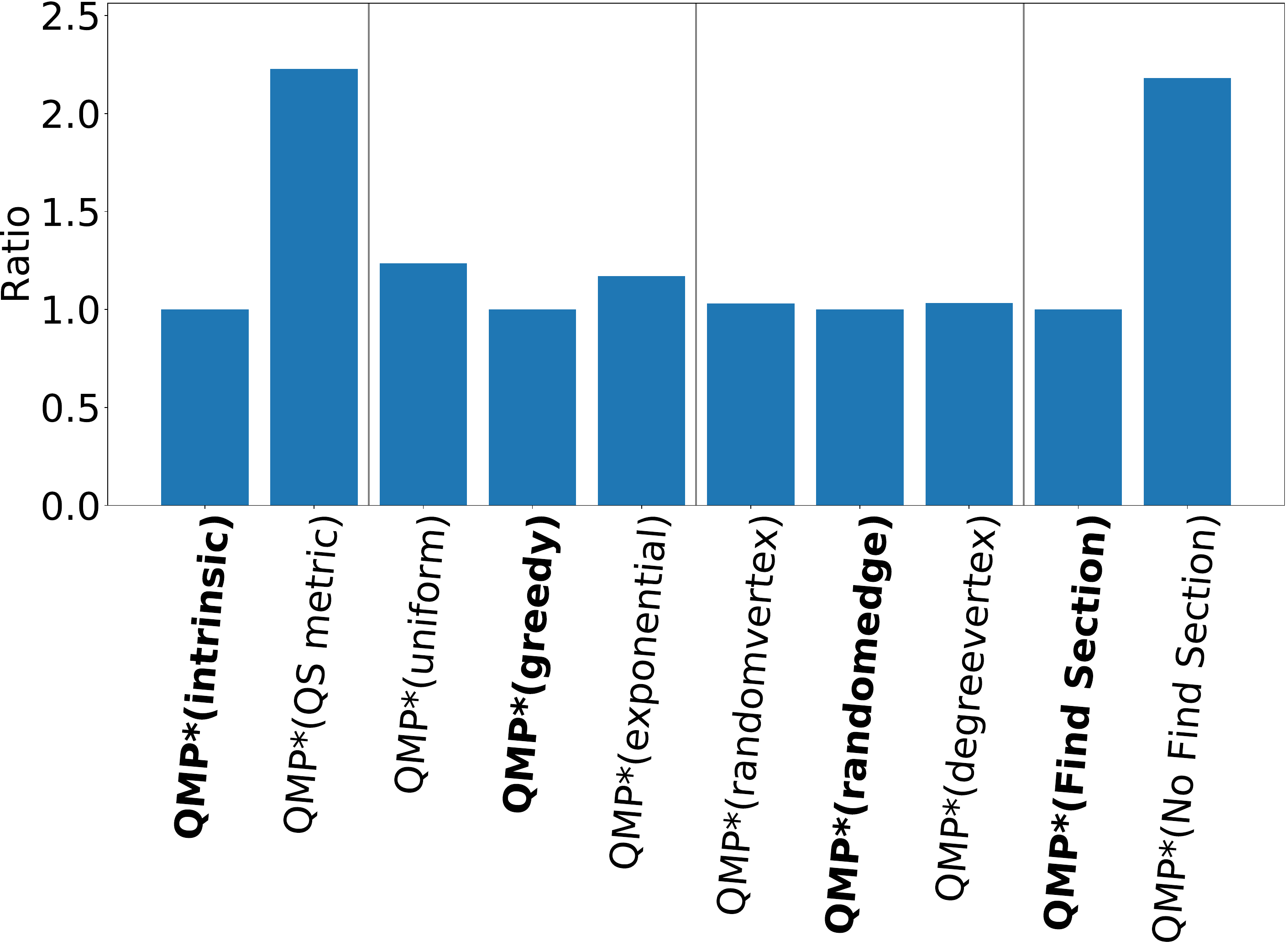}
    \caption{Meta analysis of different implementations of bundle space primitives. Each graph shows the performance of QRRT, QRRT*, QMP, and QMP* on the high-dimensional benchmark set by comparing four different primitives, the metric (intrinsic vs. quotient-space), the importance selection (uniform, greedy, or exponential), the sampling strategy (random vertex, random edge, random degree vertex), and having the side-step section method. See Sec.~\ref{sec:primitives} for details. The results are displayed as ratio compared to the value of the best performing implementation in each category.}
    \label{fig:metaanalysis}
\end{figure*}

As discussed in Sec.~\ref{sec:primitives}, each bundle space primitive can be implemented in multiple ways. To find out which method works best for a specific
algorithm, we perform a meta-analysis. In this meta-analysis, we select each
bundle algorithm QRRT, QRRT*, QMP and QMP* and vary its primitive methods. We
vary those methods by taking the runtime average over the same set of
environments as in Sec.~\ref{sec:evaluation} (except the hypercube). We then present the results as ratios of the best runtime. This means, to find the best sampling method for QRRT, we let QRRT run on all environments with different sampling method, then average the results for each method. We then take the method with the lowest runtime and assign it the ratio $1$. All other runtimes are represented as multiples of the lowest runtime. 

The results are shown in Fig.~\ref{fig:metaanalysis}. We divide the results into four groups. First, we compare the intrinsic metric to the quotient space (QS) metric (left group). Second, we compare the importance selection of a bundle space by comparing uniform, exponential and epsilon greedy (middle left). Third, we compare the graph sampling strategies, namely random vertex, random edge and degree vertex (middle right). Finally, we compare the algorithms with enabled find section method and without (right). 

In the case of QRRT, we observe the best metric to be the intrinsic metric (left) and that using the recursive find section method, we can lower the runtime significantly (right). 
However, for sampling and selection, we do not have a clear best strategy. 
Instead, we observe that a change in sampling or importance has a marginal influence on the performance.
For the other three algorithms QRRT*, QMP and QMP*, we observe similar results. 
One exception is QRRT*, where we observe the QS metric and the no find section method to perform only $1.25$ times worse.

\subsection{Discussion of results}

The results indicate that both for sampling and importance selection, there is no clear advantage of using either strategy. This suggests that either strategy can be chosen for the scenarios under investigation. Further investigation is required to understand the influence of sampling strategies over different types of bundle spaces. 

Concerning the metric and section method, the difference in performance is significant. In detail, for all bundle planners, both the intrinsic method, and the find-section method perform significantly better. The reason why the intrinsic metric is better lies in its simplicity. While the intrinsic metric can rapidly return values, the QS metric requires an expensive graph search. While the QS metric is more accurate, this is offset by its computational burden. The reason why the find-section method performs better is due to independent movements of links caused by the L1-interpolation. This is often a decisive factor to ensure that colliding links are moved out of the way to clear the way towards the goal. Most of the problems in our evaluations benefit from this movement. An example is the box folding task, where moving outer links towards the goal positions increased our chances to find collision-free motions.

\end{document}